\documentclass{article}
 \pdfoutput=1
\PassOptionsToPackage{numbers, compress}{natbib}

\usepackage[preprint]{neurips_2023}

\RequirePackage{amssymb}
\RequirePackage{graphicx}
\RequirePackage{fullpage}
\usepackage{wrapfig}
\usepackage{comment}
\RequirePackage[colorlinks]{hyperref}
\RequirePackage{amsthm,amsmath}
\usepackage{xcolor}
\newtheorem{theorem}{Theorem}
\usepackage[linesnumbered,ruled,vlined]{algorithm2e}
\usepackage{bbm}
\RestyleAlgo{ruled}

\newtheorem{definition}[theorem]{Definition}
\newtheorem{proposition}[theorem]{Proposition}
\newtheorem{corollary}[theorem]{Corollary}

\newtheorem{lemma}[theorem]{Lemma}

\usepackage{multirow}
\allowdisplaybreaks

\usepackage[utf8]{inputenc} 
\usepackage[T1]{fontenc}    
\usepackage{hyperref}       
\usepackage{url}            
\usepackage{booktabs}       
\usepackage{amsfonts}       
\usepackage{nicefrac}       
\usepackage{microtype}      
\usepackage{xcolor}         
\usepackage{graphicx}
\usepackage{enumitem}
\usepackage{subcaption}

\title{Double Auctions with Two-sided Bandit Feedback}

\author{%
   Soumya Basu\\
   Google, Mountain View\\
    \texttt{basusoumya@google.com} \\
    \And
    Abishek Sankararaman\\
    AWS, Santa Clara \\
    \texttt{abishek@utexas.edu} \\
}

\begin{document}

\maketitle

\begin{abstract}
  Double Auction enables decentralized transfer of goods between multiple buyers and sellers, thus underpinning functioning of many online marketplaces. Buyers and sellers compete in these markets through bidding, but do not often know their own valuation a-priori. As the allocation and pricing happens through bids, the profitability of participants, hence sustainability of such markets, depends crucially on learning respective valuations through repeated interactions. We initiate the study of Double Auction markets under bandit feedback on both buyers' and sellers' side. We show with confidence bound based bidding, and `Average Pricing' there is an efficient price discovery among the participants.  In particular, the regret on combined valuation of the buyers and the sellers -- a.k.a. the social regret -- is $O(\log(T)/\Delta)$ in $T$ rounds, where $\Delta$ is the minimum price gap. Moreover, the buyers and sellers exchanging goods attain $O(\sqrt{T})$ regret, individually. The buyers and sellers who do not benefit from exchange in turn only experience $O(\log{T}/ \Delta)$  regret individually in $T$ rounds.  We augment our upper bound by showing that $\omega(\sqrt{T})$ individual regret, and $\omega(\log{T})$ social regret is unattainable in certain Double Auction markets. Our paper is the first to provide decentralized learning algorithms in a two-sided market where \emph{both sides have uncertain preference} that need to be learned.
 
\end{abstract}

\section{Introduction}
Online marketplaces, such as eBay, Craigslist, Task Rabbit, Doordash, Uber, enables allocation of resources  between supply and demand side agents at a scale through market mechanisms, and dynamic pricing. In many of these markets, the valuation of the resources are often personalized across agents (both supply and demand side), and remain apriori unknown. The agents learn their own respective valuations through repeated interactions while competing in the marketplace. In turn, the learning influences the outcomes of the market mechanisms. In a recent line of research, this interplay between learning and competition in markets has been studied in multiple systems, such as  bipartite matching markets~\cite{liu2020competing,liu2021bandit,sankararaman2021dominate,basu2021beyond}, centralized basic auctions~\cite{kandasamy2020mechanism,han2020optimal}. These works follow the `protocol model'~\cite{angluin2007computational}, where multiple agents follow a similar protocol/algorithm, while each agent executes her protocol using only her own observations/world-view up to the point of execution.
   
In this paper, we initiate the study of the decentralized Double Auction market where multiple sellers and buyers, each with their own valuation, trades an indistinguishable good. In each round, the sellers and the buyers present bids for the goods.\footnote{In some literature, the bids of the sellers is called `asks', but we use bids for both sellers and buyers.} The auctioneer is then tasked with creating an allocation, and pricing for the goods. All sellers with bids smaller than the price set by the auctioneer sell at that price, whereas all the buyers with higher bids buy at that price.  Each buyer and seller, is oblivious to all the prices including her own. Only a buyer, or a seller participating in the market observes her own valuation of the good (with added noise). Notably, our work tackles two-sided uncertainty, whereas the previous works mainly focused on one-sided uncertainty in the market.

Double auction is used in e-commerce~\cite{wurman1998flexible} -- including business-to-business and peer-to-peer markets, bandwidth allocation~\cite{iosifidis2009double,iosifidis2014double}, power allocation~\cite{majumder2014efficient}.  We focus on the `Average Mechanism' for double auction. Average mechanism  guarantees that the auctioneer and the auction participants incur no losses in the trade. It also ensures that each commodity is given to the participant that values it the highest, thus maximizing social welfare. Additionally, average mechanism can be implemented through simple transactions. These properties make average mechanism a suitable choice in large social markets, such as  energy markets~\cite{malik2022double}, environmental markets~\cite{muller2002can}, cloud-computing markets~\cite{reza2020cloud,malik2022double}, bidding on wireless spectrum~\cite{feng2012tahes}. Our objective is to design a bandit average mechanism for double auction markets when the participants are a-prioiri unaware of their exact valuation. 

Under average mechanism (detailed in Section~\ref{subsec:average_price}) first an allocation is found, by maximizing $K$ such that the $K$ highest bidding buyers all bid higher than the $K$ lowest bidding sellers. The price is set as the average of the $K$-th lowest bid among the buyers, and the $K$ highest bid among the sellers for the chosen $K$. We have two-sided uncertainty in the market, as both buyers and sellers do not know their own valuation. The uncertainty in bids manifests in two ways. Firstly, each buyer needs to compete with others by bidding high enough to get allotted so that she can discover her own price. Similarly, the sellers compete by bidding lower for price discovery. The competition-driven increase in buyers' bids, and decrease in the sellers' bids may  decrease the utility that a buyer or seller generates. Secondly, as the valuation needs to be estimated, the price set in each round as a function of these estimated valuations (communicated to the auctioneer in the form of bids) remains noisy. This noise in price also decreases the utility. However, when price discovery is slow the noise in price increases. Therefore, the main challenge in decentralized double auction with learning is to strike a balance between the competition-driven increase/decrease of bids, and controlling the noise in price.  
\subsection{Main Contributions}
Our main contributions in this paper are as follows.

1. Our paper is the first to provide decentralized learning algorithms in a two-sided market where \emph{both sides have uncertain preference} that need to be learned. Unlike in the setting with one sided uncertainity only, we identify that with {\em two-sided uncertainty} in double auction markets {\em optimism} in the face of uncertainty in learning (OFUL) from both sides causes {\em information flow bottleneck} and thus not a good strategy. We introduce the notion of \emph{domination of information flow} -- that increases the chance of trade and price discovery. The sellers bid the lower confidence bound (LCB), and buyers bid the upper confidence bound (UCB) of their respective valuation. By using UCB bids the buyers, and using the LCB bids the sellers decrease their reward and facilitate price discovery. Formally, with the above bids under average mechanism
\begin{itemize}[leftmargin=1em,noitemsep,topsep=0pt]
    \item  We show that the social welfare regret, i.e. the regret in the combined valuation of all the sellers and the buyers in the market is $O(\log(T)/\Delta)$ regret in $T$ rounds with a minimum reward gap $\Delta$. We also show a $\Omega(\log(T)/\Delta)$ lower bound on social-welfare regret and thus our upper bound is order-optimal.
    \item For individual regret, we show that each of the the sellers and the buyers that do not participate under the true valuations incur $O(\log(T)/\Delta)$ regret, while the optimal participating buyers and sellers incur a $O(\sqrt{T\log(T)})$ regret. Our upper bound holds for heterogeneous confidence widths, making it robust against the choices of the individual agents.
\end{itemize}

2. We complement the upper bounds by showing price discovery itself is $\Omega(\sqrt{T})$ hard in the minimax sense. Specifically, we consider a relaxed system where {\em (i)} the price of the good is known to all, and {\em (ii)} an infinite pool of resource exists, and hence any buyer willing to pay the price gets to buy, and any seller willing to sell at the price does so. We show under this setup, for any buyer or seller, there exists a system where that agent must incur a regret of $\Omega(\sqrt{T})$. Similarly, we establish a $\Omega(\log(T))$ lower bound for the social-welfare regret by showing that the centralized system can be reduced to a combinatorial semi-bandit and using the results of \cite{combes2015combinatorial}. 
\section{System Model}
\label{sec:system_model}

The market consists of $N \geq 1$ buyers and $M \geq 1$ sellers, trading a {\em single type} of item which are indistinguishable across sellers. This set of $M+N$ market participants, repeatedly participate in the market for $T$ rounds. Each buyer $i \in [N]$ has valuation $B_i \geq 0$, for the item and each seller $j \in [M]$ has valuation $S_i \geq 0$. No participant knows of their valuation apriori and learn it while repeatedly participating in the market over $T$ rounds.

\subsection{Interaction Protocol}
\label{subsec:interaction_protocol}
The buyers and sellers interact through an auctioneer who implements a bilateral trade mechanism at each round. 

At each round $t \geq 1$, every buyer $i \in [N]$ submits bids  $b_i(t)$ and seller $j \in [M]$ submits  asking price \footnote{Throughout, we refer to sellers `bids' as their asking price} $s_j(t)$ simultaneously. Based on the bids and asking prices in round $t$, the auctioneer outputs {\em (i)} subsets  $\mathcal{P}_b(t) \subseteq [N]$ and $\mathcal{P}_s(t) \subseteq [M]$ of participating buyers and sellers with equal cardinality $K(t) \leq \min(M,N)$, and {\em (ii)} the trading price $p(t)$ for the participating buyers and sellers in this round. Subsequently, every buyer $i\in[N]$  is {\em(i)} either part of the trade at time $t$, in which case she gets utility  $r_i^{(B)}(t) :=B_i + \nu_{b,i}(t) - p(t)$, or {\em(ii)} is not part of the trade at time $t$ and receives $0$ utility along with a signal that she did not participate. Similarly, each seller $j \in [M]$ is  either part of the trade and receives utility $r_j^{(S)}(t) := p(t)-S_j - \nu_{s,j}(t)$, or  is informed she is not part of the trade and receives $0$ utility. 

The random variables $\nu_{b,i}(t)$ and $\nu_{s,j}(t)$ for all $i \in [N]$, $j \in [M]$ and  $t \in [T]$ are i.i.d., $0$ mean, $1$ sub-Gaussian random variables.\footnote{We study the system with $1$-sub-Gaussian to avoid clutter. Extension to general $\sigma$-sub-Gaussian is trivial when $\sigma$ or an upper bound to it is known.} 

\subsection{Average price mechanism}
\label{subsec:average_price}

Throughout the paper, we assume the auctioneer implements the average price mechanism in every round $t$. Under this mechanism, at each round $t$, the auctioneer orders the bids by the `natural order', i.e., sorts the buyers bids in descending order and the seller's bids in ascending order. Denote by the sorted bids from the buyer and seller as $b_{i_1}(t) \geq \cdots b_{i_N}(t)$ and the sorted sellers bids by $s_{j_1}(t) \leq \cdots s_{j_M}(t)$. Denote by the index $K(t)$ to be the largest index such that $b_{i_{K(t)}}(t) \geq s_{j_{K(t)}}(t)$. In words, $K(t)$ is the `break-even index' such that all buyers $i_1, \cdots, i_{K(t)}$ have placed bids offering to buy at a price strictly larger than the asking price submitted by sellers $j_1, \cdots, j_{K(t)}$. The auctioneer then selects the participating buyers $\mathcal{P}_b(t) = \{i_1, \cdots, i_{K(t)}\}$, and participating sellers $\mathcal{P}_s(t) = \{j_1, \cdots, j_{K(t)}\}$. The price is set to $p(t) := \frac{b_{i_{K(t)}}+s_{j_{K(t)}}}{2}$, and thus the name of the mechanism is deemed as the average mechanism.
\begin{wrapfigure}{r}{0.35\textwidth}
  \vspace{-2em}
  \begin{center}
    \includegraphics[width=0.33\textwidth]{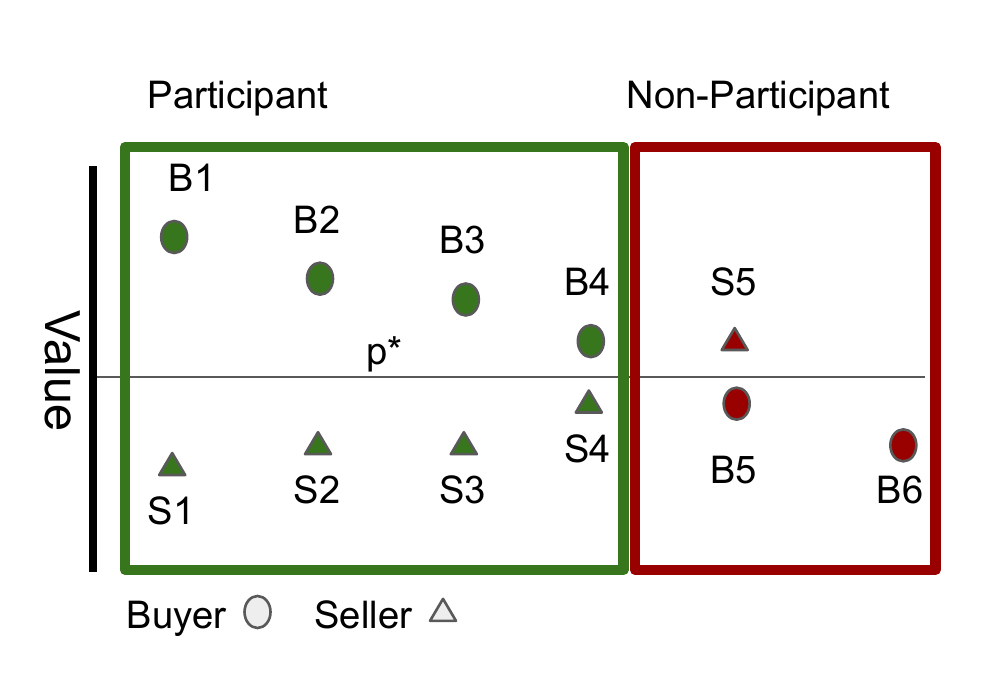}
  \end{center}
  \label{fig:double_auction}
  \caption{Average Mechanism with 6 Buyers and 5 Sellers}
  \vspace{-3.3em}
\end{wrapfigure}

\subsection{Regret definition} 
\label{subsec:regret_defn}
For the given bilateral trade mechanism, and true valuations $(B_i)_{i\in[N]}$ and $(S_j)_{j \in [M]}$, denote by $K^* \leq \min(M,N)$ be the number of matches and by $p^{*}$ to be the price under the average mechanism when all the buyers and sellers bid their true valuations.  Let $\mathcal{P}_b^{*}$ to be set of the {\em optimal participating buyers}, and  $\mathcal{P}_s^{*}$ to be set of the {\em optimal participating sellers}.
For any buyer $i \in [N]$, we denote by $(B_i - p^{*})$ to be the true utility of the buyer. Similarly, for any seller $j \in [N]$, we denote by $(p^{*} - S_j)$ to be the true utility of seller $j$. From the description of the average mechanism, in the system with true valuations, all participating agents have non-negative true utilities.

Recall from the protocol description in Section \ref{subsec:interaction_protocol} that at any time $t$, if buyer $i \in [N]$ participates, then she receives a mean utility of $(B_i - p(t))$. For a participating seller $j \in [M]$ her mean utility of $(p(t) - S_j)$ in round $t$. If in any round $t$, if a buyer $i \in [N]$ or a seller $j\in [M]$ does not participate, then she receives a deterministic utility $0$. The expected individual regret of a buyer $i$, namely $R_{b,i}(T)$, and a seller $j$, namely $R_{s,j}(T)$, are defined as 

\begin{align*}
\vspace{-1em}
    &R_{b,i}(T) 
    = T(B_i - p^*)\mathbbm{1}(i\leq K^*) - \mathbb{E}\Big[\sum_{t: i\in \mathcal{P}_b(t)} (B_i - p(t)) \Big], \\ 
    &R_{s,j}(T) 
    = T(p^* - S_j)\mathbbm{1}(j\leq K^*) - \mathbb{E}\Big[\sum_{t: j\in \mathcal{P}_s(t)} ( p(t) - S_j)\Big].
\vspace{-2em}
\end{align*}
Auctioneer has no regret as average mechanism is {\em budget balanced}, i.e. auctioneer does not gain or lose any utility during the process. 

We also define the social welfare regret similar to gain from trade regret in~\cite{cesa2021regret}. The social welfare is defined as the total valuation of the goods after the transfer of goods from seller to buyer in each round. The expected (w.r.t reward noise) total valuation after transfer is thus defined as $\left(\sum_{i\in \mathcal{P}_b(t)} B_i + \sum_{j\in [M]\setminus\mathcal{P}_s(t)} S_j\right)$, while the the expected total valuation under oracle average mechanism is $\left(\sum_{i\in \mathcal{P}_b^*} B_i + \sum_{j\in [M]\setminus\mathcal{P}_s^*} S_j\right)$. Therefore, the expected social welfare regret is defined as 
\begin{gather}
    R_{SW}(T) = T\big(\sum_{i\in \mathcal{P}_b^*} B_i + \sum_{j\in [M]\setminus\mathcal{P}_s^*} S_j\big)  - \mathbb{E}\Big[\sum_{t=1}^{T}\big(\sum_{i\in \mathcal{P}_b(t)} B_i + \sum_{j\in [M]\setminus\mathcal{P}_s(t)} S_j\big)\Big].
    \label{eqn:sw_regret_defn}
\end{gather}

\section{Decentralized Bidding for Domination of Information Flow}
\label{sec:bidding_policy}

We consider the decentralized system where each market participant bids based on their own observation, without any additional communication. The core idea is balancing {\em domination of information flow} and {\em over/under bidding}, i.e. ensuring  the number of allocation is not less than $K^*$ in each round with high probability, and the bids converge to each agent's true valuation. 

Each seller $j\in [M]$, with $n_{s,j}(t)$ participation upto round $t$, at time $t+1$ bids the lower confidence bound (scaled by $\alpha_{s,j}$), LCB($\alpha_{s,j}$) in short, of its own valuation of the item. Each buyer $i\in [N]$, with $n_{s,j}(t)$ participation upto round $t$, at time $t+1$, bids the upper confidence bound (scaled by $\alpha_{b,i}$), UCB($\alpha_{b,i}$) in short,  of its own valuation of the item. The bids are specified in  Equation~\ref{eq:bids}.
\begin{align}\label{eq:bids}
s_{j}(t+1) = \hat{s}_{j}(t) - \sqrt{\frac{\alpha_{s,j} \log(t)}{n_{s,j}(t)}},\quad b_{j}(t+1) = \hat{b}_{i}(t) + \sqrt{\frac{\alpha_{b,i} \log(t)}{n_{b,i}(t)}}
\end{align}
Here, $\hat{s}_{j}(t) = \tfrac{1}{n_{s,j}(t)}\sum_{t' \leq t:j \in \mathcal{P}_s(t')} Y_{s,j}(t')$ , and  $\hat{b}_{i}(t) = \tfrac{1}{n_{b,i}(t)}\sum_{t'\leq t:i \in \mathcal{P}_b(t')} Y_{b,i}(t')$ are the observed empirical valuation of the item upto time $t$ by seller $j$, and  buyer $i$, respectively.

Our buyers and sellers follow the {\em protocol model} (which is ubiquitous in bandit learning for markets)  and agree on UCB and LCB based bids, respectively. 
They are {\em heterogeneous} as they may use different $\alpha_{b,i}$ and $\alpha_{s,j}$ scaling parameters. The only restriction (as seen in Section~\ref{sec:analysis}) is $\min\{\alpha_{b,i}, \alpha_{s,j} \} \geq 4$ which they agree on as part of the protocol.  

\textbf{Key Insights: } We now contrast our algorithm design from standard multi-armed-bandit (MAB) problems. In a typical MAB problems, including other multi-agent settings, algorithm is designed based on optimism in the face of uncertainty in learning (OFUL) principle~\cite{abbasi2011improved}. The UCB-type indices under OFUL arises as optimistic estimate of the rewards of arms/actions. However, such optimism used from both sides, i.e. both buyer and seller using UCB indices, may lead to a standstill. The bids of the buyers with constant probability can remain below the sellers' bids. Instead, we emphasize information flow through trade. In our algorithm buyers' UCB bids and sellers' LCB bids ensure that the system increases the chance for buyer and seller to participate in each round as compared to using their true valuations. There is {\em domination of information flow}, and consequently they discover their own valuation in the market.

However, too aggressive over or under bidding can disrupt the price setting process. In particular, if even one non-participating buyer is bidding high enough (due to UCB bids) to exceed a non-participating seller's price, firstly she participates and accrues regret. More importantly, the participating sets deviates, resulting in a deviation of the price of the good from $p^*$. Thus resulting in regret for all participating agents as well. Similar problems arise if one or more non-participating seller predicts lower.  On the other hand, too low aggression is also harmful as the price discovery may not happen resulting in deviation of participating set, deviation of price, and high regret.  In the next section, we show the aggression remains within desired range, i.e. the regret of the agents remain low, even with heterogeneous UCBs and LCBs. 
\section{Regret Upper Bound}\label{sec:analysis}

In this section, we derive the regret upper bound for all the buyers and sellers in the system. Without loss of generality, let us assume that the buyers are sorted in decreasing order of valuation $B_1\geq B_2\geq ...\geq B_N$. The sellers are sorted in increasing order of valuation $S_1\leq S_2\leq ...\leq S_N$. The the buyers $i=1,\dots, K^*$ are the optimal participating buyers, and, similarly, the sellers $j=1,\dots, K^*$ are the optimal participating sellers. Let us define $\alpha_{\min}:=\min\{\alpha_{b,i}, \alpha_{s,j}\}$, and $\alpha_{\max}:=\max\{\alpha_{b,i}, \alpha_{s,j}\}$. We define the minimum distance of an agent's true valuation from the true price $p^*$ as 
$\Delta = \min_{i, \in[N], j \in [M]}\{|p^* - S_j|, |B_i - p^*|\}.$

 Our first main result, Theorem~\ref{thm:upper_social}, proves a $O(\log(T)/\Delta)$ upper bound for the social welfare regret. 
\begin{theorem} \label{thm:upper_social}
The expected social welfare regret of the Average mechanism with buyers bidding UCB($\boldsymbol{\alpha}_{b}$), and sellers bidding LCB($\boldsymbol{\alpha}_{s}$) of their estimated valuation, for $\alpha_{\min} \geq 4$,  is bounded as: 
\begin{multline*}
    R_{SW}(T) \leq \sum_{i \leq K^*}\sum_{i' > K^*} \frac{(\sqrt{\alpha_{\max}} + 2)^2}{(B_{i} - B_{i'})} \log(T)  + \sum_{j\leq K^*}\sum_{j' > K^*}\frac{(\sqrt{\alpha_{\max}} + 2)^2}{(S_{j'} - S_{j})} \log(T) 
    \quad\quad \\+ \sum_{j'> K^*}\sum_{i' > K^*}\frac{(\sqrt{\alpha_{\max}} + 2)^2}{(S_{j'} - B_{i'})} \log(T) 
    + MN b_{max} \pi^2/6.
\end{multline*}
\end{theorem}
\vspace{-1em}

The following theorem provides the individual regret bounds for all sellers and buyers in $T$ rounds.
\begin{theorem} \label{thm:upper_main_paper}
The expected regret of the Average mechanism with buyers bidding UCB($\boldsymbol{\alpha}_{b}$), and sellers bidding LCB($\boldsymbol{\alpha}_{s}$) of their estimated valuation, for $\alpha_{\min} \geq 4$, is bounded as: 
\begin{itemize}[leftmargin=*]
\setlength{\itemsep}{0pt}
    \item for a participating buyer $i \in [K^*]$ as 
    $R_{b,i}(T) \leq (2 + \sqrt{\alpha_{\max}})\sqrt{T\log(T)} + C_{b', i} \log(T)$,
    \item for a participating seller $j \in [K^*]$ as
    $R_{s,j}(T) \leq (2 + \sqrt{\alpha_{\max}})\sqrt{T\log(T)} + C_{s', j} \log(T)$,
    \item for a non-participating buyer $i \geq (K^*+1)$ as 
    $R_{b,i}(T) \leq \frac{\sqrt{(M-K^*+1)} (2 + \sqrt{\alpha_{\max}})^2}{(B_{K^*} - B_i)} \log(T)$,
    \item for a non-participating seller $j \geq (K^*+1)$ as 
    $R_{s,j}(T) \leq \frac{\sqrt{(N-K^*+1)} (2 + \sqrt{\alpha_{\max}})^2}{(S_j - S_{K^*})} \log(T)$.
\end{itemize}
\vspace{-1em}
Here $C_{b',i}$ and $C_{s',j}$  are $O\left( \frac{(M-K^*+1)(N-K^*+1)}{\Delta} \right)$ constants (see Theorem \ref{thm:upper_main} in Appendix~\ref{sec:app_upper}).
\label{thm:main_paper_upper}
\end{theorem}

We can summarize our main results (with lower bounds taken from Section~\ref{sec:lower_bound}) as
\begin{table}[!h]
\begin{center}
\small
\begin{tabular}{|c|c|c|c|c|c|}
\hline
  \textbf{Regret} & \textbf{Social} & \textbf{Participant} & \textbf{Non-participant buyer} &  \textbf{Non-participant  seller}\\ 
  \hline
 \textbf{Upper}  & $O(\tfrac{MN}{\Delta}\log(T))$ & $O(\sqrt{T}\log(T))$ & $O(\tfrac{\sqrt{M}}{\Delta}\log(T))$  & $O(\tfrac{\sqrt{N}}{\Delta}\log(T))$\\ 
 \hline
 \textbf{Lower}  & $\Omega(\tfrac{M+N}{\Delta}\log(T))$ & $\Omega(\sqrt{T})$ & $\Omega(\log(T))$ & $\Omega(\log(T))$ \\ 
 \hline
\end{tabular}
\end{center}
\caption{\label{tab:bounds}\small Regret bounds derived in this paper. The dependence on $\Delta, M,N$ is not tight in the lower bound for non-participants. }
\vspace{-2em}
\end{table}

Several comments on our main results are in order.

\textbf{Social and Individual Regret:} The social regret as well as the individual regret of the optimal non-participating buyers and sellers under the average mechanism with confidence bound bids grow as $O(\log(T))$. However, the individual regret of the optimal participating agents grow as $O(\sqrt{T\log(T)})$. The social regret, is determined by how many times a participant buyers fails to participate, and non-participants end up participating. Also, the non-participants incur individual regret through participation. The effect of price setting is not present in both the cases, and counting the number of bad events lead to $O(\log(T))$ regret. For the participant agents the error in price-setting dominates their respective individual regret of  $O(\sqrt{T\log(T)})$.

\textbf{Individually rational for $\boldsymbol{\Delta = \omega(\sqrt{\log(T)/T})}$:}
We observe that the non-participating buyers and sellers, are only having individual regret $O(\log(T)/\Delta)$. And the participating buyers and sellers are  having individual regret $O(\sqrt{T\log(T})$. This is reassuring as this does not discourage buyers and sellers from participation for a large range of system a-priori. Indeed, for the last participating buyer and seller the utility is $(B_K^* - S_K^*)/2$, which is close to $\Delta$. Hence, as long as $\Delta T = \omega(\max\{\log(T)/\Delta, \sqrt{T\log(T)}\})$ or  $\Delta  = \omega(\sqrt{\log(T)/T})$ a non-participating buyer or seller prefers entering the market then discovering her price and getting out, as compared to not participating in the beginning. Also, participating buyer or seller is guaranteed return through participation.

\textbf{Incentives and Deviations:} We now discuss the incentives of individual users closely following the notion of symmetric equilibrium in double auction~\cite{wilson1985incentive}. In this setting, a {\em myopic agent} who knows her true valuation, and greedily maximizes her reward in each round assuming all the non-strategic agents use confidence-based bidding. For average mechanism only the price-setting agents (i.e. the $K^*$-th buyer and the $K^*$-th seller) have incentive to deviate from their true valuation to increase their single-round reward. The \emph{non-price-setting agents are truthful}. When the $K^*$-th buyer deviates then each participating buyer has an average per round surplus of  $(B_{K^*} - \max(S_{K^*}, B_{K^*+1}))/2$, and each participating seller has the same deficit. On the contrary, when 
the $K^*$-th seller deviates each participating seller has average per round surplus of $(S_{K^*} - \min(B_{K^*}, S_{K^*+1}))/2$ surplus, and each participating buyer has the same average deficit in each round.  See Appendix~\ref{appendix:incentives} for more discussions.
\textbf{Scaling of  Regret with $N$, $M$, $K^*$, $\Delta$:} 

\emph{Social regret:}
The social regret scales as $O((MN - (K^*)^2) \log(T)/\Delta)$. When participant buyers are replaced by non-participant buyers the valuation decreases which leads to $O(K^*(N-K^*) \log(T)/\Delta)$ regret. Similarly, participant seller getting replaced by non-participant seller introduces $O(K^*(M-K^*) \log(T)/\Delta)$ social regret. Finally, as  goods move from non-participant buyers to non-participant sellers we obtain $O((M-K^*)(N-K^*) \log(T)/\Delta)$ regret.

\emph{Non-participant regret:} 
The individual regret scales as $O(\sqrt{M} \log(T)/\Delta)$ for a non-participant buyer, whereas it scales as $O(\sqrt{N} \log(T)/\Delta)$ for a non-participant seller. For any non-participant buyer once it has $O(\log(T)/\Delta^2)$ samples it no longer falls in top $K^*$ buyer. However, it can keep participating until each non-participant seller collects enough samples, i.e. $O(\log(T)/\Delta^2)$ samples. Finally, regret is shown to be $O(\sqrt{\# \text{participation} \log(T)})$ which leads to the $O(\sqrt{M} \log(T)/\Delta)$. 

\emph{Participants:} The leading $O(\sqrt{T\log(T)})$ term in the regret for each optimal participating buyer and seller do not scale with the size of the system. This leading term depends mainly on the random fluctuation of the bid of the lowest bidding participating buyer and the highest bidding participating seller. The $O\left( \frac{MN}{\Delta} \log(T)\right)$ regret for the participating buyer and seller comes because each time a non-participating buyer or seller ends up participating the price deviates.   

\emph{Special Case $K^* \approx N \approx M$:} We see that when the system the number of participants is very high, i.e. $(N - K^*), (M-K^*) = O(1)$, then the $O(\log(T)/\Delta)$ component, hence the regret, per buyer/seller does not scale with the system size.  Furthermore, the social regret also scales as $O((MN-(K^*)^2)\log(T)/\Delta) \approx O(\log(T)/\Delta)$. This indicates that there is rapid learning, as all the participants are discovering her own price almost every round.  

\subsection{Proof Outline of Regret Upper Bounds}
We now present an outline to the proof of Theorem~\ref{thm:main_paper_upper}. The full proof is given in Appendix~\ref{sec:app_upper}. A salient challenge in our proof comes from {\em two-sided uncertainty}. In Double auction, the outcomes of the buyers' and sellers' sides are inherently coupled. Hence error in buyers' side propagates to the sellers' side, and vice versa, under two-sided uncertainty. For example, if the buyers are bidding lower than their true valuation then the optimal non-participating sellers, even with perfect knowledge of their own valuation, ends up participating. We show that optimal participants (both buyers and sellers) get decoupled, whereas in optimal non-participants regret leads to new information in both buyer and seller side. This breaks the two-sided uncertainty obstacle.

\textbf{Monotonicity and Information Flow:} Our proof leverages {\em monotonicity} of the average mechanism, which means if the bid of each buyer is  equal or higher, and simultaneously bid of each seller is equal or lower, then the number of participants increases (Proposition \ref{prop:atleastK_star}). Additionally, UCB ensures each buyer w.h.p. bids higher than its true value, and LCB ensures each seller w.h.p. bids lower than its true value. 
Therefore, we have \emph{domination of the information flow}: $K(t)\geq K^*$ w.h.p. for all $t \geq 1$.

\textbf{Social Regret Decomposition:} For any sample path with $K(t) \geq K^*$  the  social regret, namely $r_{SW}(T)$, under average mechanism can be bounded as 
\vspace{-0.5em}
\begin{multline*}
    r_{SW}(T) \leq \sum_{i \leq K^* < i'}(B_i \mathtt{-} B_{i'})\sum_{t=1}^{T}\mathbbm{1}(i\mathtt{\notin}\mathcal{P}_b(t), i'\mathtt{\in} \mathcal{P}_b(t)) \\
     + \sum_{j\leq K^* < j'}(S_{j'} \mathtt{-}  S_j)\sum_{t=1}^{T}\mathbbm{1}(j\mathtt{\notin}\mathcal{P}_s(t), j'\mathtt{\in} \mathcal{P}_s(t)) +  \sum_{i', j'> K^*}(S_{j'} \mathtt{-}  B_{i'})\sum_{t=1}^{T}\mathbbm{1}(j'\mathtt{\in}\mathcal{P}_s(t), i'\mathtt{\in} \mathcal{P}_b(t)).
\end{multline*}
The first term corresponds to a non-participant buyer replacing a participant buyer. The second term is the same for sellers, whereas the final term corresponds to two non-participant buyer and seller getting matched.

\textbf{Individual Regret Decomposition:} We now turn to the individual regrets. Regret of a non-participant buyer $i$, can be bounded by 
\begin{multline*}
\sum_{t: i \in \mathcal{P}_b(t)} (p(t) - B_i)
\leq \sum_{t: i \in \mathcal{P}_b(t)} (b_i(t) - B_i) \lessapprox \sum_{t: i \in \mathcal{P}_b(t)} \sqrt{\tfrac{\alpha_{b,i}\log(t)}{n_{b,i}(t)}} \lessapprox \sqrt{\alpha_{b,i}n_{b,i}(T)\log(T)}.
\end{multline*}
We have $p(t)$ lesser than $b_i(t)$ because $i$-th buyer participates in round $t$. By UCB property w.h.p. $b_i(t)$ is at most $\sqrt{\tfrac{\alpha_{b,i}\log(t)}{n_{b,i}(t)}}$ away from $B_i$. A similar argument shows for a non-participating seller $j$ the regret is roughly $\sqrt{\alpha_{s,j}n_{s,j}(T)\log(T)}$.

For a participating buyer $i$, we have the regret bounded as 
\begin{multline*}
  \sum_{t: i \notin \mathcal{P}_b(t)} (B_i - p^*) + \sum_{t: i \in \mathcal{P}_b(t)} (p(t) - p^*) 
  = (T - n_{b,i}(T))(B_i - p^*) + \sum_{t: i \in \mathcal{P}_b(t)} (p(t) - p^*).
\end{multline*}
Similarly, a  participating seller $j$ has regret bound $(T - n_{s,j}(T))(p^* - S_j) + \sum_{t: j \in \mathcal{P}_s(t)} (p^* - p(t)).$

\textbf{Decoupling Participants:} In Lemma~\ref{lemm:topKstar_v2} in Appendix~\ref{sec:app_upper}, we show that learning is decoupled for the optimal participating buyers, and optimal participating sellers. It lower bounds the number of participation for optimal participating buyers and sellers.
\begin{lemma}[Informal statement of Lemma \ref{lemm:topKstar_v2}]
For $\alpha_{\min} > 4$, w.h.p., for every $i, j\in [K^*]$, and  $i', j' > K^*$  
\begin{align*}
 &(T- n_{b,i}(T)) \lessapprox \sum\limits_{i''\geq K^*+1} \tfrac{\alpha_{b,i''}\log(T)}{(B_{i} - B_{i''})^2}, (T- n_{s,j}(T))\lessapprox \sum\limits_{j''\geq K^*+1} \tfrac{\alpha_{s,j''}\log(T)}{(S_{j''}-S_{j})^2},\\
 &\sum_{t=1}^{T}\mathbbm{1}(i\mathtt{\notin}\mathcal{P}_b(t), i'\mathtt{\in} \mathcal{P}_b(t)) \lessapprox \tfrac{\alpha_{b,i'}\log(T)}{(B_{i} - B_{i'})^2},
 ~\text{and} 
 \sum_{t=1}^{T}\mathbbm{1}(j\mathtt{\notin}\mathcal{P}_s(t), j'\mathtt{\in} \mathcal{P}_s(t)) \lessapprox \tfrac{\alpha_{b,j'}\log(T)}{(S_{j'} - S_{j})^2}.
\end{align*}
\end{lemma}
\vspace{-1em}
It argues after a optimal non-participant $i''$ gets $O(\tfrac{\log(T)}{(B_{i} - B_{i''})^2})$ samples it can not participate while $i$ does not participate as the bid of $i$ is higher than the bid of $i''$ with high probability. Similarly seller side result follow.

\textbf{Non-Participant regret leads to two-sided Learning:} Unlike optimal participating agents, the effect of uncertainties on the optimal non-participating buyers and sellers cannot be decoupled directly.  With at least one non-participating buyer with large estimation error in her valuation present, the non-participating sellers can keep participating even with perfect knowledge. However, this does not ensure directly that the estimation error of this non-participating buyer decreases. Next, in Lemma~\ref{lemm:belowKstar} in Appendix~\ref{sec:app_upper}, we upper bound the number of times a non-participant can participate. Informally we have
\begin{lemma}[Informal]
For $\alpha_{\min} > 4$, with high probability for any $i, j \geq (K^*+1)$,
\begin{align*}
    n_{b,i}(T) &\lessapprox \tfrac{\alpha_{b,i}\log(T)}{(B_{K^*} - B_i)^2} + \sum_{j'\geq (K^*+1)} \tfrac{\alpha_{s,j'} \log(T)}{(S_{j'} - B_{i})^2}, 
    n_{s,j}(T) &\lessapprox \tfrac{\alpha_{s,j}\log(T)}{(S_{j} - S_{K^*})^2} + \sum_{i'\geq (K^*+1)} \tfrac{\alpha_{b,i'}\log(T)}{(S_{j} - B_{i'})^2},
\end{align*}
and   $\sum_{t=1}^{T}\mathbbm{1}(j'\in \mathcal{P}_s(t), i'\in \mathcal{P}_b(t)) \lessapprox \tfrac{\alpha_{\max}\log(T)}{(S_{j'} - B_{i'})^2}$.
\end{lemma}

A non-participant buyer $i$ after obtaining $O(\tfrac{\log(T)}{(B_{K^*} - B_i)^2})$ samples does not belong to top $K^*$ w.h.p. Hence, she participates with non-negligible probability only if a seller $j' \geq (K^*+1)$ participates. However, this implies that with each new match of buyer $i$, additionally at least one non-participating seller is matched, decreasing both their uncertainties. For this buyer $i$, we argue such spurious participation happens a total of $\sum_{j'\geq (K^*+1)} O(\tfrac{\log(T)}{(S_{j'} - B_{i})^2})$ times. After that all the non-participating sellers $j\geq (K^*+1)$ will have enough samples so that their LCB bids will separate from the $i$-th buyer's UCB bid. Reversing sellers' and buyers' roles does the rest. 

\textbf{Bounding Price Difference:} The final part of the proof establishes bound on the cumulative difference of price from the true price $p^*$ (see, Lemma~\ref{lemm:price_bound} in Appendix~\ref{sec:app_upper}). 
\begin{lemma}(Informal statement of Lemma \ref{lemm:price_bound})
For $\alpha_{\min} > 4$, w.h.p.
$
\sum_{t=1}^{T} |p(t) - p^*| \lessapprox C\log(T) + \sqrt{\alpha_{max}T\log(T)},
$
where  $C = O\big(\tfrac{(M-K^*)(N-K^*)}{\Delta}\big)$.
\end{lemma}
Let us focus on the first upper bound, i.e. of the cumulative value of $(p(t) - p^*)$. The proof breaks down the difference into two terms,
$
2 (p(t) - p^*)  = (\min_{i \in \mathcal{P}_b(t)}b_{i}(t) - B_{K^*}) +  (S_{K^*} - \max_{j \in \mathcal{P}_s(t)}s_{j}(t)).
$
For rounds when the buyer $K^*$ is present (which happens all but $O\big(\tfrac{(N-K^*)\log(T)}{(B_i - p^*)^2}\big)$ rounds) we can replace $\min_{i \in \mathcal{P}_b(t)}b_{i}(t)$ with $b_{K^*}(t)$. Finally noticing that $\sum_{t}(b_{K^*}(t) - B_{K^*}) \lessapprox \sqrt{\alpha_{max}}\sqrt{T\log(T)}$ takes care of the first term. For the second term, we need to study the process $\max_{j \in \mathcal{P}_s(t)}s_{j}(t)$. First we bound the  number of times sellers $1$ to $(K^*-1)$ crosses the seller $K^*$. Next we eliminate all the rounds where at least one seller $j\geq (K^*+1)$ are participating. Such an elimination comes at a cost of $O\big(\tfrac{(M-K^*)}{\Delta}\big)$ For any seller $j\geq (K^*+1)$ this happens  $O\Big(\tfrac{(N-K^*)\log(T)}{\Delta_j^2}\Big)$ times for some appropriate $\Delta_j$, and gives $\Delta_j$ regret in each round. This final step gives us the dominating $O\big(\tfrac{(M-K^*)(N-K^*)}{\Delta}\big)$ term.
The bound for the cumulative value of $(p^* - p(t))$ follows analogously.

\section{Lower Bounds}
\label{sec:lower_bound}

\subsection{$\Omega(\sqrt{T})$ minimax lower bound on individual regret}
\label{subsec:simple_system_lb}
We show a minimax regret lower bound of $\Omega(\sqrt{T})$ in Lemma \ref{lem:lb_exact} by considering a simpler system that decouples learning and competition.  In this system, the seller is assumed to {\em(i)} know her exact valuation, and {\em(ii)} always ask her true valuation as the selling price, i.e., is truthful in her asking price in all the rounds. Furthermore, the pricing at every round is fixed to the average $p_t = \frac{B_t + S_t}{2}$ in the event that $B_t \geq S$. We show in Corollary \ref{cor:red_two_arm_bandit} through a coupling argument that any algorithm for the classical two armed bandit problem can be converted to an algorithm for this special case. Then we use well known lower bounds for the bandit problem to give the $\Omega(\sqrt{T})$ lower bound in Lemma \ref{lem:lb_exact}. All technical details are in Appendix~\ref{sec:appendix_lower_bounds}.

\subsection{$\Omega(\log(T))$  instance dependent lower bound on Social Welfare Regret}
\label{sec:sw_regret_lb}
The key observation is that social-welfare regret in Equation (\ref{eqn:sw_regret_defn}) is \emph{independent} of the pricing mechanism and only depends on the participating buyers $\mathcal{P}_b(t)$ and sellers $\mathcal{P}_s(t)$ at each time $t$. We will establish a lower bound on a centralized decision maker (DM), who at each time, observes all the rewards obtained by all agents thus far, and decides $\mathcal{P}_b(t)$ and $\mathcal{P}_s(t)$ for each $t$. In Appendix~\ref{sec:reduction_to_semi_bandit}, we show that the actions of the DM can be coupled to that of a combinatorial semi-bandit model \cite{combes2015combinatorial}, where the base set of arms are the set of all buyers and sellers, the valid subset of arms are those subsets having an equal number of buyers and sellers and the mean reward of any valid subset $\mathcal{A} \subseteq 2^{\mathcal{D}}$ is the difference between the sum of all valuations of buyers in $\mathcal{A}$ and of sellers in $\mathcal{A}$. In Appendix~\ref{sec:appendix_lower_bounds} we exploit this connection to semi-bandits to give a $\Omega(\log(T))$ regret lower bound for the centralized DM. Thus our upper bound of $O(\log(T))$ social welfare regret under the decentralized setting is order optimal since even a centralized system must incur $\Omega(\log(T))$ regret. 

\section{Simulation Study}\label{sec:expt}
We perform synthetic studies to augment our theoretical guarantees. For a fixed system of $N$ buyers, $M$ sellers,  $K^*$ participants, and $\Delta$ gap, the rewards are Bernoulli, with means themselves chosen uniformly at random. We vary the confidence width of the buyers, $\alpha_b$, and seller, $\alpha_s$, in $[\alpha_1, \alpha_2]$. Next we simulate the performance of the UCB($\alpha_b$) and LCB($\alpha_s$) over $100$ independent sample paths with $T= 50k$. We report the mean, 25\% and  75\% value of the trajectories. We plot the cumulative regret of the buyers, $R_{b,i}(t)$,  and the sellers, $R_{s,j}(t)$, the number of matches in the system $K(t)$, and the price difference $(p(t) - p^*)$.  In Figure~\ref{fig:Fig885app} in appendix, we have a $8\times 8$ system with $K^* = 5$. We see that $K(t)$ converges to $5$, where as $(p(t) - p^*)$ converges to $0$. The social regret grows as  $log(T)$. The particpant and non-participant individual regret of this instance is presented. We  further study behavior of heterogeneous $\alpha$, varying gaps, and different system sizes in Appendix~\ref{appendix:experiment}.

\section{Related Work}\label{sec:related}

\textbf{Classical mechanism design in double auctions:} There is a large body of work on mechanism design for double auctions, following Myerson et al. \cite{myerson1983efficient}. The average-mechanism, which is the subject of focus in this paper achieves all the above desiderata except for being incentive compatible. The VCG mechanism was developed in a series of works \cite{vickrey1961counterspeculation}, \cite{clarke1971multipart} and \cite{groves1973incentives} achieves all desiderata except being budget balanced. This mechanism requires the auctioneer to subsidize the trade. More sophisticated trade mechanisms known as the McAfee mechanism \cite{mcafee1992dominant}, trade reduction, and the probabilistic mechanism \cite{probabilistic_mechanism} all trade-off some of the desiderata for others. However, the key assumption in all of these lines of work was that all participants know their own valuations, and do not need to learn through repeated interactions.

\textbf{Bandit learning in matching markets:} In recent times, online learning for the two-sided matching markets have been extensively studied in \cite{liu2020competing}, \cite{liu2021bandit}, \cite{sankararaman2021dominate}, \cite{basu2021beyond}, \cite{dai2021learning}. This line of work studies two sided markets when one of the side does not know of their preferences apriori and learn it through interactions. However, unlike the price-discovery aspect of the present paper, the space of preferences that each participant has to learn is discrete and finite, while the valuations that agents need to learn form a continuum. This model was improved upon by \cite{cen2022regret} that added notions of price and transfer. The paper of \cite{jagadeesan2021learning} studied a contextual version of the two-sided markets where the agents preferences depend on the reveled context through an unknown function that is learnt through interactions. 

\textbf{Learning in auctions: } Online learning in simple auctions has a rich history - \cite{balcan2007single}, \cite{cole2014sample}, \cite{mohri2014learning}, \cite{blum2004online}, \cite{croissant2020real}, \cite{weed2016online}, \cite{han2020optimal}, \cite{daskalakis2022learning} to name a few, each of which study a separate angle towards learning from repeated samples in auctions. However, unlike the our setting, one side of the market knows of their true valuations apriori.
The work of \cite{kandasamy2020mechanism} is the closest to ours, where the participants apriori do not know their true valuations. However, they consider the VCG mechanism in the centralized setting, i.e., all participants can observe the utilities of all participants in every round. 

 \textbf{Learning in bi-lateral trade: } One of the first studies on learning in bilateral trade is \cite{cesa2021regret}. However, in this paper the focus is on setting a price for an incoming buyer and seller pair to facilitate trade while maximizing gain from trade, not on multi-agent auction. In particular, in each round a buyer and a seller draw her own valuation i.i.d. from their respective distributions. Then an arbiter sets the price, with trade happening if the price is between the seller's and buyer's price. They show with full-information follow-the-leader type algorithm achieves $O(\sqrt{T})$ regret, where as  with realistic (bandit-like) feedback by learning the two distributions approximately  $O(T^{2/3})$ regret can be achieved. 

 \textbf{Protocol Model: } The `protocol model' alludes to multiple agents following a similar protocol/algorithm with each agent executing protocol with private information (e.g. Platform-as-a-Service (PAAS)~\cite{guo2019learning}). In the works on learning and markets mentioned here, although the decentralized systems can be modeled as games, the protocol model is studied as a tractable way to capture the essence of the problem \cite{ashlagi2017communication, arnosti2014managing}. The technical basis for this assumption is that, in the limit when the number of participants are large, and the impact of any single participant is small, a protocol based algorithm is the limit in a precise sense of any equilibrium solution to the multi-agent game \cite{fischer2017connection}. 
\section{Conclusion and Future Work}
We study the Double auction with Average mechanism where the buyers and sellers need to know their own valuation from her own feedback. Using confidence based bounds -- UCB for the buyers and LCB for the sellers, we show that it is possible to obtain $O(\sqrt{T\log(T)})$ individual regret for the true participant buyers and sellers in $T$ rounds. Whereas, the true non-participant buyers and sellers obtain a $O(\log(T)/\Delta)$ individual regret where $\Delta$ is the smallest gap. The social regret of the proposed algorithm also admits a $O(\log(T)/\Delta)$ bound. We show that there are simpler systems where each buyer and seller must obtain a $\Omega(\sqrt{T})$ individual regret in the minimax sense. Moreover, in our setting we show even a centralized controller obtains  $\Omega(\log(T))$ social regret. Obtaining a minimax matching a $O(\sqrt{T})$ regret remains open. An important future avenue is, developing a framework and bidding strategy with provable `good' regret for general Double auction mechanisms. 
\clearpage
\bibliographystyle{plain}
\bibliography{bandits_auction}

\begin{thebibliography}{10}

\bibitem{abbasi2011improved}
Yasin Abbasi-Yadkori, D{\'a}vid P{\'a}l, and Csaba Szepesv{\'a}ri.
\newblock Improved algorithms for linear stochastic bandits.
\newblock {\em Advances in neural information processing systems}, 24, 2011.

\bibitem{angluin2007computational}
Dana Angluin, James Aspnes, David Eisenstat, and Eric Ruppert.
\newblock The computational power of population protocols.
\newblock {\em Distributed Computing}, 20(4):279--304, 2007.

\bibitem{arnosti2014managing}
Nick Arnosti, Ramesh Johari, and Yash Kanoria.
\newblock Managing congestion in decentralized matching markets.
\newblock In {\em Proceedings of the fifteenth ACM conference on Economics and
  computation}, pages 451--451, 2014.

\bibitem{ashlagi2017communication}
Itai Ashlagi, Mark Braverman, Yash Kanoria, and Peng Shi.
\newblock Communication requirements and informative signaling in matching
  markets.
\newblock In {\em EC}, page 263, 2017.

\bibitem{probabilistic_mechanism}
Moshe Babaioff and Noam Nisan.
\newblock Concurrent auctions across the supply chain.
\newblock {\em Journal of Artificial Intelligence Research}, 21:595--629, 2004.

\bibitem{balcan2007single}
Maria-Florina Balcan, Avrim Blum, and Yishay Mansour.
\newblock Single price mechanisms for revenue maximization in unlimited supply
  combinatorial auctions.
\newblock Technical report, Technical Report CMU-CS-07-111, Carnegie Mellon
  University, 2007. 3, 2007.

\bibitem{basu2021beyond}
Soumya Basu, Karthik~Abinav Sankararaman, and Abishek Sankararaman.
\newblock Beyond $ \log^ 2 (t) $ regret for decentralized bandits in matching
  markets.
\newblock In {\em International Conference on Machine Learning}, pages
  705--715. PMLR, 2021.

\bibitem{blum2004online}
Avrim Blum, Vijay Kumar, Atri Rudra, and Felix Wu.
\newblock Online learning in online auctions.
\newblock {\em Theoretical Computer Science}, 324(2-3):137--146, 2004.

\bibitem{cen2022regret}
Sarah~H Cen and Devavrat Shah.
\newblock Regret, stability \& fairness in matching markets with bandit
  learners.
\newblock In {\em International Conference on Artificial Intelligence and
  Statistics}, pages 8938--8968. PMLR, 2022.

\bibitem{cesa2021regret}
Nicol{\`o} Cesa-Bianchi, Tommaso~R Cesari, Roberto Colomboni, Federico Fusco,
  and Stefano Leonardi.
\newblock A regret analysis of bilateral trade.
\newblock In {\em Proceedings of the 22nd ACM Conference on Economics and
  Computation}, pages 289--309, 2021.

\bibitem{cesa2012combinatorial}
Nicolo Cesa-Bianchi and G{\'a}bor Lugosi.
\newblock Combinatorial bandits.
\newblock {\em Journal of Computer and System Sciences}, 78(5):1404--1422,
  2012.

\bibitem{clarke1971multipart}
Edward~H Clarke.
\newblock Multipart pricing of public goods.
\newblock {\em Public choice}, pages 17--33, 1971.

\bibitem{cole2014sample}
Richard Cole and Tim Roughgarden.
\newblock The sample complexity of revenue maximization.
\newblock In {\em Proceedings of the forty-sixth annual ACM symposium on Theory
  of computing}, pages 243--252, 2014.

\bibitem{combes2015combinatorial}
Richard Combes, Mohammad~Sadegh Talebi Mazraeh~Shahi, Alexandre Proutiere,
  et~al.
\newblock Combinatorial bandits revisited.
\newblock {\em Advances in neural information processing systems}, 28, 2015.

\bibitem{croissant2020real}
Lorenzo Croissant, Marc Abeille, and Cl{\'e}ment Calauz{\`e}nes.
\newblock Real-time optimisation for online learning in auctions.
\newblock In {\em International Conference on Machine Learning}, pages
  2217--2226. PMLR, 2020.

\bibitem{dai2021learning}
Xiaowu Dai and Michael Jordan.
\newblock Learning in multi-stage decentralized matching markets.
\newblock {\em Advances in Neural Information Processing Systems}, 34, 2021.

\bibitem{daskalakis2022learning}
Constantinos Daskalakis and Vasilis Syrgkanis.
\newblock Learning in auctions: Regret is hard, envy is easy.
\newblock {\em Games and Economic Behavior}, 2022.

\bibitem{feng2012tahes}
Xiaojun Feng, Yanjiao Chen, Jin Zhang, Qian Zhang, and Bo~Li.
\newblock Tahes: A truthful double auction mechanism for heterogeneous
  spectrums.
\newblock {\em IEEE Transactions on Wireless Communications},
  11(11):4038--4047, 2012.

\bibitem{fischer2017connection}
Markus Fischer.
\newblock On the connection between symmetric n-player games and mean field
  games.
\newblock {\em The Annals of Applied Probability}, 27:757--810, 2017.

\bibitem{graves1997asymptotically}
Todd~L Graves and Tze~Leung Lai.
\newblock Asymptotically efficient adaptive choice of control laws incontrolled
  markov chains.
\newblock {\em SIAM journal on control and optimization}, 35(3):715--743, 1997.

\bibitem{groves1973incentives}
Theodore Groves.
\newblock Incentives in teams.
\newblock {\em Econometrica: Journal of the Econometric Society}, pages
  617--631, 1973.

\bibitem{guo2019learning}
Xin Guo, Anran Hu, Renyuan Xu, and Junzi Zhang.
\newblock Learning mean-field games.
\newblock {\em Advances in Neural Information Processing Systems}, 32, 2019.

\bibitem{han2020optimal}
Yanjun Han, Zhengyuan Zhou, and Tsachy Weissman.
\newblock Optimal no-regret learning in repeated first-price auctions.
\newblock {\em arXiv preprint arXiv:2003.09795}, 2020.

\bibitem{iosifidis2014double}
George Iosifidis, Lin Gao, Jianwei Huang, and Leandros Tassiulas.
\newblock A double-auction mechanism for mobile data-offloading markets.
\newblock {\em IEEE/ACM Transactions On Networking}, 23(5):1634--1647, 2014.

\bibitem{iosifidis2009double}
George Iosifidis and Iordanis Koutsopoulos.
\newblock Double auction mechanisms for resource allocation in autonomous
  networks.
\newblock {\em IEEE Journal on Selected Areas in Communications},
  28(1):95--102, 2009.

\bibitem{jagadeesan2021learning}
Meena Jagadeesan, Alexander Wei, Yixin Wang, Michael Jordan, and Jacob
  Steinhardt.
\newblock Learning equilibria in matching markets from bandit feedback.
\newblock {\em Advances in Neural Information Processing Systems}, 34, 2021.

\bibitem{kandasamy2020mechanism}
Kirthevasan Kandasamy, Joseph~E Gonzalez, Michael~I Jordan, and Ion Stoica.
\newblock Mechanism design with bandit feedback.
\newblock {\em arXiv preprint arXiv:2004.08924}, 2020.

\bibitem{kohlberg1974repeated}
Elon Kohlberg, Shmuel Zamir, et~al.
\newblock Repeated games of incomplete information: The symmetric case.
\newblock {\em Annals of Statistics}, 2(5):1040--1041, 1974.

\bibitem{lattimore2020bandit}
Tor Lattimore and Csaba Szepesv{\'a}ri.
\newblock {\em Bandit algorithms}.
\newblock Cambridge University Press, 2020.

\bibitem{liu2020competing}
Lydia~T Liu, Horia Mania, and Michael Jordan.
\newblock Competing bandits in matching markets.
\newblock In {\em International Conference on Artificial Intelligence and
  Statistics}, pages 1618--1628. PMLR, 2020.

\bibitem{liu2021bandit}
Lydia~T Liu, Feng Ruan, Horia Mania, and Michael~I Jordan.
\newblock Bandit learning in decentralized matching markets.
\newblock {\em Journal of Machine Learning Research}, 22(211):1--34, 2021.

\bibitem{majumder2014efficient}
Bodhisattwa~P Majumder, M~Nazif Faqiry, Sanjoy Das, and Anil Pahwa.
\newblock An efficient iterative double auction for energy trading in
  microgrids.
\newblock In {\em 2014 IEEE Symposium on Computational Intelligence
  Applications in Smart Grid (CIASG)}, pages 1--7. IEEE, 2014.

\bibitem{malik2022double}
Sweta Malik, Subhasis Thakur, Maeve Duffy, and John~G Breslin.
\newblock Double auction mechanisms for peer-to-peer energy trading: A
  comparative analysis.
\newblock In {\em 2022 IEEE 7th International Energy Conference (ENERGYCON)},
  pages 1--6. IEEE, 2022.

\bibitem{mcafee1992dominant}
R~Preston McAfee.
\newblock A dominant strategy double auction.
\newblock {\em Journal of economic Theory}, 56(2):434--450, 1992.

\bibitem{mohri2014learning}
Mehryar Mohri and Andres~Munoz Medina.
\newblock Learning theory and algorithms for revenue optimization in second
  price auctions with reserve.
\newblock In {\em International conference on machine learning}, pages
  262--270. PMLR, 2014.

\bibitem{muller2002can}
R~Andrew Muller, Stuart Mestelman, John Spraggon, and Rob Godby.
\newblock Can double auctions control monopoly and monopsony power in emissions
  trading markets?
\newblock {\em Journal of environmental economics and management},
  44(1):70--92, 2002.

\bibitem{myerson1983efficient}
Roger~B Myerson and Mark~A Satterthwaite.
\newblock Efficient mechanisms for bilateral trading.
\newblock {\em Journal of economic theory}, 29(2):265--281, 1983.

\bibitem{ponssard1973zero}
Jean-Pierre Ponssard and Shmuel Zamir.
\newblock Zero-sum sequential games with incomplete information.
\newblock {\em International Journal of Game Theory}, 2(1):99--107, 1973.

\bibitem{reza2020cloud}
SM~Reza~Dibaj, Ali Miri, and SeyedAkbar Mostafavi.
\newblock A cloud priority-based dynamic online double auction mechanism
  (pb-dodam).
\newblock {\em Journal of Cloud Computing}, 9:1--26, 2020.

\bibitem{sankararaman2021dominate}
Abishek Sankararaman, Soumya Basu, and Karthik~Abinav Sankararaman.
\newblock Dominate or delete: Decentralized competing bandits in serial
  dictatorship.
\newblock In {\em International Conference on Artificial Intelligence and
  Statistics}, pages 1252--1260. PMLR, 2021.

\bibitem{vickrey1961counterspeculation}
William Vickrey.
\newblock Counterspeculation, auctions, and competitive sealed tenders.
\newblock {\em The Journal of finance}, 16(1):8--37, 1961.

\bibitem{weed2016online}
Jonathan Weed, Vianney Perchet, and Philippe Rigollet.
\newblock Online learning in repeated auctions.
\newblock In {\em Conference on Learning Theory}, pages 1562--1583. PMLR, 2016.

\bibitem{wilson1985incentive}
Robert Wilson.
\newblock Incentive efficiency of double auctions.
\newblock {\em Econometrica: Journal of the Econometric Society}, pages
  1101--1115, 1985.

\bibitem{wurman1998flexible}
Peter~R Wurman, William~E Walsh, and Michael~P Wellman.
\newblock Flexible double auctions for electronic commerce: Theory and
  implementation.
\newblock {\em Decision Support Systems}, 24(1):17--27, 1998.

\end{thebibliography}
\onecolumn
\appendix
\section{Proofs from Section~\ref{sec:analysis}}\label{sec:app_upper}
We setup some notation for the analysis.  The total number of samples collected by buyer $i$ is given as $n_{b,i}(t)$, and similarly for seller $j$, it is  $n_{s,j}(t)$. Let $K(t)$ be the number of participants in round $t$. For any $i \in [N]$ and $j\in [M]$, let $\chi_{b,i}(t)$ be the indicator of buyer $i$ participating in round $t$, and $\chi_{s,j}(t)$ be the indicator of seller $j$ participating in round $t$. Let $\mathcal{P}_{b}(t)$ denote the set of participating buyers, and $\mathcal{P}_{s}(t)$ denote the set of participating sellers.

For any buyer $i \in [N]$ and any time $t$, we have
$$
\mathcal{E}^{(\beta)}_{i,t} :=  \left\{ | \widehat{b}_i(t) - B_i | \leq \sqrt{\frac{\beta\log(t)}{n_{b,i}(t)}} \right\},\quad \mathbb{P}[\mathcal{E}_{i,t}] \geq 1 - 1/t^{\beta/2}.
$$
Similarly, for any seller $j\in [M]$ we have 
$$
\mathcal{E}^{(\beta)}_{j,t} :=   \left\{ | \widehat{s}_j(t) - S_j | \leq \sqrt{\frac{\beta\log(t)}{n_{s,j}(t)}} \right\},\quad \mathbb{P}[\mathcal{E}_{j,t}] \geq 1 - 1/t^{\beta/2}.
$$

Without loss of generality, let true valuation of the buyers be in the descending order, and for the seller in the ascending order. Then under the Average mechanism with the true bids (a.k.a. oracle Average mechanism) the buyers $1$ to $K^*$, and the sellers $1$ to $K^*$ participate, while the others do not  participate. Let $\alpha_{\min}  =\min\{\alpha_{s,j}, \alpha_{b,i}:j \in [M], i \in [N]\}$, and $\alpha_{\max}  =\min\{\alpha_{s,j}, \alpha_{b,i}:j \in [M], i \in [N]\}$.

\begin{lemma}\label{lemm:bounds}
Under the event $\mathcal{E}^{(\beta)}_t := \cap_{i\in [N]} \cap_{j \in [M]} \mathcal{E}^{(\beta)}_{i,t} \cap \mathcal{E}^{(\beta)}_{j,t}$, in time $t$, the bid for the $i$-th buyer and the $j$-th seller admits the (random) bounds with $\alpha_{\min} \geq \beta$
\begin{align*}
b_i(t) \in \left[B_i + (\sqrt{\alpha_{b,i}} - \sqrt{\beta})\sqrt{\frac{\log(t)}{n_{b,i}(t)}}, B_i + (\sqrt{\alpha_{b,i}} + \sqrt{\beta})\sqrt{\frac{\log(t)}{n_{b,i}(t)}}\right],\\
s_{j}(t) \in \left[S_j - (\sqrt{\alpha_{s,j}} + \sqrt{\beta})\sqrt{\frac{\log(t)}{n_{s,j}(t)}}, S_j - (\sqrt{\alpha_{s,j}} - \sqrt{\beta})\sqrt{\frac{\log(t)}{n_{s,j}(t)}}\right].    
\end{align*}
In particular, for any buyer $i \in [N]$, any seller $j \in [M]$, and $\alpha_{max}\geq \beta$, we have $b_i(t) \geq B_i$ and $s_j(t) \leq S_j$. 
\end{lemma}

\begin{proposition}\label{prop:atleastK_star}
For any round $t$, under the event $\mathcal{E}^{(\beta)}_t$, the following events are true, for $\min\{\alpha_{s,j}, \alpha_{b,i}\} \geq \beta$, the number of participants $K(t) \geq K^*$. 
\end{proposition}
\begin{proof}
Under the conditions, from Lemma \ref{lemm:bounds} we know that $b_i(t) \geq B_i$ and $s_j(t) \leq S_j$ for all buyers $i$ and sellers $j$. The number of participant in the system given any value profile is given as $K(t) = \max_{p} \min(|\{i\in [N]: b_i(t)\geq p\}|,|\{j\in [M]: s_j(t)\leq p\}|)$. Let $p^* \in (S_{K^*}, B_{K^*})$. Then we have $|\{i\in [N]: b_i(t)\geq p^*\}| \geq |\{i\in [N]: B_i\geq p^*\}| = K^*$. Similarly,    $|\{j\in [M]: s_j(t)\leq p^*\}| \geq |\{j\in [M]: S_j\leq p^*\}| = K^*$. Therefore, we have $K(t) \geq K^*$.
\end{proof}

Let us define the gaps for the seller $j$ from the minimum seller as $\Delta_{s,j} = (S_j - S_{K^*})$, and the gap for the buyer $i$ from the maximum buyer as $\Delta_{b,i} = (B_{K^*} - B_i)$. Also, we define for the seller $j$, $\Delta_{s,b, j} = (S_j - B_{K^*})$, and for the buyer $i$, $\Delta_{b,i} = (S_{K^*} - B_i)$. We have $S_{K^*} \leq B_{K^*}$. We now introduce a definition next to ease exposition of simultaneous participation of two buyers or two sellers.

\begin{definition}
A buyer $i$ (a seller $j$) {\em precedes} a buyer $i'$ (resp., a seller $j'$) if and only if buyer $i$ (resp., seller $j$) participates, and buyer $i'$ (resp., seller $j'$) does not participate.
\end{definition}

The next lemma states that once a true non-participant buyer (seller) have enough samples, this buyer (resp., seller) never precede the true participant buyers (resp., sellers).  

\begin{lemma}\label{lemm:topKstar}
For any round $t$, under the event $\mathcal{E}^{(\beta)}_t$, the following events are true, for $\alpha_{\min} > \beta$
\begin{itemize}
    \item for any specific buyer $i'\in [K^*]$, if for a buyer $i\geq (K^*+1)$, $n_{b,i}(t) \geq \frac{(\sqrt{\alpha_{b,i}} + \sqrt{\beta})^2}{(B_{i'} - B_i)^2} \log(t)$  then  buyer $i$ does not precede buyer $i'$.
    \item for any specific seller $j'\in [K^*]$, if for any $j\geq (K^*+1)$, $n_{s,j}(t) \geq  \frac{(\sqrt{\alpha_{s,j}} + \sqrt{\beta})^2}{(S_j - S_{j'})^2} \log(t)$  then  seller $j$ does not precede seller $j'$.
\end{itemize}
\end{lemma}
\begin{proof}
Let $K(t)$ be the number of participants in round $t$. We know that under $\mathcal{E}^{(\beta)}_t$ and $\alpha_{\min} > \beta$, $s_j(t) < S_j$ and $b_i(t) > B_i$ for all $i \in [N]$ and $j \in [M]$. We note that if for any $i\geq (K^*+1)$ and $i'\leq K^*$, if $n_{b,i}(t) \geq  \frac{(\sqrt{\alpha_{b,i}} + \sqrt{\beta})^2}{(B_{i'} - B_i)^2} \log(t)$ then $b_i(t) \leq B_{i'}$. For buyers $i'$ we have $b_{i'}(t) > B_{i'}$, hence buyer $i$ can not precede any of the buyers $i'$. This is true as under the current mechanism for any $i'$ with $b_{i'}(t) > b_i(t)$, it can not happen that buyer $i'$ participates but buyer $i$ does not. A similar argument proves the seller side statement.
\end{proof}

Furthermore, the above Lemma~\ref{lemm:topKstar} can be used in conjunction to Proposition~\ref{prop:atleastK_star}, to show the true participant buyers (or sellers) fail to match only logarithimically many times. 
\begin{lemma}\label{lemm:topKstar_v2}
Under the event $\mathcal{E}^{(\beta)}=\cup_{t=1}^{T}\mathcal{E}^{(\beta)}_t$ and $\alpha_{\min} > \beta$, we have 
\begin{itemize}
    \item $n_{b,i}(T) \geq T - \sum_{i'\geq K^*+1} \frac{(\sqrt{\alpha_{b,i'}} + \sqrt{\beta})^2}{(B_{i} - B_{i'})^2} \log(T)$ for any  $i\in [K^*]$,
    \item $n_{s,j}(T)\geq T - \sum_{j'\geq K^*+1} \frac{(\sqrt{\alpha_{s,j'}} + \sqrt{\beta})^2}{(S_{j'}-S_{j})^2} \log(T) $ for any $j\in [K^*]$. 
\end{itemize}
\end{lemma}
\begin{proof}
We know that under the condition of the lemma, $K(t) \geq K^*$ for all $1\leq t\leq T$. Therefore, we have $\sum_{i\in [N]} n_{b,i}(T) \geq K^* T$. 
Furthermore, as $K(t)\geq K^*$ in each round, a buyer $i \in [K^*]$ does not participate, only if there exists at least one participant $i' \geq (K^*+1)$ that precedes buyer $i$. This is because if in some round no participant $i' \geq (K^*+1)$ precedes buyer $i$ and buyer $i$ does not match then that implies at most there can be $(K^*-1)$ matches in that round. This leads to a contradiction of $K(t) \geq K^*$.
However, under event $\mathcal{E}^{(\beta)}$ and $\alpha_{\min} > \beta$, from Lemma~\ref{lemm:topKstar} we know that any $i' \geq (K^*+1)$ can precede buyer $i$ only $\frac{(\sqrt{\alpha_{b,i'}} + \sqrt{\beta})^2}{(B_{i} - B_{i'})^2} \log(T)$ many times. This implies that $(T - n_{b,i}(T)) \leq \sum_{i'\geq K^* + 1} \frac{(\sqrt{\alpha_{b,i'}} + \sqrt{\beta})^2}{(B_{i} - B_{i'})^2} \log(T).$

A similar treatment of the sellers give us the remaining result.
\end{proof}

We next show that the true non-participants only match logarithimically many times. This means $K(t)$ converges to $K^*$ fast. It is important to note that the previous argument does not conclude that where we mainly relied upon $K(t)\geq K^*$ for all $t$ w.h.p. To that end we introduce the following definition.

\begin{definition}
A buyer $i$, and a seller $j$ {\em co-participates} in a given round, if and only if both buyer $i$, and seller $j$ participates in that round.
\end{definition}

\begin{lemma}\label{lemm:belowKstar}
Under the event $\mathcal{E}^{(\beta)}=\cup_{t=1}^{T}\mathcal{E}^{(\beta)}_t$ and $\alpha_{\min} > \beta$, we have 
\begin{itemize}
    \item  $\sum_{t=1}^{T}\mathbbm{1}(j'\in \mathcal{P}_s(t), i'\in \mathcal{P}_b(t)) \leq \tfrac{(\sqrt{\alpha_{\max}}+\sqrt{\beta})^2}{(S_{j'} - B_{i'})^2}\log(T)$ for any  $i', j'\geq (K^*+1)$
    \item $n_{b,i}(T) \leq \left(\frac{(\sqrt{\alpha_{b,i}} + \sqrt{\beta})^2}{(B_{K^*} - B_i)^2} + \sum_{j\geq (K^*+1)} \frac{(\sqrt{\alpha_{s,j}} + \sqrt{\beta})^2}{(S_{j} - B_{i})^2}\right) \log(T)$ for any  $i\geq (K^*+1)$ and \\
    $\sum_{i\geq (K^*+1)} n_{b,i}(T) \leq \left(\sum_{i\geq (K^*+1)} \frac{(\sqrt{\alpha_{b,i}} + \sqrt{\beta})^2}{(B_{K^*} - B_i)^2} + \sum_{j\geq (K^*+1)} \frac{(\sqrt{\alpha_{s,j}} + \sqrt{\beta})^2}{(S_{j} - B_{K^*})^2}\right) \log(T)$,
    \item $n_{s,j}(T)\leq \left(\frac{(\sqrt{\alpha_{s,j}} + \sqrt{\beta})^2}{(S_{j} - S_{K^*})^2} + \sum_{i\geq (K^*+1)} \frac{(\sqrt{\alpha_{b,i}} + \sqrt{\beta})^2}{(S_{j} - B_{i})^2}\right) \log(T) $ for any $j\geq (K^*+1)$, and \\
    $
    \sum_{j\geq (K^*+1)} n_{s,j}(T)\leq \left(\sum_{j\geq (K^*+1)}\frac{(\sqrt{\alpha_{s,j}} + \sqrt{\beta})^2}{(S_{j} - S_{K^*})^2} + \sum_{i\geq (K^*+1)} \frac{(\sqrt{\alpha_{b,i}} + \sqrt{\beta})^2}{(S_{K^*} - B_{i})^2}\right) \log(T)
    $. 
\end{itemize}
\end{lemma}
\begin{proof}
Let us consider a non-participant buyer $i$. From Lemma~\ref{lemm:topKstar} we know that if $n_{b,i}(t) \geq Th_{i}(t) \equiv \frac{(\sqrt{\alpha_{b,i}} + \sqrt{\beta})^2}{(B_{K^*} - B_i)^2} \log(t)$ then buyer $i$ does not precede any buyer $i'\in [K^*]$. Therefore, for this buyer to match there should exist at least $(K^*+1)$ sellers with bids no more than the bid of this buyer $i$, and at least $(K^*+1)$ seller participates. This further implies that for this buyer to have additional participation, after $n_{b,i}(t) = Th_{i}(T)$: 
\begin{enumerate}
    \item There exists at least $1$  seller $j \geq (K^*+1)$ with bid no more than the bid of this buyer $i$. 
    \item There exists at least $1$  seller $j \geq (K^*+1)$ such that buyer $i$ and seller $j$ co-participates.
\end{enumerate}

Let $\mathcal{S}_{i,j}$ be the rounds after $n_{b,i}(t) = Th_{i}(T)$, and buyer $i$ and seller $j$ co-participates. Therefore, from the 2nd point above we conclude that $n_{b,i}(T) \leq Th_{i}(T) + |\cup_{j\geq (K^*+1)} \mathcal{S}_{i,j}|$.

However, if the seller $j$ has  $n_{s,j}(t) > Th_{i,j}(t) \equiv \frac{(\sqrt{\alpha_{s,j}} + \sqrt{\beta})^2}{(S_{j} - B_{i})^2} \log(t)$ many participation then under $\mathcal{E}^{(\beta)}_t$ and $\alpha_{\min} > \beta$ we have $s_j(t) > b_i(t)$.\footnote{Note that for any $i,j\geq (K^*+1)$ we have $S_{j} > B_{i}$.}  But then from the 1st point above we know that $|\mathcal{S}_{i,j}| \leq Th_{i,j}(T)$, because $|\mathcal{S}_{i,j}| \leq n_{s,j}(T)$, and buyer $i$ and seller $j$ can co-participate only if $n_{s,j}(t) \leq Th_{i,j}(T)$. This proves the first point.  Moreover, we have $n_{b,i}(T) \leq Th_i(T) + \sum_{j\geq (K^*+1)} Th_{i,j}(T)$.  

In fact, we can improve the cumulative bound. We have 
$$
\sum_{i\geq (K^*+1)} n_{b,i}(T) \leq \sum_{i\geq (K^*+1)} Th_{i}(T) + |\cup_{i\geq (K^*+1)}\cup_{j\geq (K^*+1)} \mathcal{S}_{i,j}|.
$$
However, we know that after $n_{s,j}(t) > \max_{i} Th_{i,j}(T)$ a seller $j\geq (K^*+1)$ can not co-participate for any seller $i\geq (K^*+1) $. Thus we can bound $|\cup_{i\geq (K^*+1)} \mathcal{S}_{i,j}| \leq \max_{i} Th_{i,j}(T)$.
$$
\sum_{i\geq (K^*+1)} n_{b,i}(T) \leq \sum_{i\geq (K^*+1)} Th_{i}(T) + \sum_{j\geq (K^*+1)} \max_{i \geq (K^*+1)} Th_{i,j}(T).
$$

A similar treatment proves the lemma for a non-participant seller $j$.
\end{proof}
\begin{corollary}\label{corr:belowKstar}
Under the event $\mathcal{E}^{(\beta)}=\cup_{t=1}^{T}\mathcal{E}^{(\beta)}_t$ and $\alpha_{\min} > \beta$, we have 
\begin{itemize}
    \item $n_{b,i}(T) \leq \frac{(M-K^*+1) (\sqrt{\alpha_{\max}} + \sqrt{\beta})^2}{(B_{K^*} - B_i)^2}\log(T)$ for any  $i\geq (K^*+1)$,
    \item $n_{s,j}(T)\leq \frac{(N-K^*+1)(\sqrt{\alpha_{\max}} + \sqrt{\beta})^2}{(S_{j} - S_{K^*})^2} \log(T) $ for any $j\geq (K^*+1)$.
\end{itemize}
\end{corollary}

We have shown, up to this point, that the optimal participating buyers and sellers participate in all but $O(log(T))$ rounds. Moreover, true non-participating buyers and sellers participate in $O(log(T))$ rounds. This suffices to show the regret for non-participating buyers and sellers is $O(log(T))$ (we will state this precisely later). However, to compute the regret for participating buyers and sellers we next need to understand how the price is set in each round. We next argue that if in any round $t$, sellers $j\geq (K^*+1)$ do not participate, and buyer $K^*$ participates then the regret of a optimal participating buyer is small. Similarly, if in any round $t$, buyers $i\geq (K^*+1)$ do not participate, and seller $K^*$ participates then the regret of a optimal participating seller is small.

\subsection{Social Welfare regret of buyers and sellers.} We now compute the social welfare regret.
\begin{lemma}\label{lemm:sw_regret}
Under the event $\mathcal{E}^{(\beta)}=\cup_{t=1}^{T}\mathcal{E}^{(\beta)}_t$ and $\alpha_{\min} > \beta$, we have
the social regret 
\begin{align*}
    r_{SW}(T) &\leq \sum_{i \leq K^*}\sum_{i' > K^*} \frac{(\sqrt{\alpha_{\max}} + \sqrt{\beta})^2}{(B_{i} - B_{i'})} \log(T) + \sum_{j\leq K^*}\sum_{j' > K^*}\frac{(\sqrt{\alpha_{\max}} + \sqrt{\beta})^2}{(S_{j'} - S_{j})} \log(T)\nonumber\\
    &\quad\quad\quad + \sum_{j'> K^*}\sum_{i' > K^*}\frac{(\sqrt{\alpha_{\max}} + \sqrt{\beta})^2}{(S_{j'} - B_{i'})}.
\end{align*}
\end{lemma}
\begin{proof}
Let us consider that  the event $\mathcal{E}^{(\beta)}=\cup_{t=1}^{T}\mathcal{E}^{(\beta)}_t$ holds, and  $\min\{\alpha_{s,j}, \alpha_{b,i}\} \geq \beta$. With that assumption we can bound the social welfare regret as follows.

\begin{align*}
    r_{SW}(T) &= T\big(\sum_{i\in \mathcal{P}_b^*} B_i + \sum_{j\in [M]\setminus\mathcal{P}_s^*} S_j\big) - \Big[\sum_{t=1}^{T}\big(\sum_{i\in \mathcal{P}_b(t)} B_i + \sum_{j\in [M]\setminus\mathcal{P}_s(t)} S_j\big)\Big]\\
    &=\sum_{t=1}^{T}\Big[\Big(\sum_{i\in [K^*]\setminus \mathcal{P}_b(t)} B_i - \sum_{i'\in \mathcal{P}_b(t)\setminus[K^*]} B_{i'}\Big) + \Big(\sum_{j'\in \mathcal{P}_s(t)\setminus [K^*]} S_{j'} - \sum_{j\in [K^*]\setminus \mathcal{P}_s(t)} S_j\Big)\Big]
\end{align*}
We use $\mathcal{P}_b^* = \mathcal{P}_b^* = [K^*]$ without loss of generality, as mentioned earlier. 

We first notice that above the number of positive and negative terms are equal. Now we consider pairing some of the positive and negative terms. Under the event $\mathcal{E}^{(\beta)}=\cup_{t=1}^{T}\mathcal{E}^{(\beta)}_t$ and  $\min\{\alpha_{s,j}, \alpha_{b,i}\} \geq \beta$, we know that $K(t)\geq K^*$  where $\mathcal{P}_b(t) = \mathcal{P}_s(t) = K(t)$. Hence, we have 
\begin{gather*}
|[K^*]\setminus \mathcal{P}_b(t)| \leq |\mathcal{P}_b(t) \setminus [K^*]|, \quad\quad |[K^*]\setminus \mathcal{P}_s(t)| \leq |\mathcal{P}_s(t) \setminus [K^*]|.    
\end{gather*}

Therefore, in the final term we can {\em pair up} each buyer with positive contribution with a buyer with negative contribution, and each seller with negative contribution with a seller with positive contribution. Finally, we may be left with some sellers with positive, and some buyers with negative contributions. But we have the number of sellers with positive contribution equals the number of buyers with negative contribution. To see this observe 
\begin{align*}
(|\mathcal{P}_b(t) \setminus [K^*]| - |[K^*]\setminus \mathcal{P}_b(t)|) = (|\mathcal{P}_s(t) \setminus [K^*]| - |[K^*]\setminus \mathcal{P}_s(t)|) = (K(t) - K^*)    
\end{align*}

Let us define by $i'(i,t) \in \mathcal{P}_b(t) \setminus [K^*]$ as the pair for the buyer $i\in [K^*]\setminus \mathcal{P}_b(t)$, such that 
$i'(i,t)$ are all unique.  Similarly, we denote by  $j'(j,t)\in \mathcal{P}_s(t) \setminus [K^*]$ the pair for $j\in [K^*]\setminus \mathcal{P}_s(t)$.
We denote the set of remaining buyers and sellers, respectively, as 
 \begin{align*}
 \mathcal{I}'(t) = (\mathcal{P}_b(t) \setminus [K^*]) \setminus \cup_{i\in [K^*]\setminus \mathcal{P}_b(t)} i'(i,t),\\
\mathcal{J}'(t) = (\mathcal{P}_s(t) \setminus [K^*]) \setminus \cup_{j\in [K^*]\setminus \mathcal{P}_s(t)} j'(j,t).
 \end{align*}
Finally, we denote by $i(j',t)\in \mathcal{I}'(t)$ as the pair for $j' \in \mathcal{J}'(t)$.\footnote{Any arbitrary pairing works for this purpose. For concreteness, we may assume the the buyer/seller in one set, is matched with that of the other ranked by their respective ids.} 

With these definitions we can bound the social regret as
\begin{align}
    r_{SW}(T) &=\sum_{t=1}^{T}\Big[\sum_{i\in [K^*]\setminus \mathcal{P}_b(t)} (B_i - B_{i'(i,t)}) + \sum_{j\in [K^*]\setminus\mathcal{P}_s(t)} (S_{j'(j,t)} - S_j) + \sum_{j\in \mathcal{J}'(t)} (S_{j}- B_{i(j',t)})\Big] \nonumber\\
     &= \sum_{t=1}^{T}\Big[\sum_{i \leq K^*}(B_i - B_{i'(i,t)}) \mathbbm{1}(i'(i,t)\in \mathcal{P}_b(t), i\notin \mathcal{P}_b(t)) \nonumber\\
    &\quad\quad\quad + \sum_{j\leq K^*}(S_{j'(j,t)} -  S_j)\mathbbm{1}(j'(j,t)\in \mathcal{P}_s(t), j\notin \mathcal{P}_s(t)) \nonumber\\
    &\quad\quad\quad + \sum_{j' > K^*}(S_{j'} -  B_{i(j',t)})\mathbbm{1}(j'\in \mathcal{P}_s(t), i(j',t)\in \mathcal{P}_b(t))\Big]\nonumber\\
    &\leq \sum_{t=1}^{T}\Big[\sum_{i \leq K^*}\sum_{i' > K^*}(B_i - B_{i'}) \mathbbm{1}(i'\in \mathcal{P}_b(t), i\notin \mathcal{P}_b(t)) \nonumber\\
    &\quad\quad\quad + \sum_{j\leq K^*}\sum_{j' > K^*}(S_{j'} -  S_j)\mathbbm{1}(j'\in \mathcal{P}_s(t), j\notin \mathcal{P}_s(t)) \nonumber\\
    &\quad\quad\quad + \sum_{i'> K^*}\sum_{j' > K^*}(S_{j'} -  B_{i'})\mathbbm{1}(j'\in \mathcal{P}_s(t), i'\in \mathcal{P}_b(t))\Big]\label{eq:sw_ineq1}\\
    & = \sum_{i \leq K^*}\sum_{i' > K^*}(B_i - B_{i'}) \sum_{t=1}^{T}\mathbbm{1}(i'\in \mathcal{P}_b(t), i\notin \mathcal{P}_b(t)) \nonumber\\
    &\quad\quad\quad + \sum_{j\leq K^*}\sum_{j' > K^*}(S_{j'} -  S_j)\sum_{t=1}^{T}\mathbbm{1}(j'\in \mathcal{P}_s(t), j\notin \mathcal{P}_s(t)) \nonumber\\
    &\quad\quad\quad + \sum_{i'> K^*}\sum_{j' > K^*}(S_{j'} -  B_{i'})\sum_{t=1}^{T}\mathbbm{1}(j'\in \mathcal{P}_s(t), i'\in \mathcal{P}_b(t))\nonumber\\
    & \leq \sum_{i \leq K^*}\sum_{i' > K^*} \frac{(\sqrt{\alpha_{b,i'}} + \sqrt{\beta})^2}{(B_{i} - B_{i'})} \log(T) + \sum_{j\leq K^*}\sum_{j' > K^*}\frac{(\sqrt{\alpha_{s,j'}} + \sqrt{\beta})^2}{(S_{j'} - S_{j})} \log(T)\nonumber\\
    &\quad\quad\quad + \sum_{j'> K^*}\sum_{i' > K^*}\frac{(\sqrt{\alpha_{s,i'}} + \sqrt{\beta})^2}{(S_{j'} - B_{i'})}\log(T).\label{eq:sw_ineq2} 
\end{align}

The first inequality ~\eqref{eq:sw_ineq1} upper bounds $i'(i,t)$, $j'(j,t)$, and $i(j',t)$ with the sum over the sets each of them can belong to. 

Under the event $\mathcal{E}^{(\beta)}=\cup_{t=1}^{T}\mathcal{E}^{(\beta)}_t$ and  $\min\{\alpha_{s,j}, \alpha_{b,i}\} \geq \beta$, the final inequality ~\eqref{eq:sw_ineq2} follows from Lemma~\ref{lemm:topKstar} and Lemma~\ref{lemm:belowKstar}. In particular,  $\sum_{t=1}^{T}\mathbbm{1}(i'\in \mathcal{P}_b(t), i\notin \mathcal{P}_b(t))$ denotes the number of times a true non-participating buyer $i'$ can precede a optimal participating buyer $i$, and Lemma~\ref{lemm:topKstar} bounds these terms. Similarly, the terms $\sum_{t=1}^{T}\mathbbm{1}(j'\in \mathcal{P}_s(t), j\notin \mathcal{P}_s(t))$ are bounded with the help of Lemma~\ref{lemm:topKstar}.  Finally, the terms $\mathbbm{1}(j'\in \mathcal{P}_s(t), i'\in \mathcal{P}_b(t))$ denote how many times a pair of true non-participating buyer and seller $i'$ and $j'$ can  co-participate. Following proof of Lemma~\ref{lemm:belowKstar}, we can bound this with $\frac{(\sqrt{\alpha_{s,i'}} + \sqrt{\beta})^2}{(S_{j'} - B_{i'})}\log(T)$. This finishes the proof of the lemma.
\end{proof}

Using Lemma~\ref{lemm:sw_regret}, the expected regret can be bounded for $\beta \geq 4$ as 
\begin{align*}
R_{SW}(T) &\leq \mathbb{E}[r_{SW}(T)| \mathcal{E}^{(\beta)}] +  b_{max}(1- \mathbb{P}[\mathcal{E}^{(\beta)}])\\
&\leq \mathbb{E}[r_{SW}(T)| \mathcal{E}^{(\beta)}] + b_{max}\sum_{t=1}^{T} MN/t^{\beta/2}\\
&\leq \mathbb{E}[r_{SW}(T)| \mathcal{E}^{(\beta)}] + MN b_{max} \pi^2/6.
\end{align*}

\subsection{Individual regret of buyers and sellers.}

Recall that $p^* = (S_{K^*}+B_{K^*})/2$ be the price under true bids for the average mechanism, and $p(t) = (\min_{i' \in \mathcal{P}_b(t)}b_{i'}(t) + \max_{j\in \mathcal{P}_s(t)} s_j(t))/2$ denotes the price in round $t$. Let $\chi_{b,i}(t)$, and $\chi_{s,j}(t)$ is the participation indicator for buyer $i$, and seller $j$ respectively.
The individual regret for any true non-participating buyer or seller can be computed easily. We will present it later.

The regret for any optimal participating buyer $i \in [K^*]$ can be decomposed in rounds where the buyer $i$ does not participate, and where the buyer $i$ participates.
\begin{align*}
r_{b,i}(T) &= \sum_{t: \chi_{b,i}(t) = 0} (B_i - (B_{K^*}+S_{K^*})/2) +   \sum_{t: \chi_{b,i}(t)=1} (p(t) - p^*).
\end{align*}
Similarly,  for any optimal participating seller $j \in [K^*]$ the regret is bounded as 
\begin{align*}
r_{s,j}(T) &= \sum_{t: \chi_{s,j}(t) = 0} ((B_{K^*}+S_{K^*})/2 - S_j) +   \sum_{t: \chi_{s,j}(t)=1} (p^*-p(t)).
\end{align*}

We first focus on the regret of the buyers which implies upper bounding $\sum_{t} (p(t) - p^*)$ in the next lemma.

\begin{lemma}\label{lemm:price_bound}
Under the event $\mathcal{E}^{(\beta)}=\cup_{t=1}^{T}\mathcal{E}^{(\beta)}_t$ and $\alpha_{\min} > \beta$, we have
for all $\tilde{i} \in [N]$ and $\tilde{j} \in [M]$
\begin{gather*}
    \sum_{t: \chi_{b,\tilde{i}}(t)=1} (p(t) - p^*) \leq  C_b\log(T) + (\sqrt{\alpha_{max}}  + \sqrt{\beta})\sqrt{n_{b,\tilde{i}}(T)\log(T)},\\
    \sum_{t: \chi_{s,\tilde{j}}(t)=1} (p^* - p(t)) \leq C_s \log(T) + (\sqrt{\alpha_{max}}  + \sqrt{\beta})\sqrt{n_{b,\tilde{j}}(T)\log(T)},
\end{gather*}
where 
\begin{align*}
    2 C_b &= \sum_{j < K^*}\frac{(\sqrt{\alpha_{max}}  + \sqrt{\beta})^2}{(S_{K^*}-S_j)} + \sum_{i\geq (K^*+1)} \frac{(\sqrt{\alpha_{max}}  + \sqrt{\beta})^2\sqrt{(M-K^*+1)}}{(B_{K^*} - B_i)}\\
    &+\sum_{j\geq (K^*+1)} \left(\frac{(N-K^*+1)(\sqrt{\alpha_{max}}  + \sqrt{\beta})^2}{(S_{j} - S_{K^*})}
    + \frac{\sqrt{(N-K^*+1)}(\sqrt{\alpha_{max}}  + \sqrt{\beta})^2}{(S_{j} - S_{K^*})}\right),\\
    2 C_s &= \sum_{i < K^*}\frac{(\sqrt{\alpha_{max}}  + \sqrt{\beta})^2}{(B_i - B_{K^*})} + \sum_{j\geq (K^*+1)} \frac{(\sqrt{\alpha_{max}}  + \sqrt{\beta})^2\sqrt{(N-K^*+1)}}{(S_j - S_{K^*})}\\
    &+\sum_{i\geq (K^*+1)} \left(\frac{(M-K^*+1)(\sqrt{\alpha_{max}}  + \sqrt{\beta})^2}{(B_{K^*} - B_{i})}
    + \frac{\sqrt{(M-K^*+1)}(\sqrt{\alpha_{max}}  + \sqrt{\beta})^2}{(B_{K^*} - B_{i})}\right).
\end{align*}
\end{lemma}
We are now in a position to prove out main theorem.
\begin{theorem} \label{thm:upper_main}
The regret of the Average mechanism with buyers bidding UCB($\alpha$), and sellers bidding LCB($\alpha$) of their estimated valuation, for $\alpha_{\min} > \beta \geq 4$, we have the expected regret is bounded as: 
\begin{itemize}
    \item for a participating buyer $i \in [K^*]$ as 
    $R_{b,i}(T) \leq (\sqrt{\alpha_{max}}  + \sqrt{\beta})\sqrt{T\log(T)} + C_{b', i} \log(T)$,
    \item for a participating seller $j \in [K^*]$ as
    $R_{s,j}(T) \leq (\sqrt{\alpha_{max}}  + \sqrt{\beta})\sqrt{T\log(T)} + C_{s', j} \log(T)$,
    \item for a non-participating buyer $i \geq (K^*+1)$ as\\ 
    $R_{b,i}(T) \leq \frac{\sqrt{(M-K^*+1)} (\sqrt{\alpha_{max}}  + \sqrt{\beta})^2}{(B_{K^*} - B_i)} \log(T)$,
    \item for a non-participating seller $j \geq (K^*+1)$ as \\
    $R_{s,j}(T) \leq \frac{\sqrt{(N-K^*+1)} (\sqrt{\alpha_{max}}  + \sqrt{\beta})^2}{(S_j - S_{K^*})} \log(T)$.
\end{itemize}
Here $C_b$ and $C_s$ is as defined in Lemma~\ref{lemm:price_bound}, and 
\begin{gather*}
    C_{b', i} = \left((N-K^*) \frac{(\sqrt{\alpha_{max}}  + \sqrt{\beta})^2 }{(B_i - p^*)}  + C_b\right),\\
    C_{s', j} = \left((M-K^*) \frac{(\sqrt{\alpha_{max}}  + \sqrt{\beta})^2 }{(p^* - S_j)} + C_s\right).
\end{gather*}
\end{theorem}
\begin{proof}[Proof of Regret Upper Bound]
We first bound the regret under the event $\mathcal{E}^{(\beta)}=\cup_{t=1}^{T}\mathcal{E}^{(\beta)}_t$ and $\alpha_{\min} > \beta$. Applying the bounds in Lemma~\ref{lemm:price_bound}, we bound the of a optimal participating buyer $i \in [K^*]$ as
\begin{align*}
r_{b,i}(T) &= \sum_{t: \chi_{b,i}(t) = 0} (B_i - (B_{K^*}+S_{K^*})/2) +   \sum_{t: \chi_{b,i}(t)=1} (p(t) - p^*)\\
&\leq  \sum_{i'\geq K^*+1} \frac{(\sqrt{\alpha_{max}}  + \sqrt{\beta})^2 (B_i - (B_{K^*}+S_{K^*})/2)}{(B_{i} - B_{i'})^2} \log(T)  + C_b\log(T) + (\sqrt{\alpha_{max}}  + \sqrt{\beta})\sqrt{T\log(T)}\\
&\leq  \left((N-K^*)\frac{(\sqrt{\alpha_{max}}  + \sqrt{\beta})^2 }{(B_i - p^*)}  + C_b\right)\log(T) + (\sqrt{\alpha_{max}}  + \sqrt{\beta})\sqrt{T\log(T)}
\end{align*}
Similarly,  for any optimal participating seller $j \in [K^*]$ the regret is bounded as 
\begin{align*}
r_{s,j}(T) &= \sum_{t: \chi_{s,j}(t) = 0} ((B_{K^*}+S_{K^*})/2 - S_j) +   \sum_{t: \chi_{s,j}(t)=1} (p^*-p(t))\\
&\leq  \sum_{j'\geq K^*+1} \frac{(\sqrt{\alpha_{max}}  + \sqrt{\beta})^2 ((B_{K^*}+S_{K^*})/2 - S_j)}{(S_{j'} - S_{j})^2} \log(T)  + C_s\log(T) + (\sqrt{\alpha_{max}}  + \sqrt{\beta})\sqrt{T\log(T)}\\
&\leq  \left((M-K^*) \frac{(\sqrt{\alpha_{max}}  + \sqrt{\beta})^2 }{(p^* - S_j)}  + C_s\right)\log(T) + (\sqrt{\alpha_{max}}  + \sqrt{\beta})\sqrt{T\log(T)}\\
\end{align*}

The regret for any true non-participating buyer $i \geq (K^*+1)$ is non negative only when the buyer $i$ participates. Under $\mathcal{E}^{(\beta)}$ and $\alpha_{\min} > \beta$ we have
\begin{align*}
r_{b,i}(T) &= \sum_{t: \chi_{b,i}(t) = 1} (p(t) - B_i)\\
&\leq \sum_{t: \chi_{b,i}(t) = 1} (b_i(t) - B_i) \\
&\leq \sum_{t: \chi_{b,i}(t) = 1} (\sqrt{\alpha_{max}}  + \sqrt{\beta}) \sqrt{\tfrac{log(t)}{n_{b,i}(t)}}\\
&\leq \sum_{n=1}^{n_{b,i}(T)} (\sqrt{\alpha_{max}}  + \sqrt{\beta}) \sqrt{\tfrac{log(T)}{n}}\\
&\leq (\sqrt{\alpha_{max}}  + \sqrt{\beta})\sqrt{n_{b,i}(T)\log(T)}
\end{align*}
Where the first inequality is due to the fact that if buyer $i$ participates in round $t$ then bid $b_i(t)\geq p(t)$.

Also, the regret for any true non-participating buyer $j \geq (K^*+1)$ is non negative only when the buyer $i$ participates. Under $\mathcal{E}^{(\beta)}$ and $\alpha_{\min} > \beta$ we have similarly
\begin{align*}
r_{s,j}(T) &= \sum_{t: \chi_{s,j}(t) = 1} (S_j - s_j(t)) \leq (\sqrt{\alpha_{max}}  + \sqrt{\beta})\sqrt{n_{s,j}(T)\log(T)}
\end{align*}

The terms $n_{b,i}(T)$ and $n_{s,j}(T)$ above can be bounded using Lemma~\ref{lemm:belowKstar} when $\mathcal{E}^{\beta}$ holds for $\alpha_{max}\geq \beta$

Therefore, the expected regret can be bounded for $\beta \geq 4$ as 
\begin{align*}
R_{b,i}(T) &\leq \mathbb{E}[r_{b,i}(T)| \mathcal{E}^{(\beta)}] +  b_{max}(1- \mathbb{P}[\mathcal{E}^{(\beta)}])\\
&\leq \mathbb{E}[r_{b,i}(T)| \mathcal{E}^{(\beta)}] + b_{max}\sum_{t=1}^{T} MN/t^{\beta/2}\\
&\leq \mathbb{E}[r_{b,i}(T)| \mathcal{E}^{(\beta)}] + MN b_{max} \pi^2/6.
\end{align*}
Similarly, for $\beta \geq 4$ we have 
\begin{align*}
R_{s,j}(T) &\leq \mathbb{E}[r_{s,j}(T)| \mathcal{E}^{(\beta)}] + MN s_{max} \pi^2/6.
\end{align*}
This concludes the proof.
\end{proof}

Let us recall the minimum gap is $\Delta = \min_{i, \in[N], j \in [M]}\{|p^* - S_j|, |B_i - p^*|\}$.
Then we can bound the constants associated with the logarithmic terms as 
\begin{corollary}
The constants in Theorem~\ref{thm:upper_main} is upper bounded as 
\begin{itemize}
    \item $C_{b'} \leq \tfrac{\small\left(N + (N - K^*+ 1)\sqrt{M- K^*+1} + \sqrt{N- K^*+1}(M - K^*+ 1) + (N - K^*+ 1)(M- K^*+1) \right)(\sqrt{\alpha_{max}}  + \sqrt{\beta})^2}{\Delta} $
    \item $C_{s'} \leq \tfrac{\small\left(M + (N - K^*+ 1)\sqrt{M- K^*+1} + \sqrt{N- K^*+1}(M - K^*+ 1) + (N - K^*+ 1)(M- K^*+1) \right)(\sqrt{\alpha_{max}}  + \sqrt{\beta})^2}{\Delta} $
\end{itemize}
\label{cor:c_b_prime}
\end{corollary}

\subsection{Proof of Lemma~\ref{lemm:price_bound}}
\begin{proof}
We first focus on the upper bound for some buyer $i\in [N]$.
\begin{align*}
&\sum_{t: \chi_{b,\tilde{i}}(t)=1} (p(t) - p^*)
= \tfrac{1}{2}\sum_{t:\chi_{b,\tilde{i}}(t)=1} \left(\min_{i' \in \mathcal{P}_b(t)}b_{i'}(t) - B_{K^*}\right)
+ \tfrac{1}{2} \sum_{t: \chi_{b,\tilde{i}}(t)=1}(\max_{j\in \mathcal{P}_s(t)} s_j(t) - S_{K^*})
\end{align*}

The rounds where buyer $K^*$ participates, we have $\min_{i \in \mathcal{P}_b(t)}b_{i}(t) \leq  b_{K^*}(t)$. Therefore, we can bound
\begin{align*}
&\sum_{t} \left(\min_{i' \in \mathcal{P}_b(t)}b_{i'}(t) - B_{K^*}\right) \\
&\leq \sum_{t: \chi_{b,K^*}(t)=0} \left(\min_{i' \in \mathcal{P}_b(t)}b_{i'}(t) - B_{K^*}\right) + \sum_{t: \chi_{b,K^*}(t)=1, \chi_{b,\tilde{i}}(t)=1} \left(b_{K^*}(t) - B_{K^*}\right)
\end{align*}

Under $\mathcal{E}^{(\beta)}=\cup_{t=1}^{T}\mathcal{E}^{(\beta)}_t$ and $\alpha_{\min} > \beta$, we know that 
$\chi_{b,K^*}(t)=0$ only if $\max_{i\geq (K^*+1)}\chi_{b,K^*}(t) = 1$. That is $K^*$ does not participate, only if at least one of the non-participant buyers participate. Hence we further have,
\begin{align*}
&\sum_{t: \chi_{b,K^*}(t)=0} \left(\min_{i' \in \mathcal{P}_b(t)}b_{i'}(t) - B_{K^*}\right) \\
&\leq \sum_{i\geq (K^*+1)}\sum_{t: \chi_{b,i}(t)=1, \chi_{b,K^*}(t)=0} \left(b_{i}(t) - B_{K^*}\right)\\
& \leq \sum_{i\geq (K^*+1)}\sum_{t: \chi_{b,i}(t)=1} \left(b_{i}(t) - B_{K^*}\right)\\
& \leq \sum_{i\geq (K^*+1)}\sum_{t: \chi_{b,i}(t)=1} \left(b_{i}(t) - B_{i}\right)
\end{align*}

Under the event $\mathcal{E}^{(\beta)}=\cup_{t=1}^{T}\mathcal{E}^{(\beta)}_t$ and $\alpha_{\min} > \beta$, we have
\begin{align*}
    &\sum_{t:\chi_{b,K^*}(t)=1, \chi_{b,\tilde{i}}(t)=1}(b_{K^*}(t) - B_{K^*}) \\
    &\leq \sum_{t:\chi_{b,K^*}(t)=1, \chi_{b,\tilde{i}}(t)=1}(\sqrt{\alpha_{b,K^*}} + \sqrt{\beta})\sqrt{\tfrac{\log(t)}{n_{b, K^*}(t)}} \\
    &\leq \sum_{n=1}^{n_{b,\tilde{i}}(T)} (\sqrt{\alpha_{b,K^*}} + \sqrt{\beta})\sqrt{\tfrac{\log(T)}{n}} \leq (\sqrt{\alpha_{b,K^*}} + \sqrt{\beta})\sqrt{n_{b,\tilde{i}}(T)\log(T)},
\end{align*}
Above, we use the logic that the summation is minimized when $n_{b, K^*}(t)$ increases in unison with $n_{b, i}(t)$, otherwise we will have larger denominator. 

From Corollary~\ref{corr:belowKstar} we know that under the event $\mathcal{E}^{(\beta)}$ and $\alpha_{\min} > \beta$, the maximum number of time a buyer $i\geq (K^*+1)$ can participate is 
$$
\tilde{Th}_i(T) =  (M-K^*+1)\frac{(\sqrt{\alpha_{\max}} + \sqrt{\beta})^2}{(B_{K^*} - B_i)^2}\log(T).$$
Under the event $\mathcal{E}^{(\beta)}=\cup_{t=1}^{T}\mathcal{E}^{(\beta)}_t$ and $\alpha_{\min} > \beta$, we have
\begin{align*}
    &\sum_{t: \chi_{b,i}(t)=1} \left(b_{i}(t) - B_{i}\right) \\
    &\leq \sum_{t:\chi_{b,i}(t)=1}(\sqrt{\alpha_{\max}} + \sqrt{\beta})\sqrt{\tfrac{\log(t)}{n_{b, i}(t)}}\\ &\leq \sum_{n=1}^{\tilde{Th}_{i}(T)} (\sqrt{\alpha_{\max}} + \sqrt{\beta})\sqrt{\tfrac{\log(T)}{n}}\\
    &\leq (\sqrt{\alpha_{\max}} + \sqrt{\beta})\sqrt{\tilde{Th}_{i}(T)\log(T)}\\
&\leq \frac{(\sqrt{\alpha_{\max}} + \sqrt{\beta})^2\sqrt{(M-K^*+1)}}{(B_{K^*} - B_i)}\log(T).
\end{align*}

Let  $j_{\max}(t) = \arg\max_{j\in [M]} \left((S_j - S_{K^*}) + (\sqrt{\alpha_{s,j}} + \sqrt{\beta})\sqrt{\tfrac{\log(t)}{n_{b, j}(t)}}\right).$ 
We have under the event $\mathcal{E}^{(\beta)}=\cup_{t=1}^{T}\mathcal{E}^{(\beta)}_t$ and $\alpha_{\min} > \beta$,
\begin{align*}
&\sum_{t} \left(\max_{j \in \mathcal{P}_s(t)}s_{j}(t) - S_{K^*}\right) \\
&\leq \sum_{t} \max_{j\in \mathcal{P}_s(t)} \left((S_j - S_{K^*}) + (\sqrt{\alpha_{max}} + \sqrt{\beta})\sqrt{\tfrac{\log(t)}{n_{b, j}(t)}}\right)\\
&\leq \sum_{j\in [M]} \sum_{t: \chi_{s,j}(t)=1, j_{\max}(t) = j}  \left((S_j - S_{K^*}) + (\sqrt{\alpha_{max}} + \sqrt{\beta})\sqrt{\tfrac{\log(t)}{n_{b, j}(t)}}\right)
\end{align*}

Further, under the event $\mathcal{E}^{(\beta)}$ and $\alpha_{\min} > \beta$, we know that  $K(t) \geq K^*$, which implies at least one seller $j \geq K^*$ is active. Hence, for any $j < K^*$, $j_{\max}(t) = j$ only if 
$$ n_{b, j}(t) \leq  \tilde{Th}_{j}(T) = \min_{j' \geq K^*} \frac{(\sqrt{\alpha_{max}} + \sqrt{\beta})^2}{(S_{j'}-S_j)^2} \log(T) =  \frac{(\sqrt{\alpha_{max}} + \sqrt{\beta})^2}{(S_{K^*}-S_j)^2} \log(T).$$ 
Also, the maximum number of times any seller $j \geq (K^*+1)$ participates under the event $\mathcal{E}^{(\beta)}$ and $\alpha_{\min} > \beta$ is  $Th_{j}(T) \geq \frac{(N-K^*+1)(\sqrt{\alpha_{max}} + \sqrt{\beta})^2}{(S_{j} - S_{K^*})^2} \log(T)$, according to \ref{corr:belowKstar}.

Therefore, we can proceed as  
\begin{align*}
    &\sum_{t:\chi_{b,\tilde{i}}(t)=1}(\max_{j\in \mathcal{P}_s(t)} s_j(t) - S_{K^*}) \\
    &\leq \sum_{j\in [M]} \sum_{t: \chi_{s,j}(t)\chi_{s,\tilde{i}}(t)=1, j_{\max}(t) = j}  \left((S_j - S_{K^*}) + (\sqrt{\alpha_{max}} + \sqrt{\beta})\sqrt{\tfrac{\log(t)}{n_{b, j}(t)}}\right)\\
    & \leq \sum_{j\geq (K^*+1)} \sum_{t: \chi_{s,j}(t)=1}  \left((S_j - S_{K^*}) + (\sqrt{\alpha_{max}} + \sqrt{\beta})\sqrt{\tfrac{\log(t)}{n_{b, j}(t)}}\right)\\
    &+ \sum_{j\leq K^*} \sum_{t: \chi_{s,j}(t)\chi_{b,\tilde{i}}(t)=1, j_{\max}(t) = j}  \left((S_j - S_{K^*}) + (\sqrt{\alpha_{max}}  + \sqrt{\beta})\sqrt{\tfrac{\log(t)}{n_{b, j}(t)}}\right)\\
    & \leq \sum_{j\geq (K^*+1)}  \left(Th_j(T)(S_j - S_{K^*}) + \sum_{n=1}^{Th_j(T)}(\sqrt{\alpha_{max}}  + \sqrt{\beta})\sqrt{\tfrac{\log(T)}{n}}\right)\\
    &+ \sum_{j < K^*} \sum_{n=1}^{\tilde{Th}_j(T)}  (\sqrt{\alpha_{max}}  + \sqrt{\beta})\sqrt{\tfrac{\log(T)}{n}}
    + \sum_{n=1}^{n_{b,\tilde{i}}(T)}(\sqrt{\alpha_{max}}  + \sqrt{\beta})\sqrt{\tfrac{\log(T)}{n}}\\
    &\leq \sum_{j\geq (K^*+1)}  \left(Th_j(T)(S_j - S_{K^*}) + (\sqrt{\alpha_{max}}  + \sqrt{\beta})\sqrt{Th_j(T)\log(T)}\right)\\
    &+ \sum_{j < K^*}(\sqrt{\alpha_{max}}  + \sqrt{\beta})\sqrt{\tilde{Th}_j(T) \log(T)}
    + (\sqrt{\alpha_{max}}  + \sqrt{\beta})\sqrt{n_{b,\tilde{i}}(T) \log(T)}\\
    & \leq \sum_{j\geq (K^*+1)} \left(\frac{(N-K^*+1)(\sqrt{\alpha_{max}}  + \sqrt{\beta})^2}{(S_{j} - S_{K^*})}
    + \frac{\sqrt{(N-K^*+1)}(\sqrt{\alpha_{max}}  + \sqrt{\beta})^2}{(S_{j} - S_{K^*})}\right)\log(T)\\
    &+ \sum_{j < K^*}\frac{(\sqrt{\alpha_{max}}  + \sqrt{\beta})^2}{(S_{K^*}-S_j)} \log(T)
    + (\sqrt{\alpha_{max}}  + \sqrt{\beta})\sqrt{n_{b,\tilde{i}}(T) \log(T)}
\end{align*}

Combining the above bounds, we get 
\begin{align*}
&2 \sum_{t:\chi_{b,\tilde{i}}(t) = 1} (p(t) - p^*) \\
&\leq \underbrace{2(\sqrt{\alpha_{max}}  + \sqrt{\beta})\sqrt{n_{b,\tilde{i}}(T)\log(T)}}_{\text{price setting buyer and seller}} 
+ \underbrace{\sum_{j < K^*}\frac{(\sqrt{\alpha_{max}}  + \sqrt{\beta})^2}{(S_{K^*}-S_j)} \log(T)}_{\text{non-price setting participant sellers}}\\
& + \underbrace{\sum_{i\geq (K^*+1)} \frac{(\sqrt{\alpha_{max}}  + \sqrt{\beta})^2\sqrt{(M-K^*+1)}}{(B_{K^*} - B_i)}\log(T)}_{\text{non-participating buyers}} \\
& + \underbrace{\sum_{j\geq (K^*+1)} \left(\frac{(N-K^*+1)(\sqrt{\alpha_{max}}  + \sqrt{\beta})^2}{(S_{j} - S_{K^*})}
    + \frac{\sqrt{(N-K^*+1)}(\sqrt{\alpha_{max}}  + \sqrt{\beta})^2}{(S_{j} - S_{K^*})}\right)\log(T)}_{\text{non-participating sellers}}.
\end{align*}

Reversing the role of buyer and seller in the above derivation, and leveraging the symmetry in the system, we can get 
\begin{align*}
&2 \sum_{t:\chi_{s,\tilde{j}}(t) = 1} (p^* - p(t)) \\
&\leq \underbrace{2(\sqrt{\alpha_{max}}  + \sqrt{\beta})\sqrt{n_{s,\tilde{j}}(T)\log(T)}}_{\text{price setting buyer and seller}} 
+ \underbrace{\sum_{i < K^*}\frac{(\sqrt{\alpha_{max}}  + \sqrt{\beta})^2}{(B_i - B_{K^*})} \log(T)}_{\text{non-price setting participant buyers}}\\
& + \underbrace{\sum_{j\geq (K^*+1)} \frac{(\sqrt{\alpha_{max}}  + \sqrt{\beta})^2\sqrt{(N-K^*+1)}}{(S_j - S_{K^*})}\log(T)}_{\text{non-participating sellers}} \\
& + \underbrace{\sum_{i\geq (K^*+1)} \left(\frac{(M-K^*+1)(\sqrt{\alpha_{max}}  + \sqrt{\beta})^2}{(B_{K^*} - B_{i})}
+ \frac{\sqrt{(M-K^*+1)}(\sqrt{\alpha_{max}}  + \sqrt{\beta})^2}{(B_{K^*} - B_{i})}\right)\log(T)}_{\text{non-participating buyers}}.
\end{align*}
\end{proof}

\section{Proofs of the Lower Bounds}
\label{sec:appendix_lower_bounds}

\subsection{Minimax lower bound on individual regret}

We show a minimax regret lower bound of $\Omega(\sqrt{T})$ in Lemma \ref{lem:lb_exact} by considering a simpler system that decouples learning and competition.  {\color{black} In this system, the seller is assumed to {\em(i)} know her exact valuation, and {\em(ii)} always ask her true valuation as the selling price, i.e., is truthful in her asking price in all the rounds. Furthermore, the pricing at every round is fixed to the average $p_t = \frac{B_t + S_t}{2}$ in the event that $B_t \geq S$. 
The utility of the buyer at time $t$ is defined as $U_t =  (p_t - B)\mathbf{1}(B_t \geq S).$ } 
\begin{lemma}
The utility maximizing action of an oracle buyer that knows $B$ and $S$ is to bid $B_t = S \mathbf{1}(B \geq S) + (S-\varepsilon)\mathbf{1}(B<S)$ for any $\varepsilon > 0$, at all times. In words, the oracle buyer either bids $S$ and pays the price $S$ or ``abstains" by bidding less than $S$. 
\label{lem:oracle_lb_optim}
\end{lemma}

\begin{corollary}\label{cor:red_two_arm_bandit}
To minimize expected utility in the simple system, it suffices for the buyer each round to decide to either bid $S$ and participate in the market by paying price $S$, or abstain without participation and obtain no reward.
\end{corollary}

\textbf{Reduction to a two armed bandit problem:}
Corollary \ref{cor:red_two_arm_bandit} gives that at each time, it suffices for the buyer that does not know her true valuation to either bid $S$ and participate at price $S$, or abstain from participating. In any round $t$ that the buyer participates, she obtains a mean reward of $B-S$, while she receives $0$ reward in rounds she abstains from participating. Thus, the actions of the buyer are equivalent to a two armed bandit, one with mean $B-S$ and the other is deterministic $0$ mean. The reduction is formalized in the following Corollary.

\begin{corollary}
Any bidding policy given in Definition \ref{defn:bidding_policy} describes a two-armed bandit policy (Definition \ref{defn:bandit_policy}) with arm-means $B-S$ and $0$.
\label{cor:bandit_red}
\end{corollary}

\begin{lemma}
For every bidding policy, there exists a system such that $\mathbb{E}[R_T] \geq \frac{1}{36}\sqrt{T}$.
\label{lem:lb_exact}
\end{lemma}
The proof follows from \cite{lattimore2020bandit} and is reproduced in Appendix \ref{subsec:lb_proof_exact} for completeness. Further, we show in Appendix \ref{sec:multi_agent_lb_extension} the lower bound can be extended to a system of multiple buyers and sellers. Thus, our upper bound of $O(\sqrt{T})$ is order-wise optimal as Lemma \ref{lem:lb_exact} and Corollary \ref{cor:bandit_red} show that 
$O(\sqrt{T})$ regret bound is un-avoidable  even in the absence of competition.

\subsection{Lower bound on Social Welfare Regret}
\label{sec:sw_regret_lb}
The key observation is that social-welfare regret in Equation (\ref{eqn:sw_regret_defn}) is \emph{independent} of the pricing mechanism and only depends on the participating buyers $\mathcal{P}_b(t)$ and sellers $\mathcal{P}_s(t)$ at each time $t$. We will establish a lower bound on a centralized decision maker (DM), who at each time, observes all the rewards obtained by all agents thus far, and decides $\mathcal{P}_b(t)$ and $\mathcal{P}_s(t)$ for each $t$. In the Appendix in Section \ref{sec:reduction_to_semi_bandit}, we show that the actions of the DM can be coupled to that of a combinatorial semi-bandit model \cite{combes2015combinatorial}, where the base set of arms are the set of all buyers and sellers $\{B_1,\cdots,B_M\} \cup \{S_1,\cdots, S_N\}$ is the set of buyers and sellers, the valid subset of arms are those that have an equal number of buyers and sellers and the mean reward of any valid subset $\mathcal{A} \subseteq 2^{\mathcal{D}}$ is the difference between the sum of all valuations of buyers in $\mathcal{A}$ and of sellers in $\mathcal{A}$.

\begin{proposition}
The optimal action for the centralized DM is to pick $\mathcal{P}_{b}^{*} \cup \mathcal{P}_{s}^{*}$.
\label{prop:sw_lb_equivalence}
\end{proposition}\vspace{-2mm}
Proof is deferred to Section \ref{sec:reduction_to_semi_bandit}. Thus, the regret of the centralized decision maker is $    R_{\text{SB}} = T ( \sum_{i \in \mathcal{P}_b^{*}}B_i - \sum_{j \in \mathcal{P}_s^{*}} S_j ) - \sum_{t=1}^T (\sum_{i \in \mathcal{P}_b(t)} B_i - \sum_{j \in \mathcal{P}_s(t)}S_j)$. By by adding and subtracting $\sum_{j \in [M]} S_j$ to both sides, we get that $R_{\text{SB}}$ is identical to $R_{\text{SW}}$ given in Equation (\ref{eqn:sw_regret_defn}). Thus, a lower bound on $R_{SB}$ implies a lower bound on $R_{SW}$.

\begin{lemma}[Theorem $1$ \cite{combes2015combinatorial}]
Suppose the reward distributions for all agents (buyers and sellers) are unit variance gaussians. Then, the regret of any uniformly good policy\footnote{This is a standard technical condition defined in the Appendix in Section \ref{sec:additional_lb_details}} for the combinatorial semi-bandit suffers regret $\limsup_{T \to \infty} \frac{R_T}{\log(T)} = c((\mu_a)_{a\in \mathcal{D}})$, where the constant $c( (\mu_a)_{a\in \mathcal{D}}) > 0$ is strictly positive if the best and the next best subsets have different mean rewards.
\label{lem:sw_reg_lb}
\end{lemma}
A restatement and proof is given in Appendix \ref{sec:reduction_to_semi_bandit}.
 Thus our upper bound of $O(\log(T))$ social welfare regret under the decentralized setting is order optimal since even a centralized system must incur $\Omega(\log(T))$ regret.


\subsection{Proof of Lemma \ref{lem:oracle_lb_optim}}
\label{appendix:proof_lb_optim}
\begin{proof}
Assume the oracle buyer knows $B$, $S$, the average price mechanism and the fact that the seller is not strategizing. 
\\

\textbf{Case I : $B < S$}. If the buyer bids $B_t \geq S$, then the buyer will be matched with price $p_t \geq S$ and will receive an utility $U_t :=  B - p_t < 0$. If on the other hand, the buyer puts any bid $B_t < S$, then the buyer is not matched and receives an utility of $0$. Thus, the optimal choice for the oracle buyer in this case is to place any bid $B_t < S$. Thus, bidding $B_t := S-\varepsilon$, for every $\varepsilon > 0$ is optimal.
\\

\textbf{Case II : $B \geq S$}. If the buyer bids $B_t \geq S$, then the buyer will be matched with price $p_t = \frac{B_t +S}{2}$ and will receive an utility $U_t :=  B - p_t = \frac{2B - B_t - S}{2}$. Observe that the utility $U_t$ is non-decreasing in the bid-price $B_t$. If $B_t < S$ however, no match occurs and the oracle buyer will receive $0$ utility. If on the other-hand $S \leq B_t \leq B$, then $B-B_t \geq 0$ and $B - S \geq 0$. Thus, the utility $U_t \geq 0$. This along with the non-increasing nature of $U_t$ gives that the optimal action is to play $B_t = S$.
\end{proof}

\subsection{Proof of Lemma \ref{lem:lb_exact}}
\label{subsec:lb_proof_exact}

\begin{proof}
This follows the same recipe of bandit lower bounds \cite{lattimore2020bandit}. Let the bidding policy be arbitrary as in the hypothesis of the lemma.
Fix a system and denote by $\varepsilon := S-B$. We will choose an appropriate value of $\varepsilon$ later in the proof. Denote by this system where $ S-B = \varepsilon$ as $\nu$. Denote by the system in which $B = S + \varepsilon$ as $\nu^{'}$. We denote by $R_T$ and $R_T^{'}$ to be the regret obtained by the bidding policy in system $\nu$ and $\nu^{'}$ respectively. The first observation we will make is that the divergence decomposition lemma gives 
\begin{align}
    \mathbb{E}_{\nu}\left[\sum_{t=1}^TZ_t\right] D(-\varepsilon, \varepsilon) \geq \log \left( \frac{1}{2(\mathbb{P}_{\nu}(A)+\mathbb{P}_{\nu^{'}}(A^{\complement}))} \right),
    \label{eqn:divergence_decomp}
\end{align}
for any measurable event $A$. The proof of this claim follows from the well known Bretagnolle–Huber inequality (Theorem $14.2$ \cite{lattimore2020bandit}) and the divergence decomposition lemma (Lemma $15.1$ \cite{lattimore2020bandit}) to compute the divergence between the system $\nu$ and $\nu^{'}$. The formula from Lemma $15.1$ of \cite{lattimore2020bandit} when applied to our system simplifies to the LHS of Equation (\ref{eqn:divergence_decomp}) by observing the fact that in both system $\nu$ and system $\nu^{'}$, not participating in the market gives a deterministic $0$ reward. Thus, the KL divergence between the reward distributions for not participating is $0$. 
\\

The rest of the proof is to verbatim follow the proof of Theorem $15.2$ of \cite{lattimore2020bandit} to re-arrange Equation (\ref{eqn:divergence_decomp}) to yield the desired result. We denote by the event $A := \left\{\sum_{t=1}^TZ_t \geq T/2 \right\}$. Thus, trivially, $R_T \geq \frac{\varepsilon T}{2}\mathbb{P}_{\nu}[A]$ and $R_T^{'} \geq \frac{\varepsilon T}{2}\mathbb{P}_{\nu}[A^{\complement}]$. Furthermore, $\mathbb{E}\left[ \sum_{t=1}^T Z_t\right] \leq T$. 
Now, using the fact that for unit variance gaussians, $D(-\varepsilon, \varepsilon) = 2\varepsilon^2$ and plugging these estimates in Equation (\ref{eqn:divergence_decomp}), we obtain that 
\begin{align*}
    R_T + R_T^{'} \geq \frac{\varepsilon T}{4} \exp \left( -2 \varepsilon^2 T \right).
\end{align*}
As $\varepsilon > 0$ was arbitrary, we can set it to be equal to $\frac{1}{\sqrt{T}}$, and use the well known fact that for any non-negative $a,b$, we have $a+b \leq 2\max(a,b)$, to get  
\begin{align*}
    \max(R_T, R_T^{'}) \geq \frac{1}{36}\sqrt{T}.
\end{align*}
\end{proof}

\subsection{Proof of Corollary \ref{cor:red_two_arm_bandit}}

\begin{definition}[Bidding Policy]
A sequence of binary random variables $(Z_t)_{t \geq 1}$, such that for all $t \geq 1$, $Z_t \in \{0,1\}$ denotes whether the buyer participates (by placing a bid of $S$) in the $t$th round or not such that $Z_t$ is measurable with respect to the sigma-algebra generated by decisions and rewards $Z_1Y_{b,1}, \cdots, Z_{t-1}Y_{b,t-1}$ observed till time $t-1$.
\label{defn:bidding_policy}
\end{definition}

\begin{definition}[Two-armed bandit policy:]
A sequence of $\{0,1\}$ valued random variables $(\widehat{Z}_t)_{t \geq 1}$ such that for all time $t$, $\widehat{Z}_t$ is measurable with respect to a sequence of $\mathbb{R}$ valued random variables $\mathcal{F}_{t-1} := \sigma(\widehat{X}_1,\cdots,\widehat{X}_{t-1})$ such that for all time $t$, conditioned on $\widehat{Z}_t$, {\em (i)} $\widehat{X}_t$ is independent of $\mathcal{F}_{t-1}$, and {\em (ii)} $\widehat{X}_t \sim \mathcal{P}_{\widehat{Z}_t}$, where $\mathcal{P}_i$, $i \in \{1,2\}$ are two fixed probability distributions on $\mathbb{R}$. 
\label{defn:bandit_policy}
\end{definition}

This definition formalizes the intuitive notion that a sequence of binary decisions made at each time, with the decision being measurable function of all past observations and the observations themselves being independent conditioned on the arm chosen. 

\label{subsec:cor_red_two_bandits_proof_appendix}
\begin{proof}
A sequence of binary valued random variables $\{Z_`,\cdots, Z_T\}$ and $\mathbb{R}$ valued observation random variables $\{Y_1,\cdots,Y_T\}$ describes a policy for a two armed stochastic bandits if and only if {\em (i)} $Z_t$ is measurable with respect to the sigma-algebra $\mathcal{F}_{t-1} := \sigma(Y_1, \cdots, Y_{t-1})$ generated by all observed rewards upto time $t-1$, and {\em (ii)} conditioned on the arm $Z_t$, the observed reward $Y_t$ is conditionally independent of $(Y_1,\cdots,Y_{t-1})$ and $(Z_1,\cdots, Z_{t-1})$ the actions and rewards obtained in the past. It is easy to observe both of these for the bidding system described before. 
\end{proof}

\subsection{Extension to the multi-agent setting}
\label{sec:multi_agent_lb_extension}

The simplified system in \ref{subsec:simple_system_lb} specified a single buyer and seller system that had no competition as the seller's behaviour was fixed. In this section, we show that a $\sqrt{T}$ minimax lower bound is inevitable even in an appropriately simplified multi-agent system that decouples learning from competition.

In this simplified setting, we assume {\em(i)} the selling price $p_t := p^{*}$ is set constant for every round $t$, and {\em (ii)} there is no shortage of goods, i.e., any buyer(seller) that wants to participate in a given round by paying(selling at) $p^{*}$, can buy(sell) so. Thus, under this setup, the entire market is a decoupled union of individual agents interacting against a fixed environment dictated by a price $p^{*}$. It can be observed, that for any buyer(seller) in this simplified setting, the lower bound from Lemma \ref{lem:lb_exact} applies verbatim following the same coupling arguments. 

\begin{corollary}
For every bidding policy for the general system, for any seller $j\in [M]$ there exists a system such that $\mathbb{E}[R_{s,j}(T)] \geq \frac{1}{36}\sqrt{T}$. Similarly, for any buyer $i\in [N]$ there exists a system such that $\mathbb{E}[R_{b,i}(T)] \geq \frac{1}{36}\sqrt{T}$. 
\end{corollary}

\subsection{Reduction of the centralized DM in Section \ref{sec:sw_regret_lb} to a combinatorial semi-bandit model}
\label{sec:reduction_to_semi_bandit}

A combinatorial semi-bandit \cite{cesa2012combinatorial} is the following variant of the standard $\mathcal{D} := \{1,\cdots, D\}$ armed bandit problem. At each time $t$, the decision maker chooses a subset $\mathcal{A}_t \subseteq \mathcal{D}$ from a fixed set of subsets $\boldsymbol{\mathcal{A}} \subseteq 2^{\mathcal{D}}$. At each time $t$, the received reward is $\sum_{a \in \mathcal{A}_t}X_t(a)$, where for every arm $a \in \mathcal{D}$, $(X_t(a))_{t \geq 1}$ is an i.i.d. sequence with mean $\mu_a$. Later in Lemma \ref{lem:sw_reg_lb}, we will assume that the reward distriutions are unit-variance gaussians. The arm-means $(\mu_a)_{a \in \mathcal{D}}$ are a-priori unknown to the decision maker. The goal of the decision maker is to minimize the cumulative regret defined as $R_{SB} := \max_{\mathcal{A} \in \boldsymbol{\mathcal{A}}}T \sum_{a \in \mathcal{A}} \mu_a - \sum_{t=1}^T \sum_{a \in \mathcal{A}_t}\mu_a$.

Consider the centralized DM who at each time decides  the participating buyers and sellers $\mathcal{P}_b(t)$ and $\mathcal{P}_s(t)$. The base set of `arms' $\mathcal{D} := \{B_1,\cdots,B_N\} \cup \{S_1,\cdots, S_M\}$ is the set of buyers and sellers. The mean reward $\mu_a$ for any buyer $a$ is their unknown valuation, while $\mu_a$ for any seller $a$ is the \emph{negative} of their true valuation. The set of allowed subsets $\boldsymbol{\mathcal{A}}$ are those subsets of $\mathcal{D}$ that contain the same number of buyers and sellers.

\subsubsection{Proof of Proposition \ref{prop:sw_lb_equivalence}}
\begin{proof}
The proof follows from two observations. {\em (i)} the mean reward for any buyer is positive and that for the seller is negative. {\em (ii)} the set of allowed subsets to play consists of an equal number of buyers and sellers. Thus, from definition of $\mathcal{P}_b^{*}$ and $\mathcal{P}_s^{*}$, every other subset satisfying condition $(ii)$ will not have mean reward strictly larger than that of $\mathcal{P}_{b}^{*} \cup \mathcal{P}_{s}^{*}$. 
\end{proof}

Thus, the regret of the centralized decision maker is $    R_{\text{SB}} = T ( \sum_{i \in \mathcal{P}_b^{*}}B_i - \sum_{j \in \mathcal{P}_s^{*}} S_j ) - \sum_{t=1}^T (\sum_{i \in \mathcal{P}_b(t)} B_i - \sum_{j \in \mathcal{P}_s(t)}S_j)$. By by adding and subtracting $\sum_{j \in [M]} S_j$ to both sides, we get that
\begin{align}
R_{\text{SB}}
    = T ( \sum_{i \in \mathcal{P}_b^{*}}B_i + \sum_{j \in [M]\setminus\mathcal{P}_s^{*}} S_j ) -  \sum_{t=1}^T (\sum_{i \in \mathcal{P}_b(t)} B_i + \sum_{j \in [M]\setminus\mathcal{P}_s(t)}S_j).
    \label{eqn:comb_bandit}
\end{align}
As Equation (\ref{eqn:comb_bandit}) is identical to the definition of the social regret in Equation (\ref{eqn:sw_regret_defn}), a lower bound on $R_{SB}$ implies a lower bound on the social-welfare regret $R_{SW}$.

\begin{lemma}[Theorem $1$ \cite{combes2015combinatorial}]
Suppose the reward distributions for all agents (buyers and sellers) are unit variance gaussians. Then, the regret of any uniformly good policy\footnote{This is a standard technical condition defined in the Appendix in Section \ref{sec:additional_lb_details}} for the combinatorial semi-bandit suffers regret $\limsup_{T \to \infty} \frac{R_T}{\log(T)} = c(\boldsymbol{\mathcal{A}}; (\mu_a)_{a\in \mathcal{D}})$, where the constant $c(\boldsymbol{\mathcal{A}}; (\mu_a)_{a\in \mathcal{D}}) > 0$ is strictly positive if the best and the next best subsets have different mean rewards.
\label{lem:sw_reg_lb}
\end{lemma}

 Thus our upper bound of $O(\log(T))$ social welfare regret under the decentralized setting is order optimal since even a centralized system must incur $\Omega(\log(T))$ regret.

\subsection{Additional Details on Combinatorial Semi-Bandits}
\label{sec:additional_lb_details}

For completeness, we collect all relevant definitions and statements from \cite{combes2015combinatorial} that are used in the lower bound proof in Section \ref{sec:sw_regret_lb}. Recall the setup -- at each time $t$, the decision maker can choose a subset of arms from a given set $\boldsymbol{\mathcal{A}} \subseteq 2^{\mathcal{D}}$ of subsets. In this setting, the mean reward obtained by any arm is a unit variance gaussian random variable. Let $\Theta$ denote the set of possible parameters for the means which in our example is the set of $M+N$ means, i.e., $\Theta := \mathbb{R}^{M+N}$, and $\theta \in \Theta$ denote the mean reward for the elements in $\mathcal{D}$. For any subset $\mathcal{A} \in \boldsymbol{\mathcal{A}}$, let $\mu_{\mathcal{A}}(\theta) := \sum_{a \in \mathcal{A}} \theta_a$. We use $\mu_{\mathcal{A}}$ instead of  $\mu_{\mathcal{A}}(\theta)$ when the underlying instance $\theta$ is clear. Let the subset with highest mean be $\mathcal{A}^{*}(\theta) := \arg\max_{\mathcal{A} \in \boldsymbol{\mathcal{A}}} \sum_{a \in \mathcal{A}}\theta_a$, and $\mu^*(\theta)$ be the highest reward, i.e., $\mu^{*}(\theta) = \max_{\mathcal{A} \in \boldsymbol{\mathcal{A}}}\mu_{\mathcal{A}}(\theta) := \mu_{\mathcal{A}^{*}(\theta)}(\theta)$.  An assumption we make throughout in Section \ref{sec:sw_regret_lb} is that $\mathcal{A}^{*}$ is unique and the gap is strictly positive, i.e.

\begin{align}
    \Delta(\mathcal{A}; \theta) := \min_{\mathcal{A} \in \boldsymbol{\mathcal{A}} \setminus \{ \mathcal{A}^{*} \}} (\mu^*(\theta) - \mu_{\mathcal{A}}(\theta)) > 0.
    \label{eqn:assumption_lb}
\end{align}

A policy $\pi$ for the decision maker is a measurable function that at each time $t$, maps all the observed rewards of the arms pulled to a subset $\mathcal{A}^\pi_{t}$ to play. Let $R^{\pi}(T)$ be the $T$ round regret of the policy $\pi$ for an instance $\theta\in \Theta$, which is defined as $R^{\pi}(T; \theta) = \sum_{t=1}^{T}(\mu_{\mathcal{A}^*}(\theta) - \mu_{\mathcal{A}^\pi_{t}}(\theta))$.

\begin{definition}[Uniformly Good]
A policy $\pi$ for the DM is uniformly good, if for all problem settings $\theta \in \Theta$ satisfying assumption in Equation (\ref{eqn:assumption_lb}), $R^{\pi}(T; \theta) = o(T^{\alpha})$, for every $\alpha \in (0,1)$.  
\end{definition}

\begin{theorem}[Theorem $1$ in \cite{combes2015combinatorial}]
\label{thm:combLB}
For any instance $\theta \in \Theta$ satisfying the assumption in Equation (\ref{eqn:assumption_lb}) and any uniformly good policy $\pi$, $\limsup_{T \to \infty} \frac{R^{\pi}(T;\theta)}{\log(T)} \geq c(\boldsymbol{\mathcal{A}}; \theta)$, where $c(\boldsymbol{\mathcal{A}}; \theta)$ is given by the solution of the following optimization problem
\begin{align*}
   c(\boldsymbol{\mathcal{A}}; \theta) &:= \inf_{x \in \mathbb{R}_{+}^{|\boldsymbol{\mathcal{A}}|}} \sum_{\mathcal{A} \in \boldsymbol{\mathcal{A}}}x_{\mathcal{A}}(\mu^{*}(\theta) - \mu_{\mathcal{A}}(\theta)) \text{ such that } \\ & \sum_{a \in \mathcal{D}} \left( \sum_{\mathcal{A} \in \boldsymbol{\mathcal{A}} \text{ s.t. } a \in \mathcal{A} }x_{\mathcal{A}} \right) kl( \theta_a, \lambda_a ) \geq 1, \text{   } \forall \lambda \in B(\theta),
\end{align*}
where $B(\theta) := \{ \lambda \in \Theta, \text{ s.t. } \theta_i = \lambda_i~\forall i \in \mathcal{A}^{*}(\theta), \mu^{*}(\lambda) > \mu^{*}(\theta) \}$, and for any $u,v \in \mathbb{R}$, $kl(u,v) = \frac{1}{2}( u-v)^2$.
\label{thm:lb_comb}
\end{theorem}

The proof of this theorem is given as the proof of theorem $1$ in \cite{combes2015combinatorial}. Although the proof in \cite{combes2015combinatorial} is given in the case for Bernoulli distributed arms, the proof follows verbatim for gaussian distributed arms by using the fact that the KL divergence between two unit variance gaussians with means $u, v \in \mathbb{R}$ is $kl(u,v) := \frac{1}{2}(u-v)^2$.

{\color{black} 
We give a simple lower bound to $c(\boldsymbol{\mathcal{A}}; \theta)$ in the following lemma.

\begin{lemma}
\begin{align*}
c(\mathbb{\mathcal{A}}; \theta) 
\geq \sum_{i= K^*+1}^{N} \frac{2}{\min(B_{K^*}, S_{K^*+1}) - B_i} + \sum_{j= K^*+1}^{M} \frac{2}{S_j - \max(S_{K^*}, B_{K^*+1})}
\end{align*}
\label{lem:LB_kl_algebra}
\end{lemma}
\begin{proof}

For ease of exposition, we will use the notation that for any buyer $i \in \{1,\cdots,N\}$, $\theta_{b,i}$ to be the true valuation of buyer $i$ under model $\theta$. Similarly, for any seller $j \in \{1,\cdots,M\}$, we denote by $\theta_{s,j}$ to be the true valuation of seller $j$ under model $\theta$. Observe that under the mapping to the combinatorial bandits, we denote by all agents' indices by the set of cardinality $M+N$, denoted by $\mathcal{D}$. We denote by the first $N$ indices of $\mathcal{D}$ to represent the buyers and the last $M$ indices of $\mathcal{D}$ to denote the sellers.

 For every $\varepsilon > 0$, let 
\begin{align*}
       B_0(\theta) &= \{ \lambda \text{ s.t. }, \forall i, j\leq  K^* \lambda_{b,i} = \theta_{b,i}, \lambda_{s,j} = \theta_{s,j}\},\\
       B_1(\theta; \varepsilon) &= \{ \lambda \in B_0(\theta); \exists \text{ exactly one } i' \in \{K^*+1, \dots, N\},  \lambda_{b,i'} = \min(\theta_{b,K^*},\theta_{s,K^*+1})  + \varepsilon \},\\
   B_2(\theta; \varepsilon) &= \{ \lambda \in B_0(\theta); \exists \text{ exactly one } j' \in \{K^*+1, \dots, M\},  \lambda_{s,j'} = \max(\theta_{s,K^*}, \theta_{b,K^*+1}) - \varepsilon \}.
\end{align*}
Observe by definition that for all $\varepsilon > 0$, $B_1(\theta, \varepsilon) \cap B_2(\theta, \varepsilon) = \emptyset$ and $B_1(\theta, \varepsilon) \cup B_2(\theta, \varepsilon) \subset B(\theta)$. Define by $c_{\varepsilon}(\boldsymbol{\mathcal{A}}; \theta)$ to be the solution to the following optimization problem 
\begin{align}
   c_{\varepsilon}(\boldsymbol{\mathcal{A}}; \theta) &:= \inf_{x \in \mathbb{R}_{+}^{|\boldsymbol{\mathcal{A}}|}} \sum_{\mathcal{A} \in \boldsymbol{\mathcal{A}}}x_{\mathcal{A}}(\mu^{*}(\theta) - \mu_{\mathcal{A}}(\theta)) \text{ such that } \label{eqn:op_var_lb} \\ & \sum_{a \in \mathcal{D}} \left( \sum_{\mathcal{A} \in \boldsymbol{\mathcal{A}} \text{ s.t. } a \in \mathcal{A} }x_{\mathcal{A}} \right) kl( \theta_a, \lambda_a ) \geq 1, \text{   } \forall \lambda \in \mathcal{B}_1(\theta, \varepsilon) \cup B_2(\theta, \varepsilon) \nonumber.
\end{align}
The optimization problem for $c_{\varepsilon}(\boldsymbol{\mathcal{A}}; \theta)$ differs from that for $c(\boldsymbol{\mathcal{A}}; \theta)$ in Theorem \ref{thm:lb_comb} in the constraints. Since $B_1(\theta,\varepsilon) \cup B_2(\theta, \varepsilon) \subset B(\theta)$, we have that $c(\boldsymbol{\mathcal{A}}; \theta) \geq c_{\varepsilon}(\boldsymbol{\mathcal{A}}; \theta)$. As this in-equality holds for all $\varepsilon > 0$, we have that 
\begin{align*}
    c(\boldsymbol{\mathcal{A}}; \theta) \geq \sup_{\varepsilon > 0} c_{\varepsilon}(\boldsymbol{\mathcal{A}}; \theta).
\end{align*}
In the rest of this proof, we will develop a lower bound for $\sup_{\varepsilon > 0}c_{\varepsilon}(\boldsymbol{\mathcal{A}}; \theta)$. 

For notational simplicity, we denote by $y_{a} := \sum_{\mathcal{A} \in \boldsymbol{\mathcal{A}} : a \in \mathcal{A}}x_{\mathcal{A}}$. The constraint in the optimization problem in Equation (\ref{eqn:op_var_lb}) yields that, for all buyers $i' \in \{K^*+1, \cdots, N\}$ and all sellers $j' \in \{K^*+1,\cdots, M\}$, we need to have 
\begin{align*}
    &y_{b,i'} \geq \frac{1}{kl(\theta_{b,i'},\min(\theta_{b,K^*}, \theta_{s,K^*+1})+\varepsilon)},
    \text{ and }
    y_{s,j'} \geq \frac{1}{kl(\theta_{s,j'}, \max(\theta_{b,K^*+1}, \theta_{s,K^*}) - \varepsilon)}.
\end{align*}

Additionally, monotonicity yields that for any buyer $i' \in \{K^*+1, \cdots, N\}$,
\begin{align*}
 \mathcal{A}_{b}^{(i')} := {\arg\min}_{\mathcal{A} \in \boldsymbol{\mathcal{A}}: B_{i'}\in \mathcal{A}}(\mu^*(\theta) - \mu_{\mathcal{A}}(\theta)) =   \begin{cases}
  \{B_1, \dots, B_{(K^*-1)}, B_{i'}\} \cup \{S_1, \dots, S_{K^*}\}, &B_{K^*} < S_{K^*+1},\\
  \{B_1, \dots, B_{K^*}, B_{i'}\} \cup \{S_1, \dots, S_{K^*}, S_{K^*+1}\}, &B_{K^*} > S_{K^*+1}.
 \end{cases}
\end{align*}

Similarly, for any seller $j' \in \{K^*+1, \cdots, N\}$, 
\begin{align*}
 \mathcal{A}_{s}^{(j')} := {\arg\min}_{\mathcal{A} \in \boldsymbol{\mathcal{A}}: S_{j'}\in \mathcal{A}}(\mu^*(\theta) - \mu_{\mathcal{A}}(\theta)) =   \begin{cases}
  \{B_1, \dots, B_{K^*}\} \cup \{S_1, \dots, S_{(K^*-1)}, S_{j'}\}, &B_{K^*+1} < S_{K^*},\\
  \{B_1, \dots, B_{K^*}, B_{K^*+1}\} \cup \{S_1, \dots, S_{K^*}, S_{j'}\}, &B_{K^*+1} > S_{K^*}.
 \end{cases}
\end{align*}

Note that for all $i', j', i'', j''\geq (K^*+1)$, $i'\neq i''$  and $j'\neq j''$, $B_{i''} \notin \mathcal{A}_{b}^{(i')}$, and   $S_{j''} \notin \mathcal{A}_{s}^{(j')}$.
Therefore, the minimizer of Equation (\ref{eqn:op_var_lb}) is given by assigning 
\begin{align*}
   &\forall~ i' \in \{K^*+1, \cdots, N\},~ x_{\mathcal{A}^{(i')}_b} := \frac{1}{kl(\theta_{b,i'},\min(\theta_{b,K^*}, \theta_{s,K^*+1})+\varepsilon)},\\
   &\forall~j' \in \{K^*+1, \dots, M\},~ x_{\mathcal{A}^{(j')}_s} := \frac{1}{kl(\theta_{s,j'}, \max(\theta_{b, K^*+1}, \theta_{s,K^*}) - \varepsilon)},\\
   &\text{ and for all other sets } \mathcal{A}, x_{\mathcal{A}} = 0. 
\end{align*} 
 
This ensures 
for all buyers $i' \in \{K^*+1, \cdots, N\}$ and all sellers $j' \in \{K^*+1,\cdots, M\}$, 
\begin{align*}
    &y_{b,i'} = \frac{1}{kl(\theta_{b,i'},\min(\theta_{b,K^*}, \theta_{s,K^*+1})+\varepsilon)},
    \text{ and }
    y_{s,j'} = \frac{1}{kl(\theta_{s,j'}, \max(\theta_{b,K^*+1}, \theta_{s,K^*}) - \varepsilon)}.
\end{align*}

Thus, we get 
\begin{align*}
     c_{\varepsilon}(\boldsymbol{\mathcal{A}}; \theta) &= \sum_{i' = K^*+1}^{N} \frac{\mu^*(\theta) - \mu_{\mathcal{A}^{(i')}_b}(\theta)}{kl(\theta_{b,i'},\min(\theta_{b,K^*}, \theta_{s,K^*+1})+\varepsilon)} 
     + \sum_{j' = K^*+1}^M \frac{\mu^*(\theta) - \mu_{\mathcal{A}^{(j')}_s}(\theta)}{kl(\theta_{s,j'}, \max(\theta_{b,K^*+1}, \theta_{s,K^*}) - \varepsilon)},\\
     &= \sum_{i' = K^*+1}^{N} \frac{\min(\theta_{b, K^*}, \theta_{s, K^*+1})+\varepsilon - \theta_{b,i'}}{kl(\theta_{b,i'},\min(\theta_{b,K^*}, \theta_{s,K^*+1})+\varepsilon)} 
     + \sum_{j' = K^*+1}^M \frac{\theta_{s,j'} - \max(\theta_{s,K^*}, \theta_{b,K^*+1})+\varepsilon}{kl(\theta_{s,j'}, \max(\theta_{b,K^*+1}, \theta_{s,K^*}) - \varepsilon)}
\end{align*}

Optimizing over $\varepsilon >0$ and using the fact that $kl(u,v) := \frac{1}{2}(u-v)^2$ yields the result. 
\end{proof} 

{\color{black}
\subsection{Lower Bound for Non-participating agents}
\label{sec:lb_sketch_non_participating}

The connection to combinatorial bandits also gives an evidence towards a $\Omega(\log(T))$ lower-bound for regret for any non-participating agent. 
The crux of the lower bound comes from Theorem $1$ in \citep{graves1997asymptotically} which we state here using notations of our setting. For every subset $\mathcal{A} \in \boldsymbol{\mathcal{A}}$, time-horizon $T$ and policy $\pi$, denote by $N_{\mathcal{A}}^{\pi}(T) \leq T$ to be the number of times, subset $\mathcal{A}$ is played, i.e., matched by the centralized DM. Theorem $1$ from \citep{graves1997asymptotically} restated is the following.

\begin{lemma}
    For any uniformly good policy $\pi$ and any $\mathcal{A} \in \boldsymbol{\mathcal{A}}$ such that $\mathcal{A} \in \mathcal{A}^*$ is not the optimal subset,
    \begin{align}
    \liminf_{T \to \infty}\frac{\mathbb{E}[N_{\mathcal{A}}^{\pi}(T)]}{\log(T)} \geq \sup_{\lambda \in \mathcal{B}(\theta)}\frac{1}{\sum_{a \in \mathcal{A}}kl(\theta_{a},\lambda_a)},
    \label{eqn:lb_numn_times_subopt}
\end{align}
where $\mathcal{B}(\theta)$ is defined in Theorem \ref{thm:combLB}.
\label{lem:lb_num_times}
\end{lemma}

 Following identical calculations as in Proof of Lemma \ref{lem:LB_kl_algebra}, consider the set $\widehat{\mathcal{A}}^{(b;j)} := \{B_1, \cdots, B_{K^*}, B_{j} \} \cup \{S_1, \cdots, S_{K^*}, S_{K^*+1} \}$ for some $j \geq K^*+1$, i.e., the set where one non-participating buyer $j$ and a non-participating seller $K^*+1$ is played. For this set $\widehat{\mathcal{A}}^{(b,j)}$,  $\sup_{\lambda \in \mathcal{B}(\theta)} \frac{1}{\sum_{a \in \widehat{\mathcal{A}}}kl(\theta_a, \lambda_a)}$ can be upper bounded by a finite quantity similar to the proof of Lemma \ref{lem:LB_kl_algebra}.


Now, similar to the other lower bound proofs in the paper,we make the assumption that the transaction price is constant $p^*$ in all rounds and is independent of the bids submitted by the agents. Thus, each round a sub-optimal subset $\mathcal{A}$ is played, the agents in the set $\mathcal{A}$ incur a constant $O(1)$ regret as the transaction price is fixed and unchanging.

}

}

\section{Deviations and Incentives of Agents}\label{appendix:incentives} We now discuss how agents may deviate under average mechanism under their own incentive. We limit ourselves to deviations from \emph{myopic oracle} agents, each of whom optimizes her single round reward, and possess an oracle knowledge of her own valuation. Our incentives are shaped by symmetric equilibrium in double auction~\cite{wilson1985incentive}, where all non-strategic buyers employ same strategy, and the same hold for non-strategic sellers. In particular, all the non-strategic agents use confidence-based bidding. The strategic deviant agents do not deviate from a strategy if incremental deviations in bids do not improve their own reward. 

Under the above setup, for average mechanism only the price-setting agents (i.e. the $K^*$-th buyer and $K^*$-th seller) have incentive to deviate from their true valuation to increase their single-round reward (c.f. Section 5 in \cite{wilson1985incentive}). For other agents, deviations from reporting their true value does not improve their instantaneous reward. Thus average mechanism is \emph{truthful for non-price-setting agents}. 

For the price-setting agents, the incentives, and impact on regret are as follows.

    1.  \emph{Only $K^*$-th buyer deviates:} The $K^*$-th buyer has an incentive to set her bid \emph{close}\footnote{In this context, \emph{close} means $\epsilon$ larger bids for buyers, and $\epsilon$ smaller bids for sellers for $\epsilon \gtrapprox 0$.} to $\max(S_{K^*}, B_{K^*+1})$. The long term average price is now set close to $(S_{K^*} + \max(S_{K^*}, B_{K^*+1}))/2$. When compared to Average mechanism outcomes this leads to $(B_{K^*} - \max(S_{K^*}, B_{K^*+1}))/2$ average surplus in each round for participating buyers, and the same average deficit in each round for participating sellers.
 
    2.  \emph{Only $K^*$-th seller deviates:} The $K^*$-th seller has an incentive to set her bid close to $\min(B_{K^*}, S_{K^*+1})$.  With a  long term average price of $(B_{K^*} + \min(B_{K^*}, S_{K^*+1}))/2$, each seller has a per round  $(S_{K^*} - \min(B_{K^*}, S_{K^*+1}))/2$ surplus, and each buyer has the same average deficit in each round.  

    3.  \emph{Both $K^*$-th seller and buyer deviate:} The $K^*$-th seller has an incentive to set her bid close to $\min( (B_{K^*}+S_{K^*})/2, S_{K^*+1})$. Whereas, for the $K^*$-th buyer the bid is close to $\max((B_{K^*}+S_{K^*})/2, B_{K^*+1})$. We can derive the long term average price, and surplus and deficit for the agents similarly. We leave out the exact expressions due to space limitations.

We acknowledge that the study of incentive compatibility under the notions of equilibrium in sequential games with incomplete information --  where an agent can strategize thinking about her long term consequences (c.f. ~\cite{ponssard1973zero,kohlberg1974repeated}) in presence of learning is out of scope for this paper. This is similar to other contemporary works on learning in repeated games~\cite{liu2020competing,sankararaman2021dominate,liu2021bandit,basu2021beyond}.
\section{Synthetic Experiments} \label{appendix:experiment}
In this section, we present some additional synthetic experiments to study the impact of different parameters on the performance of our algorithm. Our methodology is same as mentioned in Section~\ref{sec:expt}. Recall, that the negative regret for participant buyers/sellers is expected, as the regret increases as $(p(t) - p^*)$ for sellers, and $(p^* - p(t))$ for buyers.

\subsection{Example Systems:}
In Figure~\ref{fig:Fig885app}, we have a $8\times 8$ system with $5$ matches.  The non-participant regret, for both buyers and sellers, converges and assumes the $log(T)$. The participant regret, for both buyers and sellers, has more noise and the envelope grows as $O(\sqrt{T})$. Note that the  regret that comes from price difference has opposite sign for buyers and sellers in each sample path. Hence, if regret plot of buyers is increasing with $T$ then it will decrease for sellers, and vice versa. We defer the simulation studies of other systems to appendix.We see this is Figure~\ref{fig:Fig151510} where we have a $15\times 15$ system with $10$ matches simulated for $T= 100k$. The rest of the behavior in Figure~\ref{fig:Fig151510}  is similar to  Figure~\ref{fig:Fig885app}. 

\begin{figure}
     \centering
    \begin{subfigure}[t]{0.49\linewidth}
        \raisebox{-\height}{\includegraphics[width=\textwidth]{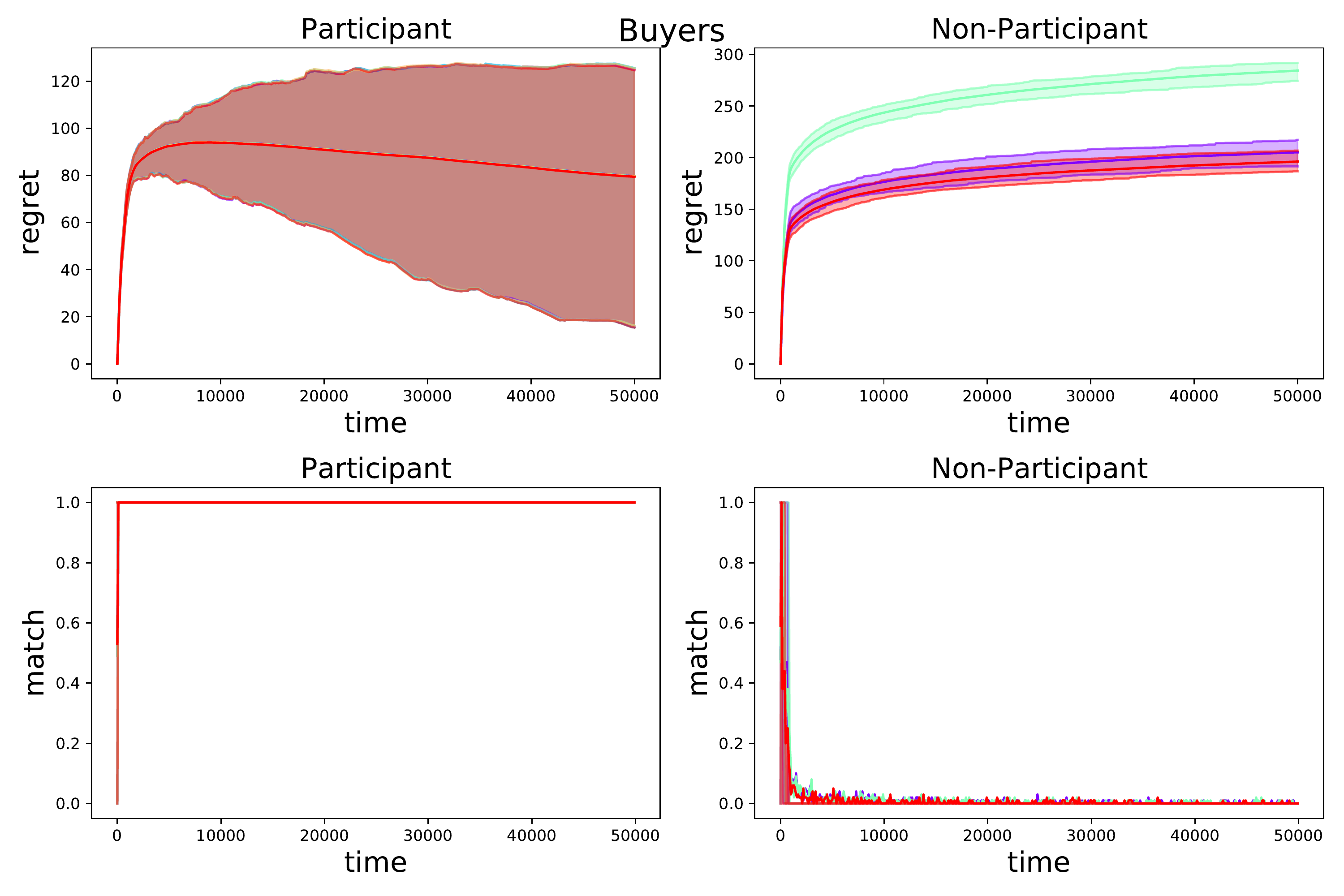}}
        \caption{Regret and Matching of Buyers}
    \end{subfigure}
    \hfill
    \begin{subfigure}[t]{0.49\linewidth}
        \raisebox{-\height}{\includegraphics[width=\textwidth]{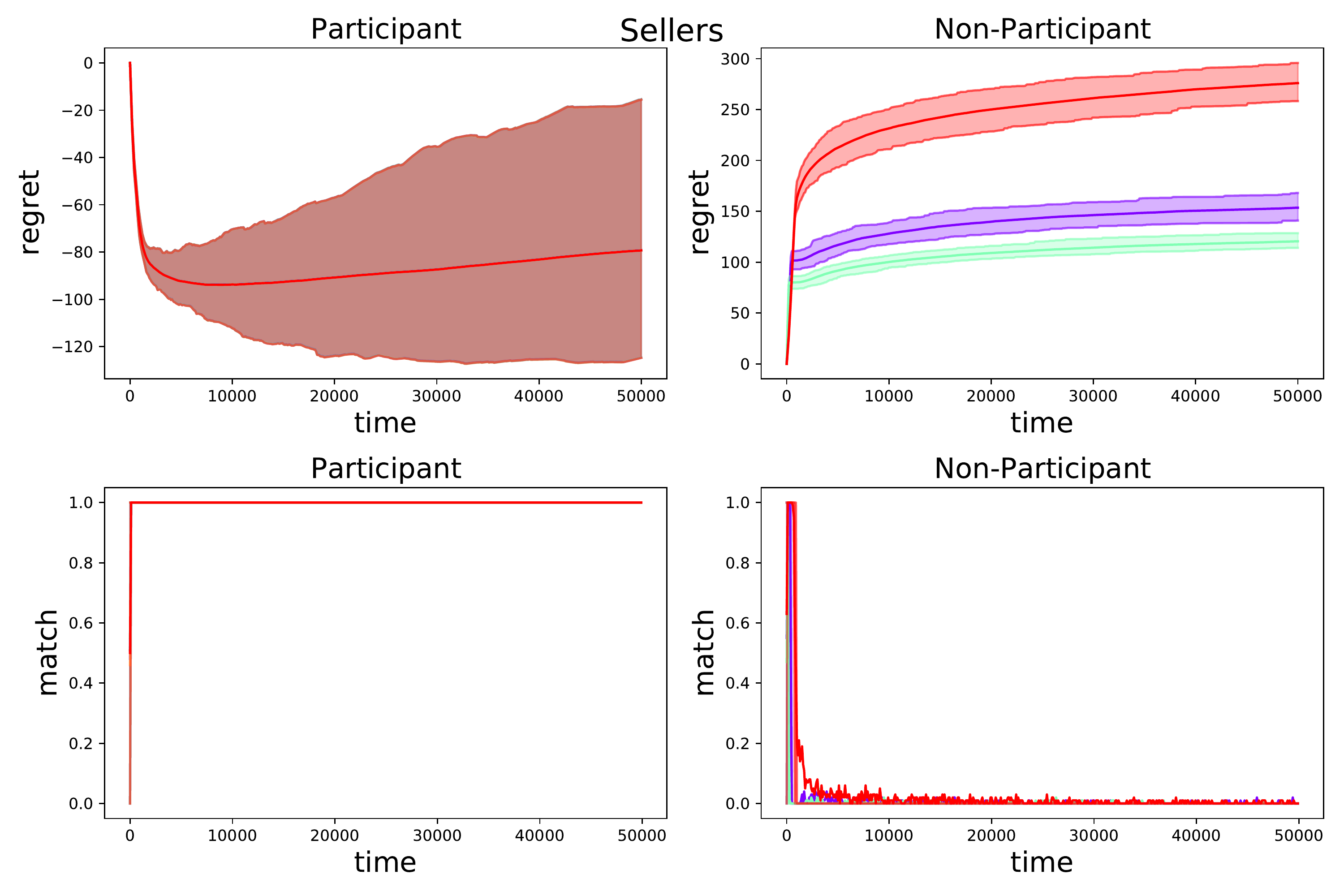}}
        \caption{Regret and Matching of Sellers}
    \end{subfigure}
    \begin{subfigure}[t]{0.34\linewidth}
        \raisebox{-\height}{\includegraphics[width=\textwidth]{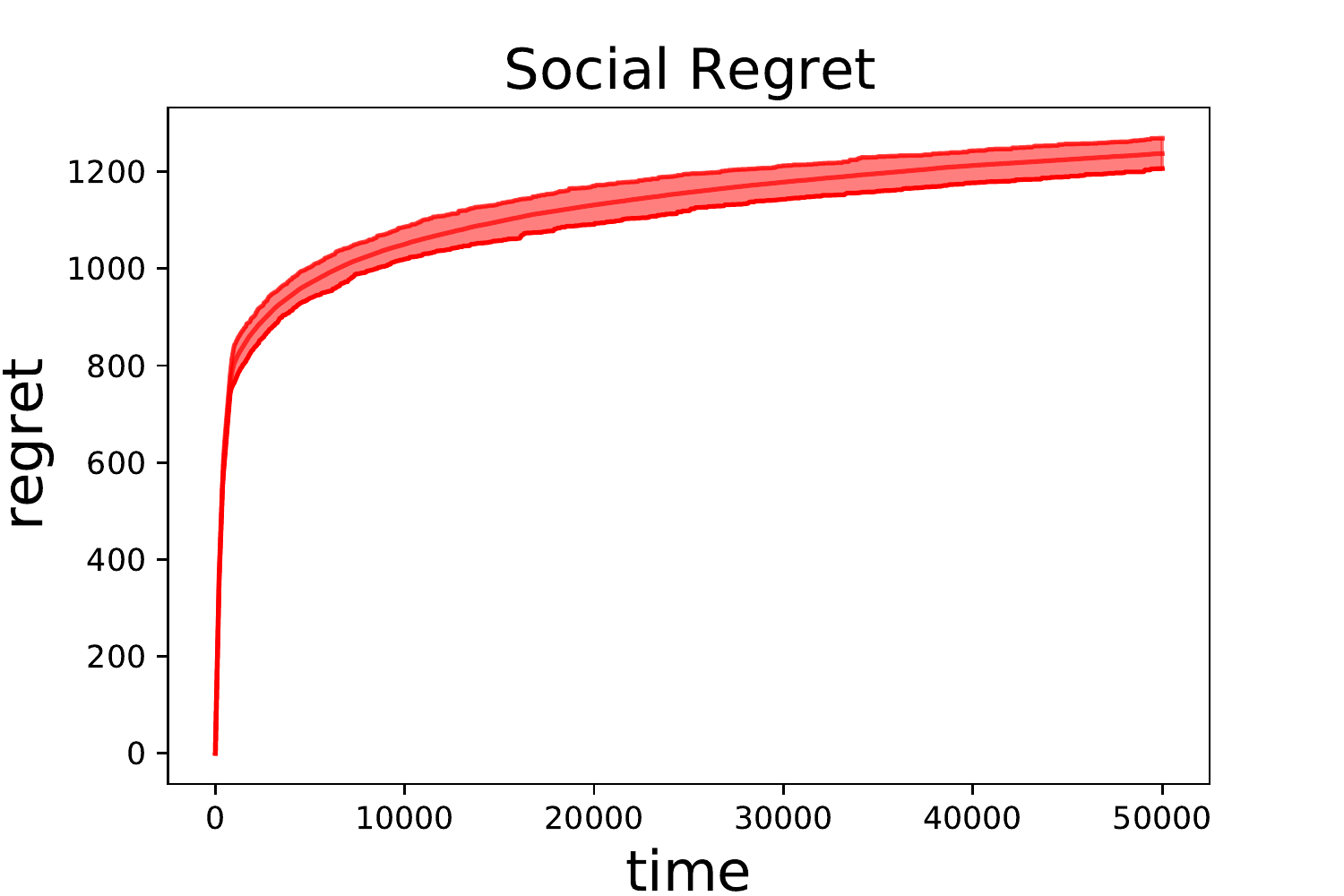}}
    \caption{Social Regret} 
    \end{subfigure}
    \hfill
    \begin{subfigure}[t]{0.64\linewidth}
        \raisebox{-\height}{\includegraphics[width=\textwidth]{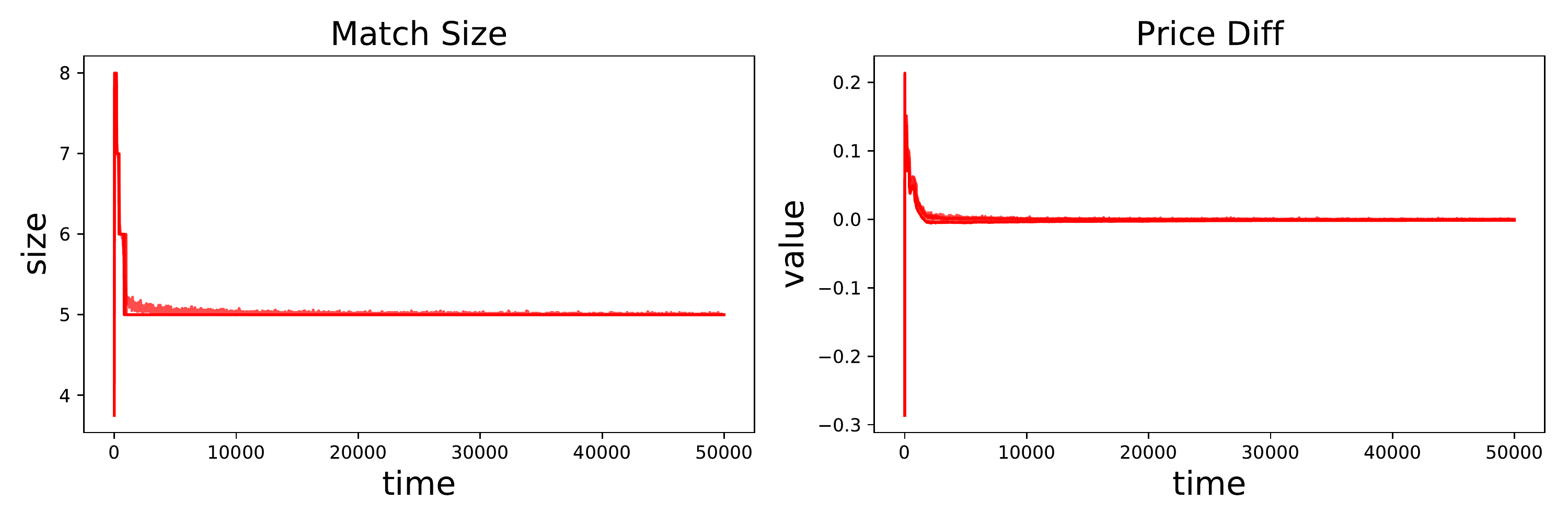}}
    \caption{Convergence of $K(t)$, $(p(t)- p^*)$} 
    \end{subfigure}
    \caption{Double Auction $N=8$, $M=8$, $K^*=5$, $\Delta = 0.2$, $\alpha_1=4$, and $\alpha_2=8$}
    \label{fig:Fig885app}
\end{figure}

\begin{figure}
     \centering
    \begin{subfigure}[t]{0.49\linewidth}
        \raisebox{-\height}{\includegraphics[width=\textwidth]{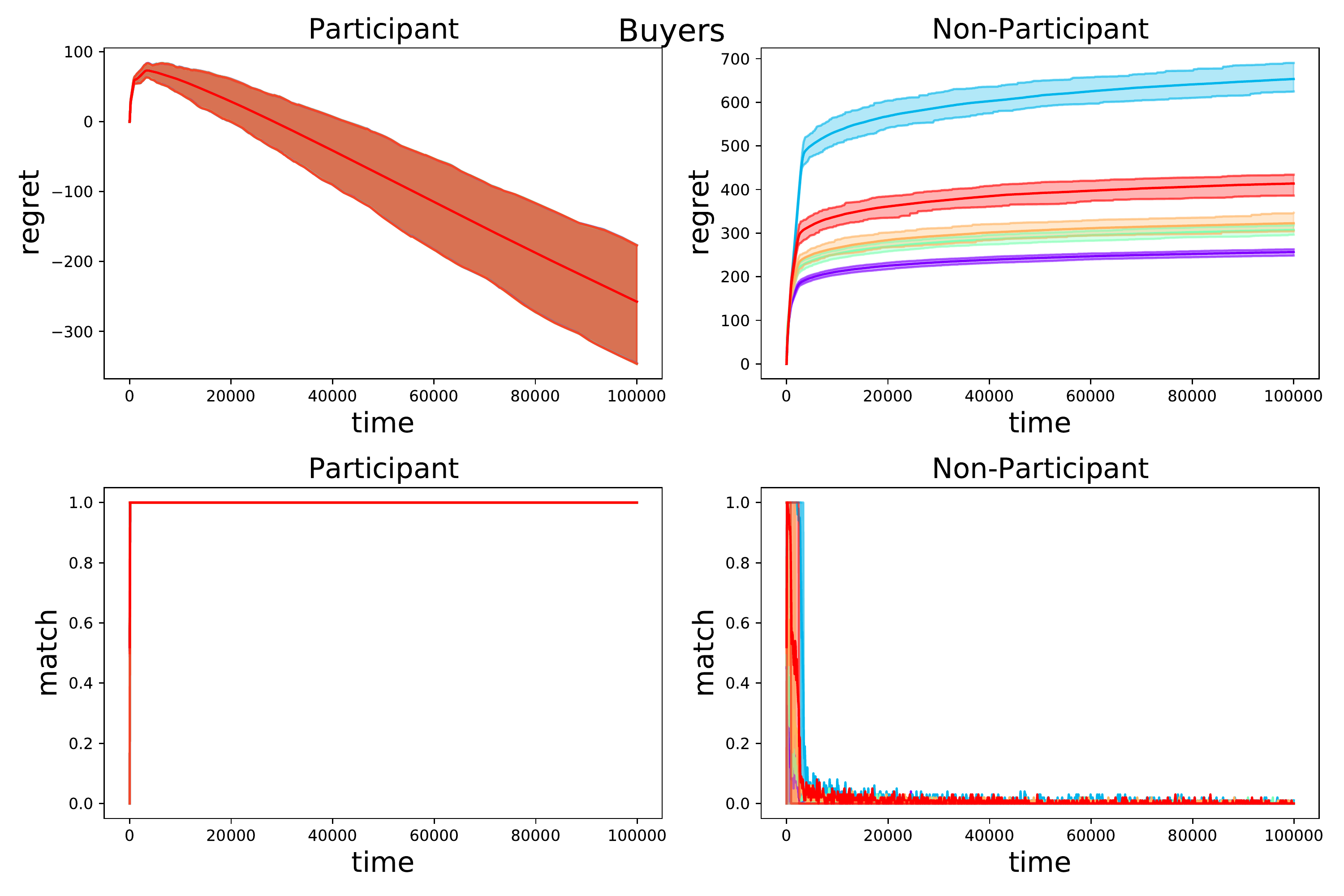}}
        \caption{Regret and Matching of Buyers}
    \end{subfigure}
    \hfill
    \begin{subfigure}[t]{0.49\linewidth}
        \raisebox{-\height}{\includegraphics[width=\textwidth]{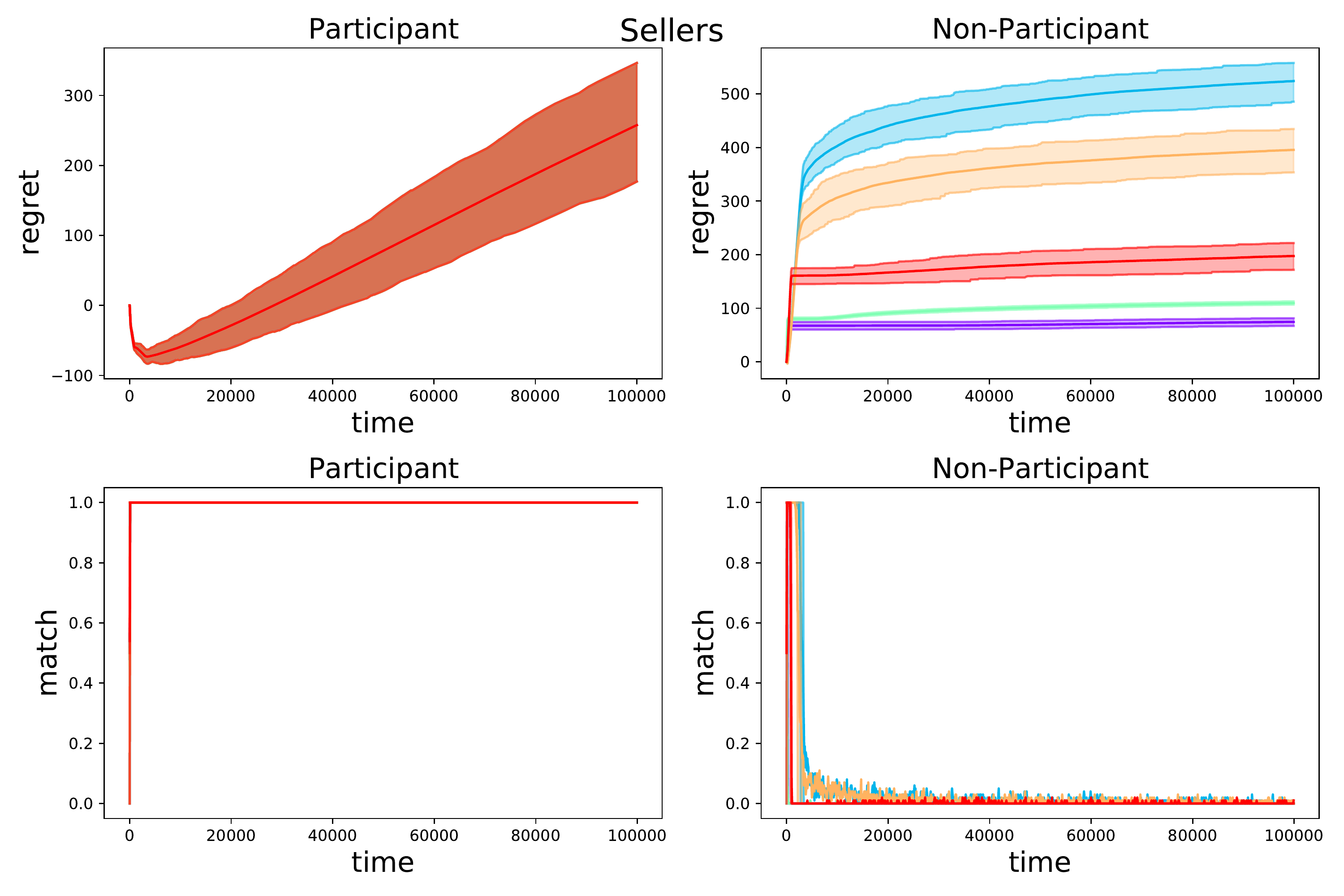}}
        \caption{Regret and Matching of Sellers}
    \end{subfigure}
    \begin{subfigure}[t]{0.34\linewidth}
        \raisebox{-\height}{\includegraphics[width=\textwidth]{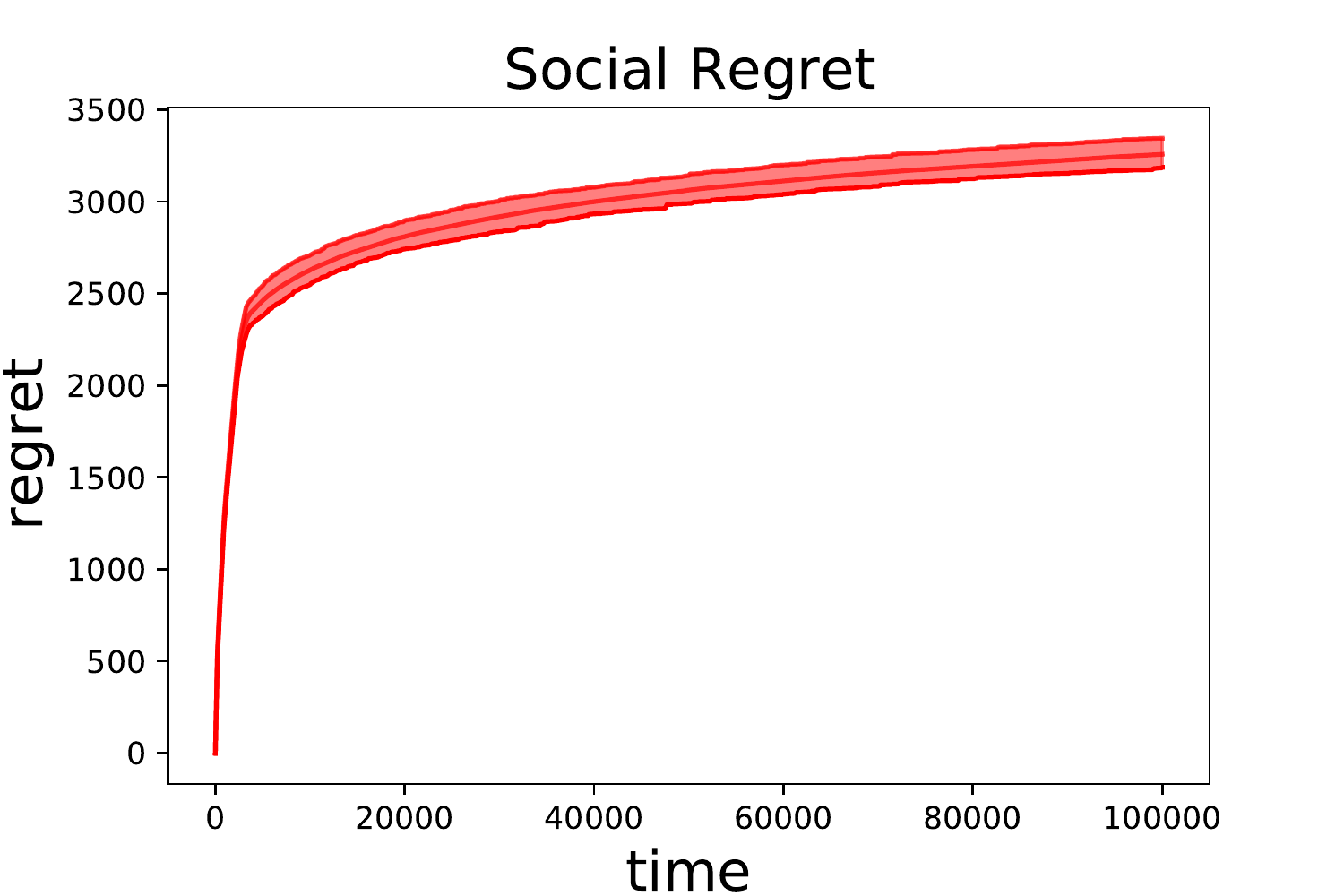}}
    \caption{Social Regret} 
    \end{subfigure}
    \hfill
    \begin{subfigure}[t]{0.64\linewidth}
        \raisebox{-\height}{\includegraphics[width=\textwidth]{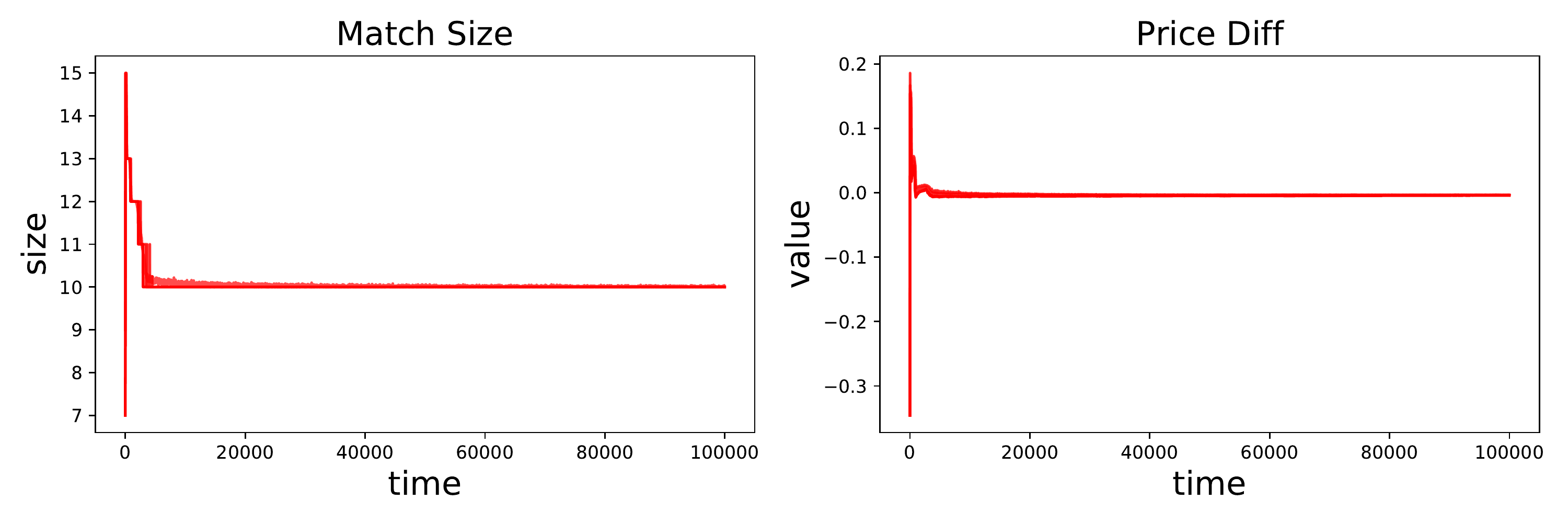}}
    \caption{Convergence of $K(t)$ and $(p(t)- p^*)$} 
    \end{subfigure}
    \caption{Double Auction $N=15$, $M=15$, $K^*=10$, $\Delta = 0.4$, $\alpha_1=4$, and $\alpha_2=8$}
    \label{fig:Fig151510}
\end{figure}

\subsection{Impact of the gap $\Delta$}
We first study the impact of changing the gap $\Delta$ on the performance, for a system of size $8x8$ and $K^* = 5$. In Figure~\ref{fig:Fig885_gap}, we study three gaps, $\Delta \in \{0.1, 0.15, 0.2\}$. Here, we observe that the convergence in the number of times an agent participates  is delayed as we decrease the gap. As a result, there is an increase in the regret of non-participants, and social regret which is dominated by $\Delta$, whereas the participant regret is not directly impacted by $\Delta$ (after the initial stage) as it is dominated by the $O(\sqrt{T\log(T)})$ term.  

\begin{figure}
     \centering
    \begin{subfigure}[t]{0.37\textwidth}
        \raisebox{-\height}{\includegraphics[width=\textwidth]{fig/1004_regret_buyer.pdf}}
        \caption{\tiny Regret and Matching of Buyers, $\Delta = 0.2$}
    \end{subfigure}
    \hfill
    \begin{subfigure}[t]{0.37\textwidth}
        \raisebox{-\height}{\includegraphics[width=\textwidth]{fig/1004_regret_seller.pdf}}
        \caption{\tiny  Regret and Matching of Sellers, $\Delta = 0.2$}
    \end{subfigure}
    \hfill
    \begin{subfigure}[t]{0.2\textwidth}
        \raisebox{-\height}{\includegraphics[width=\textwidth]{fig/1004_sw_regret.pdf}}
        \caption{\tiny  Social Regret, $\Delta = 0.2$}
    \end{subfigure}
    \begin{subfigure}[t]{0.37\textwidth}
        \raisebox{-\height}{\includegraphics[width=\textwidth]{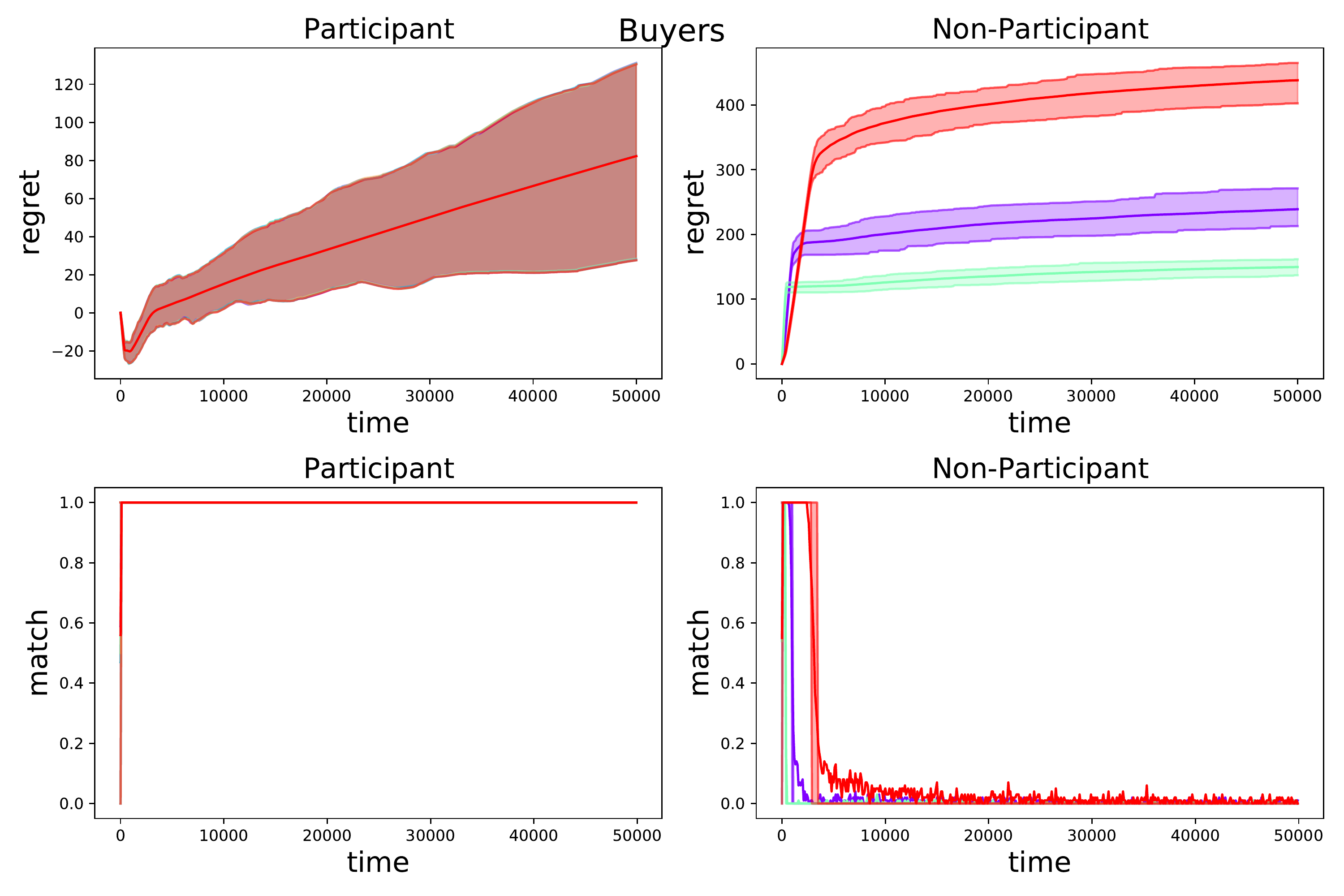}}
        \caption{\tiny Regret and Matching of Buyers, $\Delta = 0.15$}
    \end{subfigure}
    \hfill
    \begin{subfigure}[t]{0.37\textwidth}
        \raisebox{-\height}{\includegraphics[width=\textwidth]{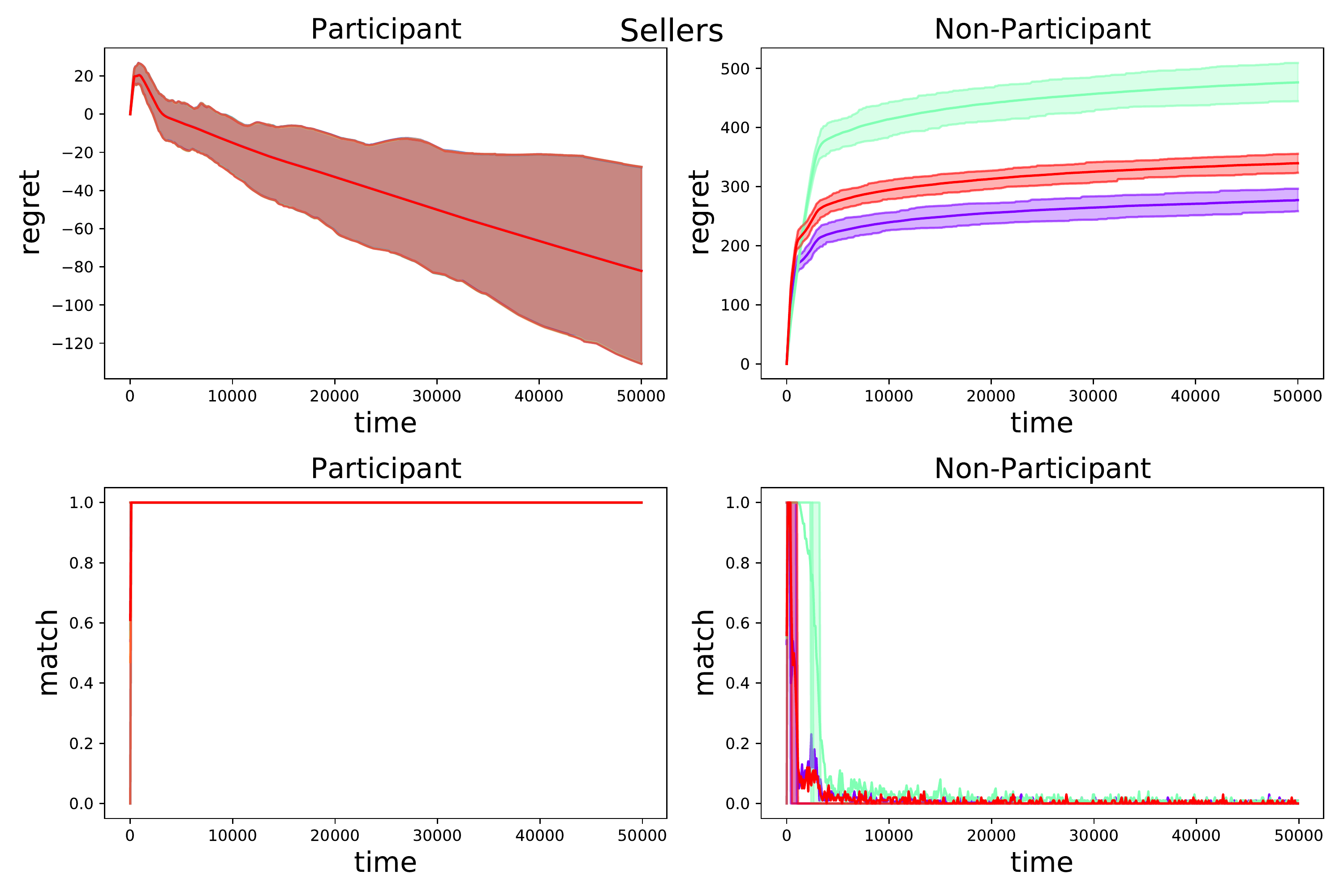}}
        \caption{\tiny Regret and Matching of Sellers, $\Delta = 0.15$}
    \end{subfigure}
    \hfill
    \begin{subfigure}[t]{0.2\textwidth}
        \raisebox{-\height}{\includegraphics[width=\textwidth]{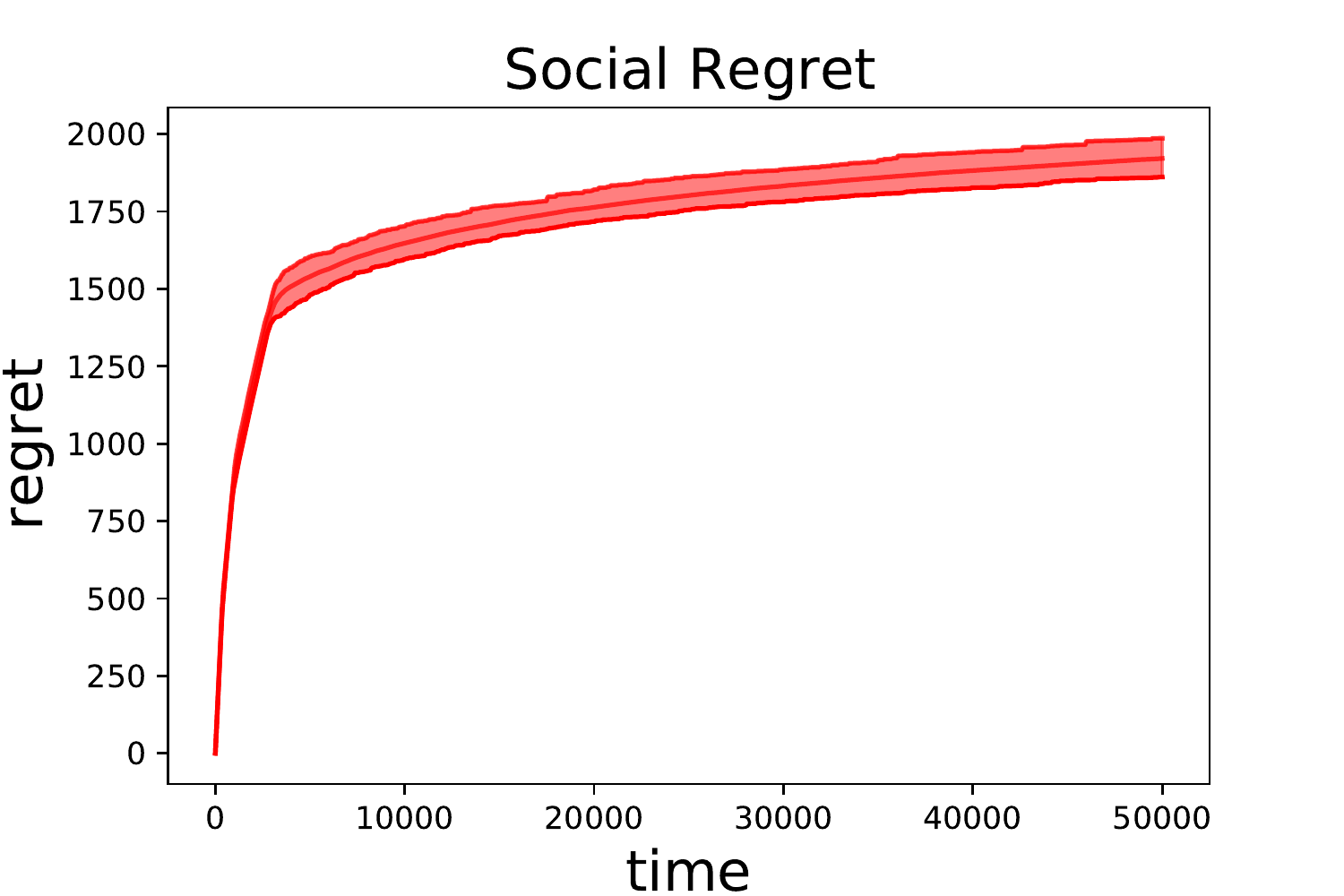}}
        \caption{\tiny  Social Regret, $\Delta = 0.15$}
    \end{subfigure}
    \begin{subfigure}[t]{0.37\textwidth}
        \raisebox{-\height}{\includegraphics[width=\textwidth]{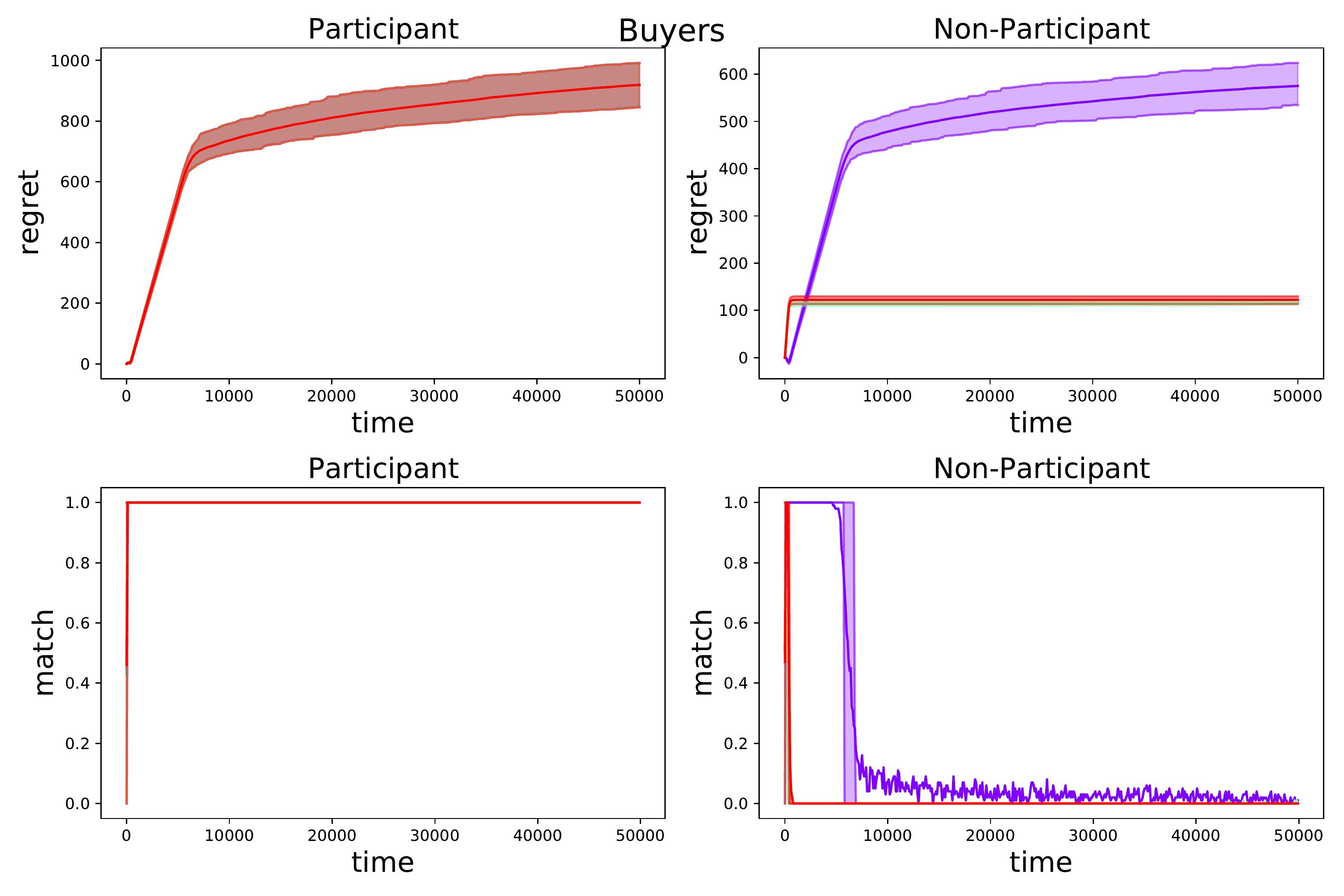}}
        \caption{\tiny Regret and Matching of Buyers, $\Delta = 0.1$}
    \end{subfigure}
    \hfill
    \begin{subfigure}[t]{0.37\textwidth}
        \raisebox{-\height}{\includegraphics[width=\textwidth]{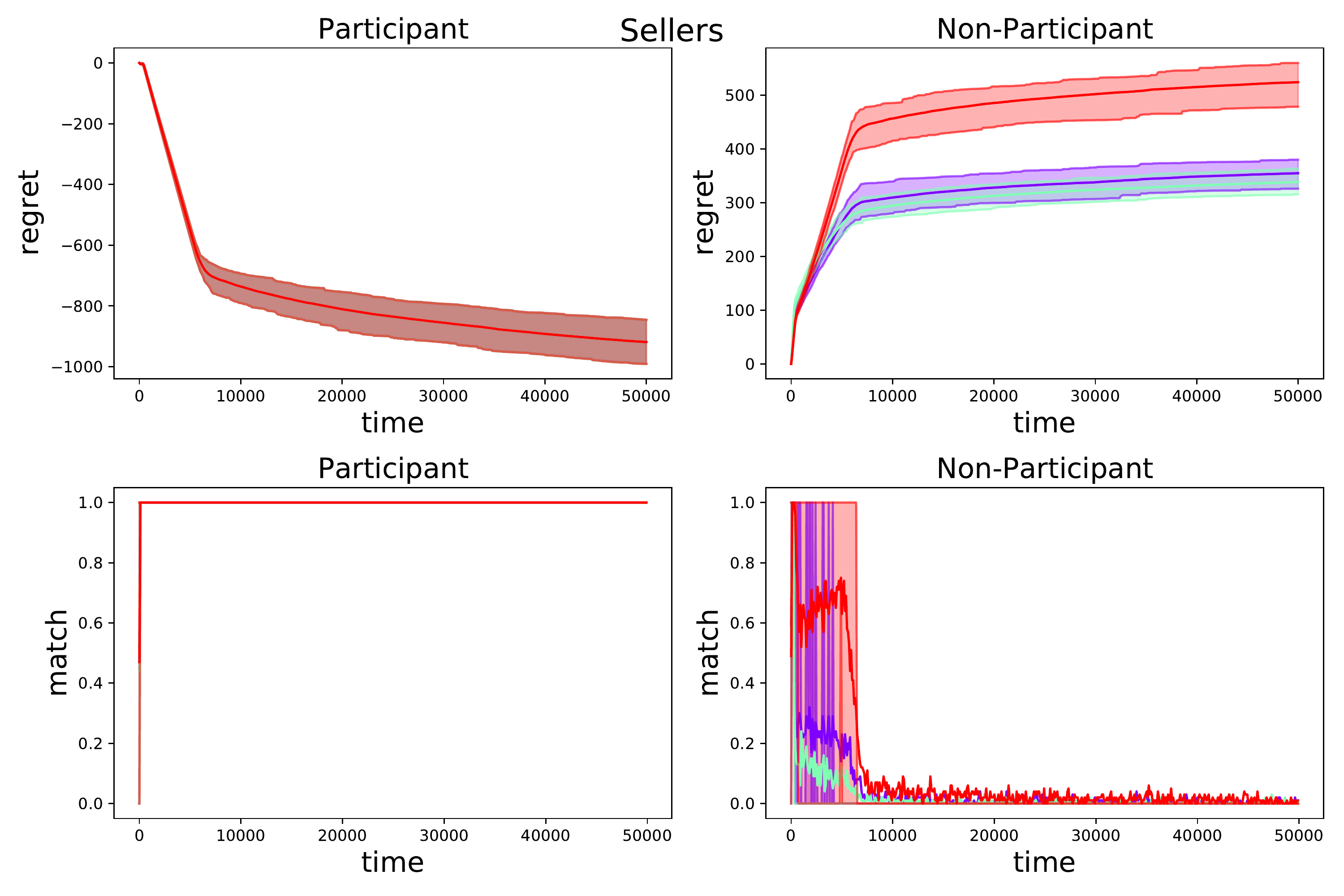}}
        \caption{\tiny Regret and Matching of Sellers, $\Delta = 0.1$}
    \end{subfigure}
    \hfill
    \begin{subfigure}[t]{0.2\textwidth}
        \raisebox{-\height}{\includegraphics[width=\textwidth]{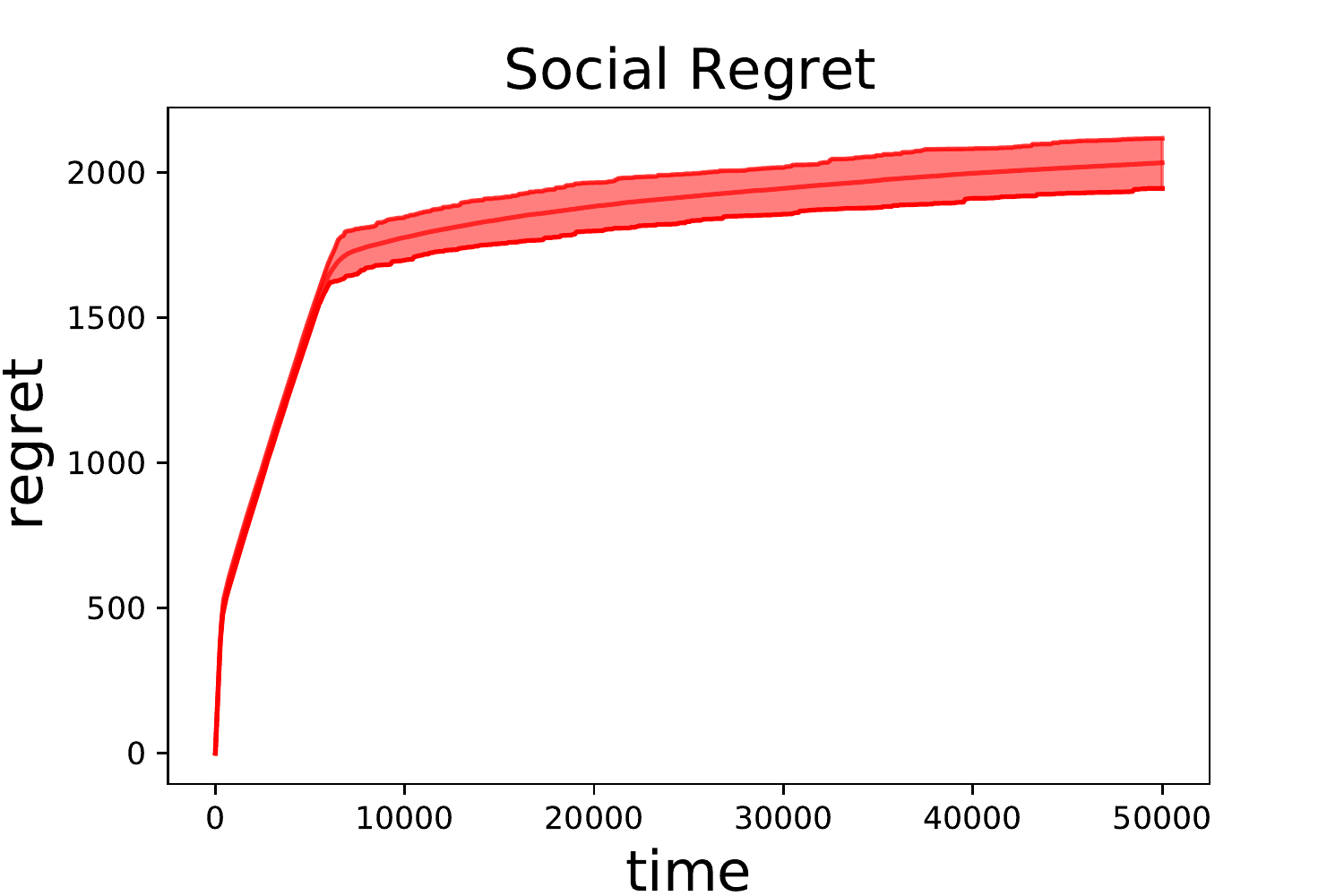}}
        \caption{\tiny  Social Regret, $\Delta = 0.1$}
    \end{subfigure}
    \caption{Double Auction with varying $\Delta$ ($N=M=8$, $K^*=5$,  $\alpha_1=4$, and $\alpha_2=8$). }
    \label{fig:Fig885_gap}
\end{figure}

\subsection{Impact of the sizes $M$, $N$, and $K^*$}
We now study the impact of the size of the system, and number of true participants.  First, with $M = N$, we vary $M$ while keeping the   $(M-K^*)$ fixed. We observe in Figure~\ref{fig:Fig885_M} that the regret of the agents do not vary a lot, which is as expected from the theory. Next, we keep $M=N=8$ fixed, while varying the participant size $K^*$. We see as $(M-K^*)$ increases, the regret increases in Figure~\ref{fig:Fig885_K} as suggested by the theory.  

\begin{figure}
     \centering
    \begin{subfigure}[t]{0.37\textwidth}
        \raisebox{-\height}{\includegraphics[width=\textwidth]{fig/1004_regret_buyer.pdf}}
        \caption{\tiny Regret and Matching of Buyers, $N\mathtt{=}M\mathtt{=}8$, $K^*\mathtt{=}5$}
    \end{subfigure}
    \hfill
    \begin{subfigure}[t]{0.37\textwidth}
        \raisebox{-\height}{\includegraphics[width=\textwidth]{fig/1004_regret_seller.pdf}}
        \caption{\tiny Regret and Matching of Sellers, $N\mathtt{=}M\mathtt{=}8$, $K^*\mathtt{=}5$}
    \end{subfigure}
    \hfill
    \begin{subfigure}[t]{0.2\textwidth}
        \raisebox{-\height}{\includegraphics[width=\textwidth]{fig/1004_sw_regret.pdf}}
        \caption{\tiny Social Regret,\\ $N=M=8$, $K^*=5$}
    \end{subfigure}
    \begin{subfigure}[t]{0.37\textwidth}
        \raisebox{-\height}{\includegraphics[width=\textwidth]{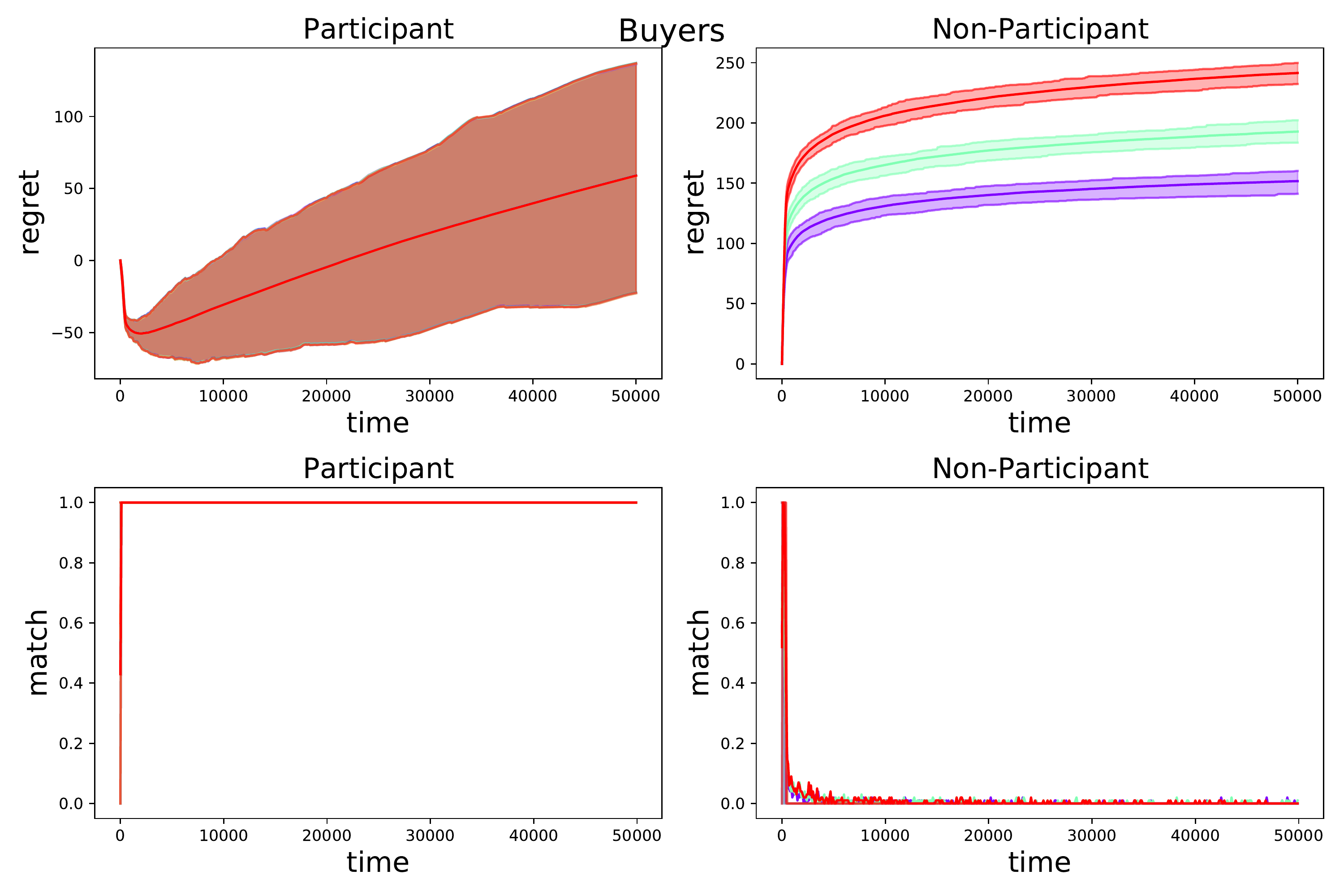}}
        \caption{\tiny Regret and Matching of Buyers, $N\mathtt{=}M\mathtt{=}10$, $K^*\mathtt{=}7$}
    \end{subfigure}
    \hfill
    \begin{subfigure}[t]{0.37\textwidth}
        \raisebox{-\height}{\includegraphics[width=\textwidth]{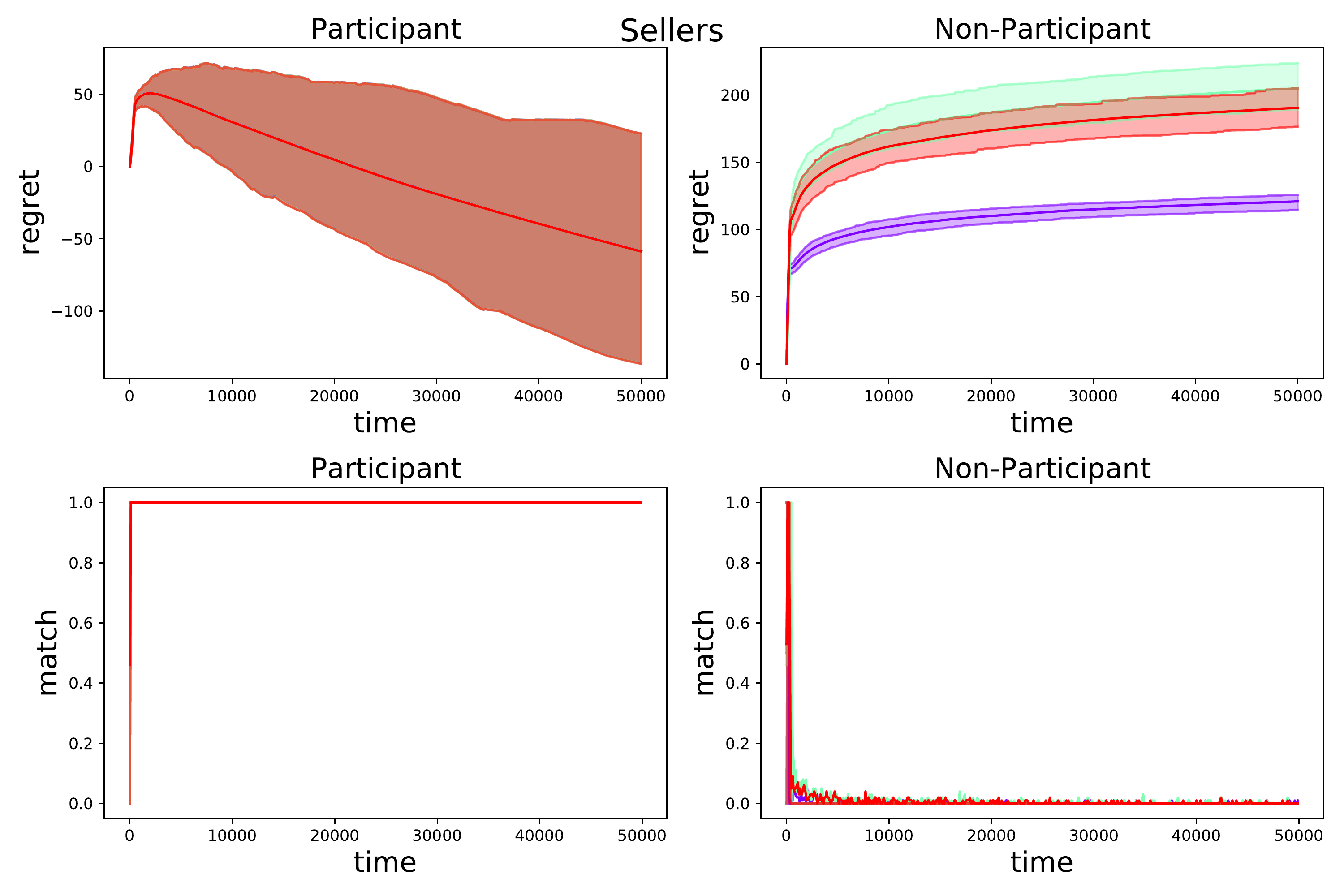}}
        \caption{\tiny Regret and Matching of Sellers, $N\mathtt{=}M\mathtt{=}10$, $K^*\mathtt{=}7$}
    \end{subfigure}
    \hfill
    \begin{subfigure}[t]{0.2\textwidth}
        \raisebox{-\height}{\includegraphics[width=\textwidth]{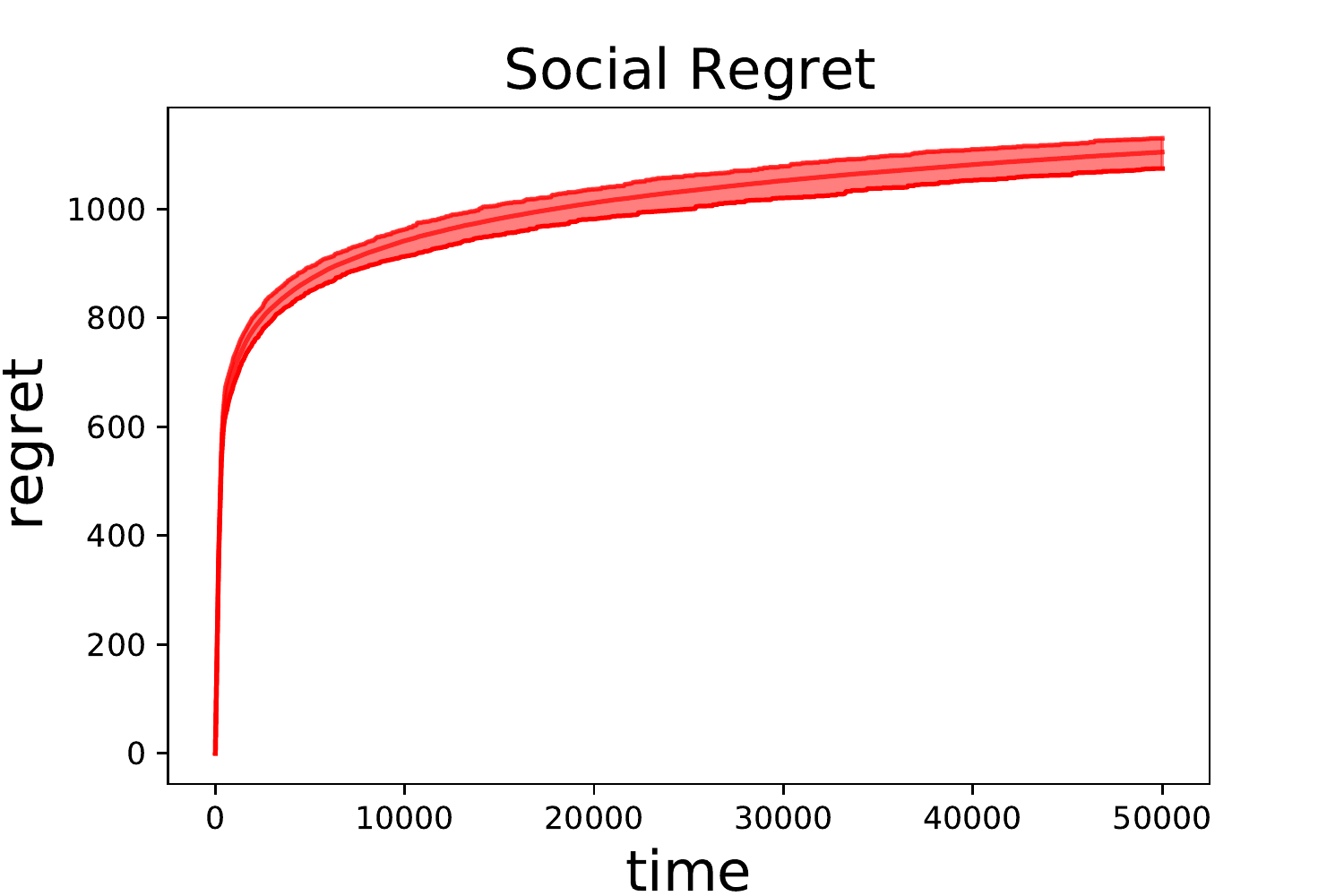}}
        \caption{\tiny Social Regret,\\ $N=M=10$, $K^*=7$}
    \end{subfigure}
    \begin{subfigure}[t]{0.37\textwidth}
        \raisebox{-\height}{\includegraphics[width=\textwidth]{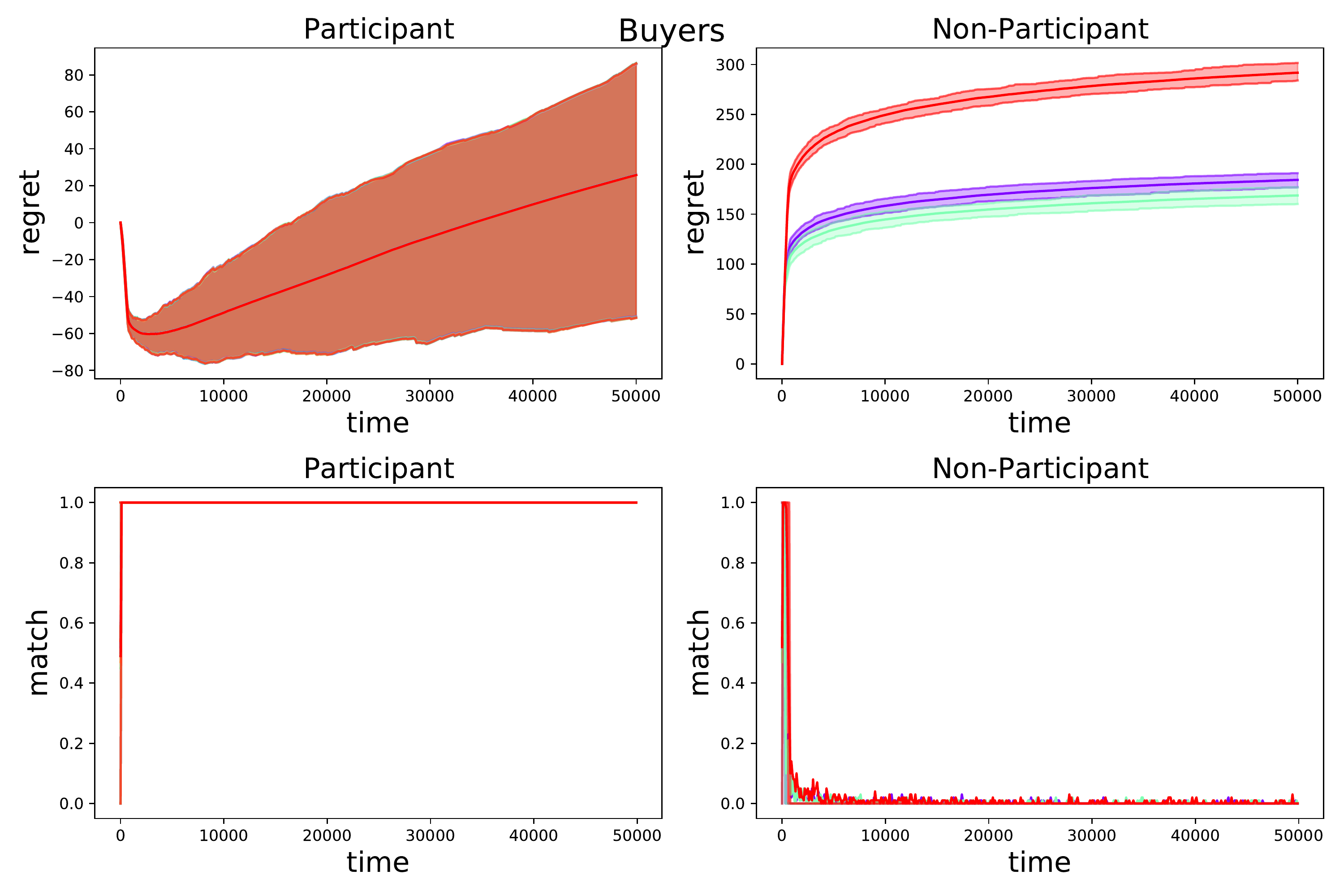}}
        \caption{\tiny Regret and Matching of Buyers, $N\mathtt{=}M\mathtt{=}12$, $K^*\mathtt{=}9$}
    \end{subfigure}
    \hfill
    \begin{subfigure}[t]{0.37\textwidth}
        \raisebox{-\height}{\includegraphics[width=\textwidth]{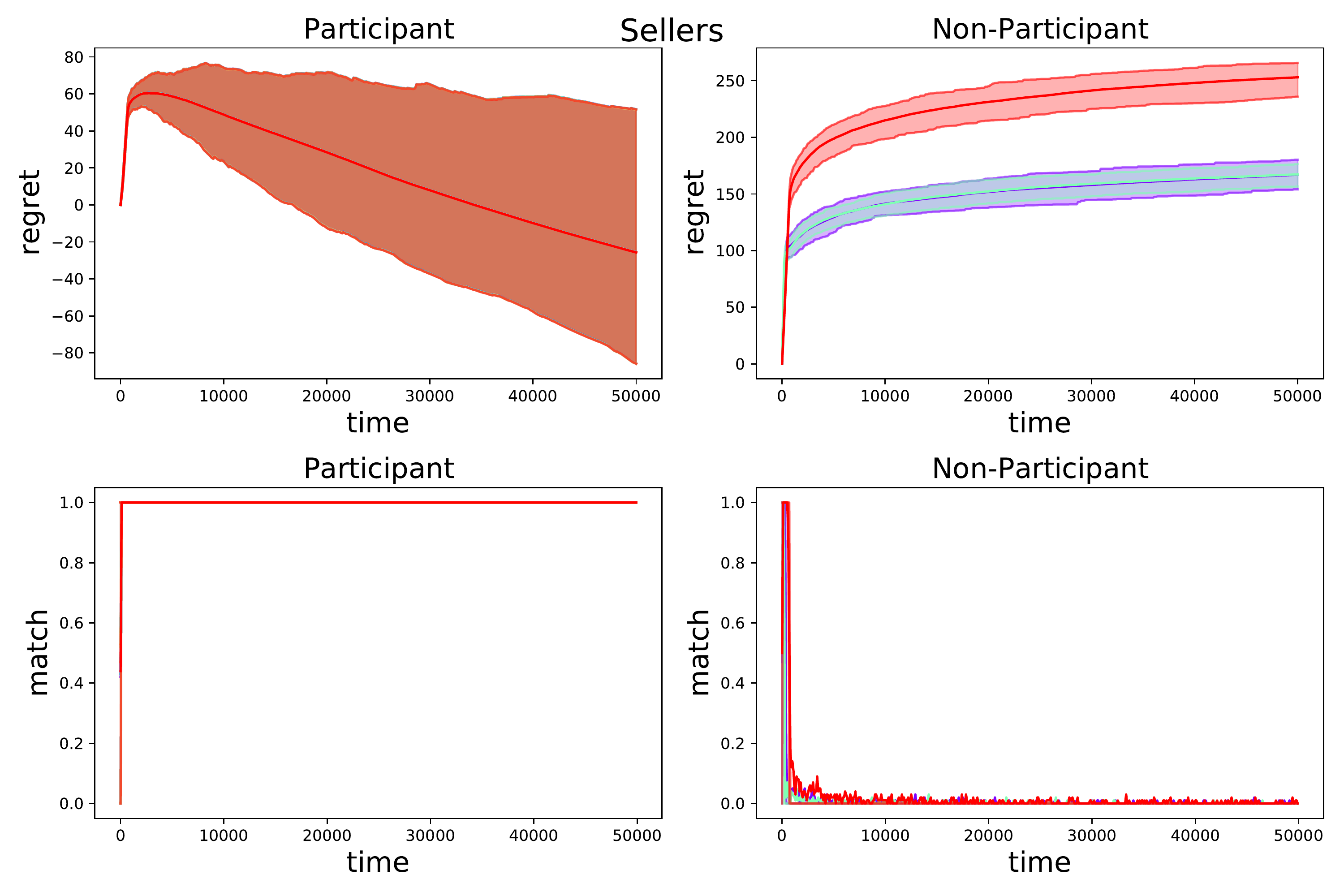}}
        \caption{\tiny Regret and Matching of Sellers, $N\mathtt{=}M\mathtt{=}12$, $K^*\mathtt{=}9$}
    \end{subfigure}
    \hfill
    \begin{subfigure}[t]{0.2\textwidth}
        \raisebox{-\height}{\includegraphics[width=\textwidth]{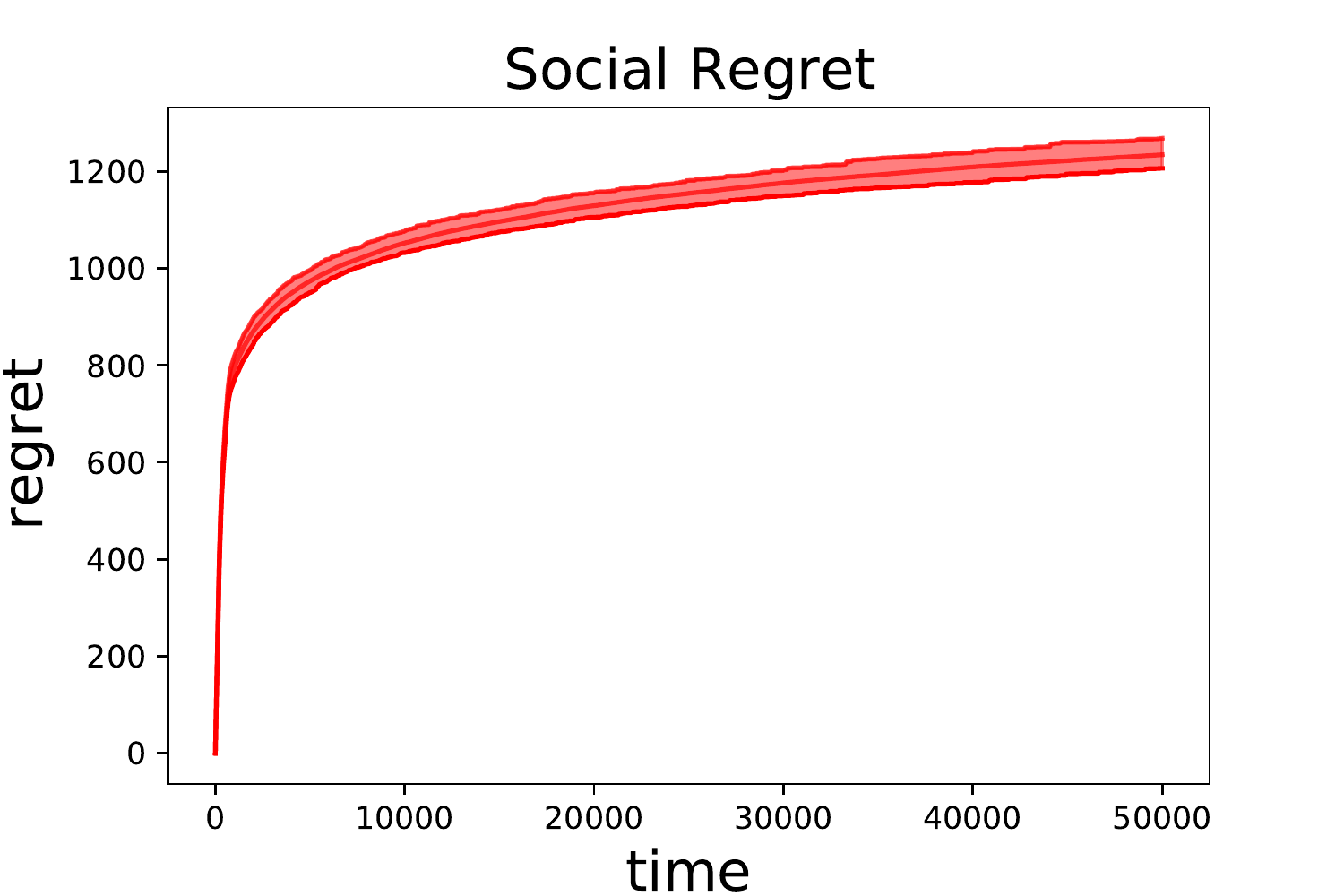}}
        \caption{\tiny Social Regret,\\ $N=M=12$, $K^*=9$}
    \end{subfigure}
    \caption{Double Auction with varying $N=M$, and fixed $(M-K^*)=3$ ($\Delta = 0.2$,  $\alpha_1=4$, $\alpha_2=8$)}
    \label{fig:Fig885_M}
\end{figure}

\begin{figure}
     \centering
    \begin{subfigure}[t]{0.37\textwidth}
        \raisebox{-\height}{\includegraphics[width=\textwidth]{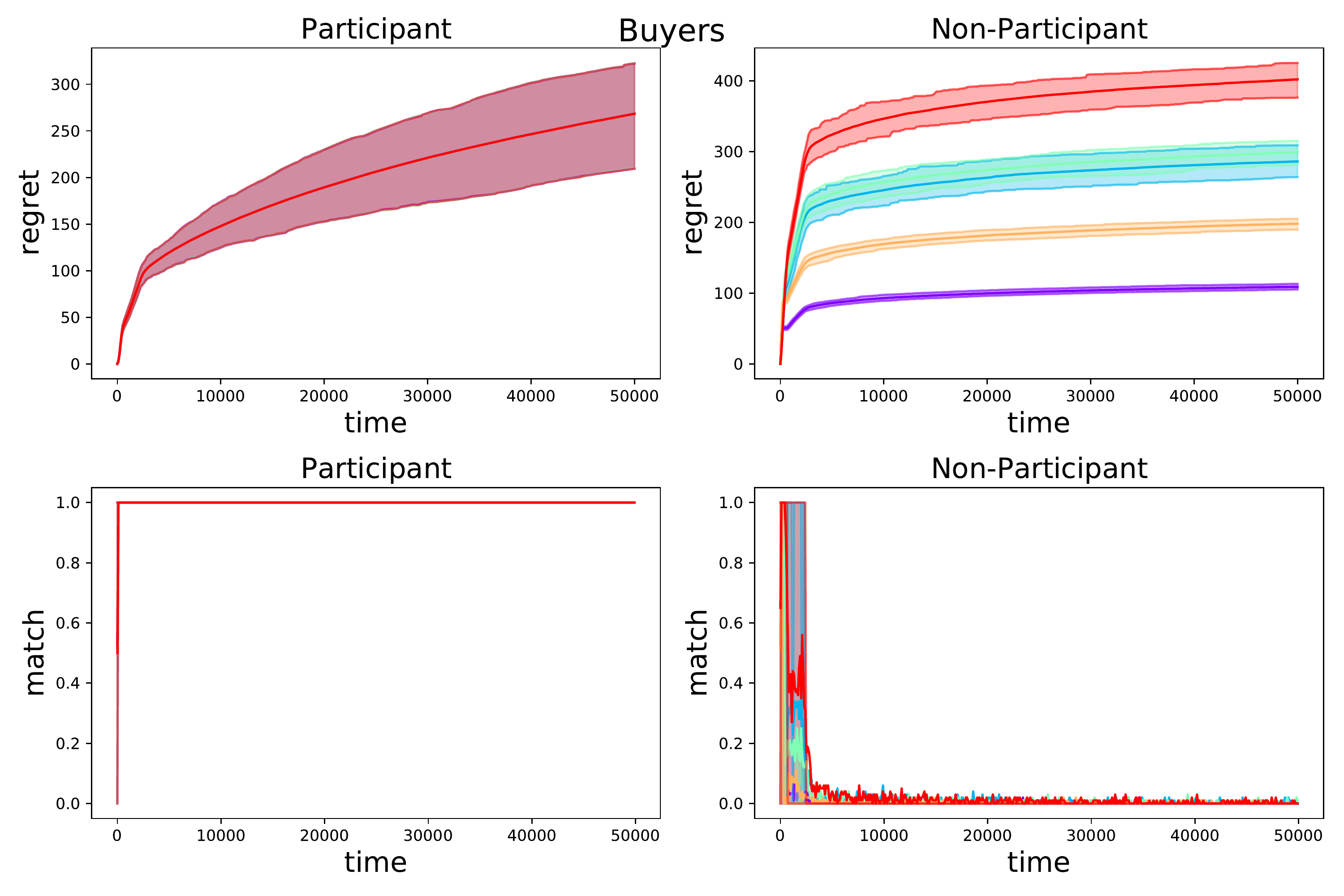}}
        \caption{\tiny Regret and Matching of Buyers, $K^*=3$}
    \end{subfigure}
    \hfill
    \begin{subfigure}[t]{0.37\textwidth}
        \raisebox{-\height}{\includegraphics[width=\textwidth]{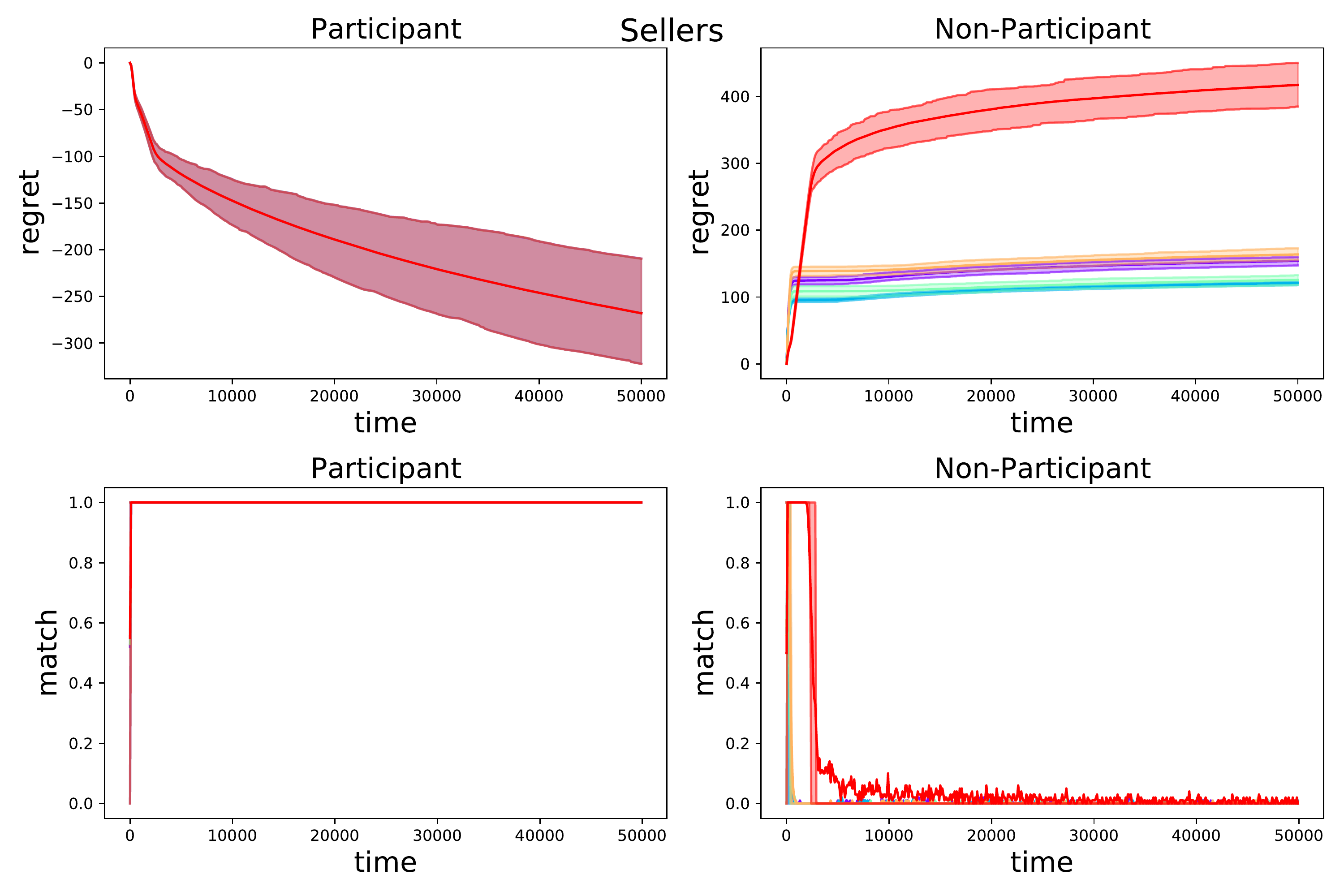}}
        \caption{\tiny Regret and Matching of Sellers,$K^*=3$}
    \end{subfigure}
    \hfill
    \begin{subfigure}[t]{0.2\textwidth}
        \raisebox{-\height}{\includegraphics[width=\textwidth]{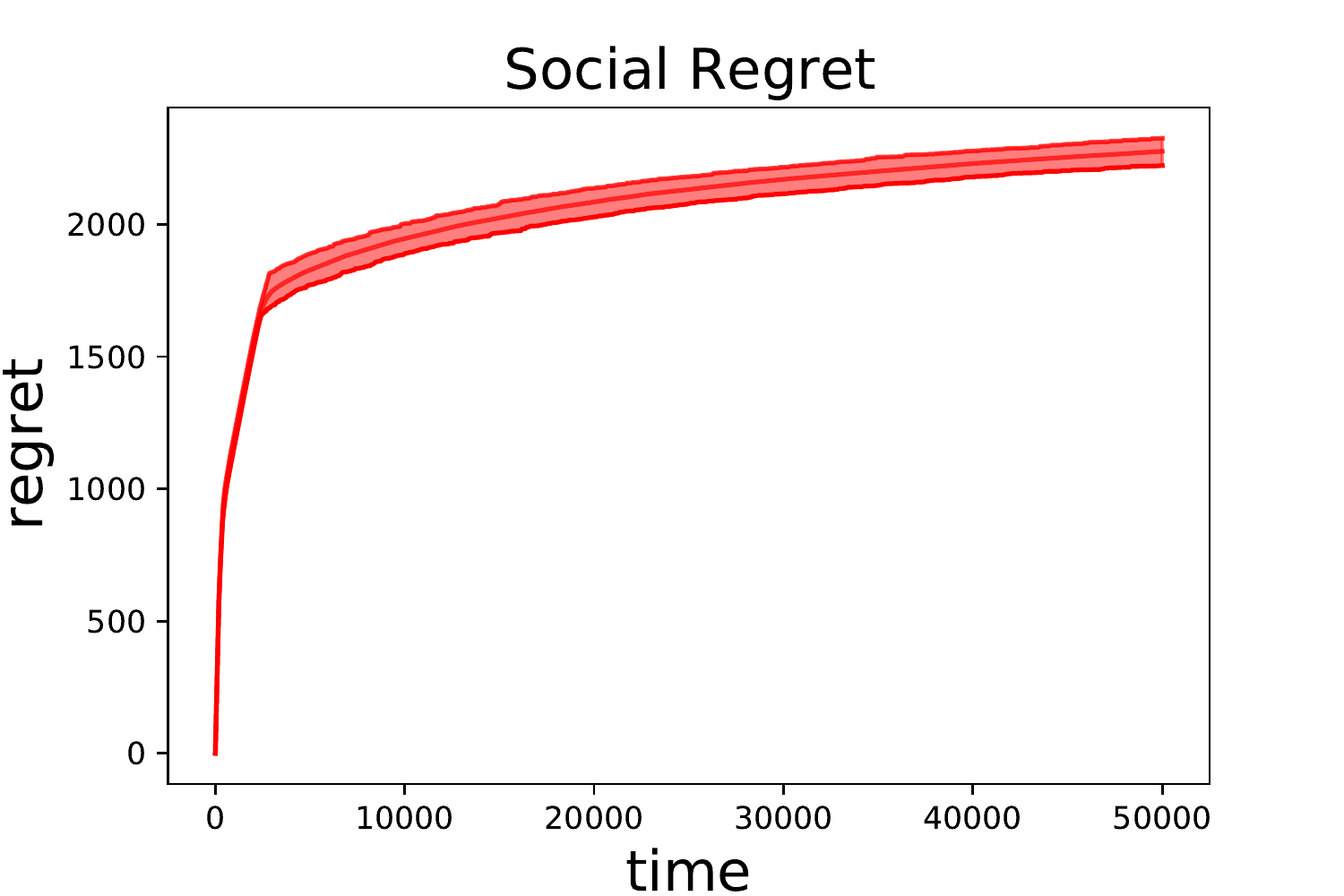}}
        \caption{\tiny Social Regret, $K^*=3$}
    \end{subfigure}
    \begin{subfigure}[t]{0.37\textwidth}
        \raisebox{-\height}{\includegraphics[width=\textwidth]{fig/1004_regret_buyer.pdf}}
        \caption{\tiny Regret and Matching of Buyers, $K^*=5$}
    \end{subfigure}
    \hfill
    \begin{subfigure}[t]{0.37\textwidth}
        \raisebox{-\height}{\includegraphics[width=\textwidth]{fig/1004_regret_seller.pdf}}
        \caption{\tiny Regret and Matching of Sellers,  $K^*=5$}
    \end{subfigure}
    \hfill
    \begin{subfigure}[t]{0.2\textwidth}
        \raisebox{-\height}{\includegraphics[width=\textwidth]{fig/1004_sw_regret.pdf}}
        \caption{\tiny Social Regret, $K^*=5$}
    \end{subfigure}
    \begin{subfigure}[t]{0.37\textwidth}
        \raisebox{-\height}{\includegraphics[width=\textwidth]{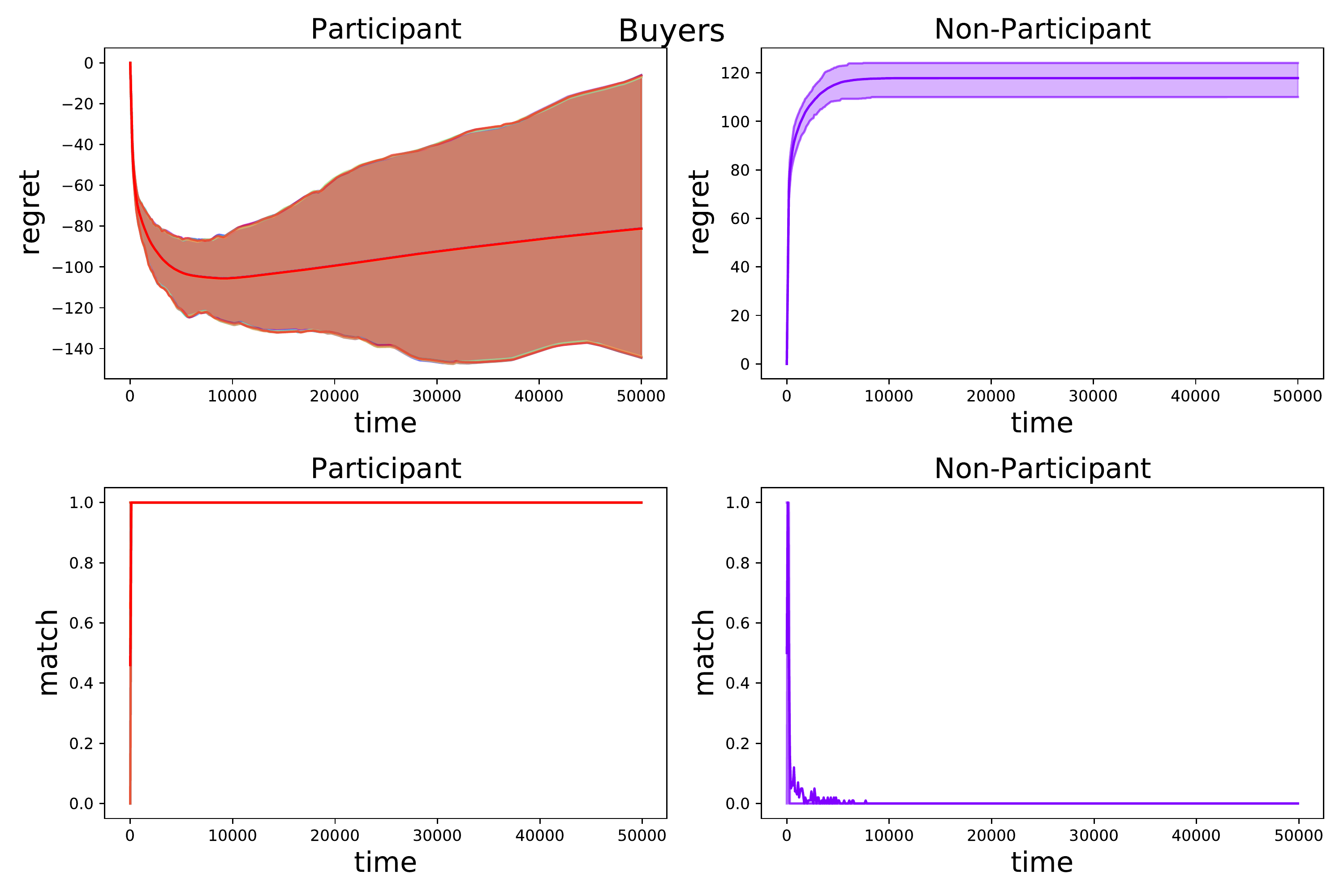}}
        \caption{\tiny Regret and Matching of Buyers, $K^*=7$}
    \end{subfigure}
    \hfill
    \begin{subfigure}[t]{0.37\textwidth}
        \raisebox{-\height}{\includegraphics[width=\textwidth]{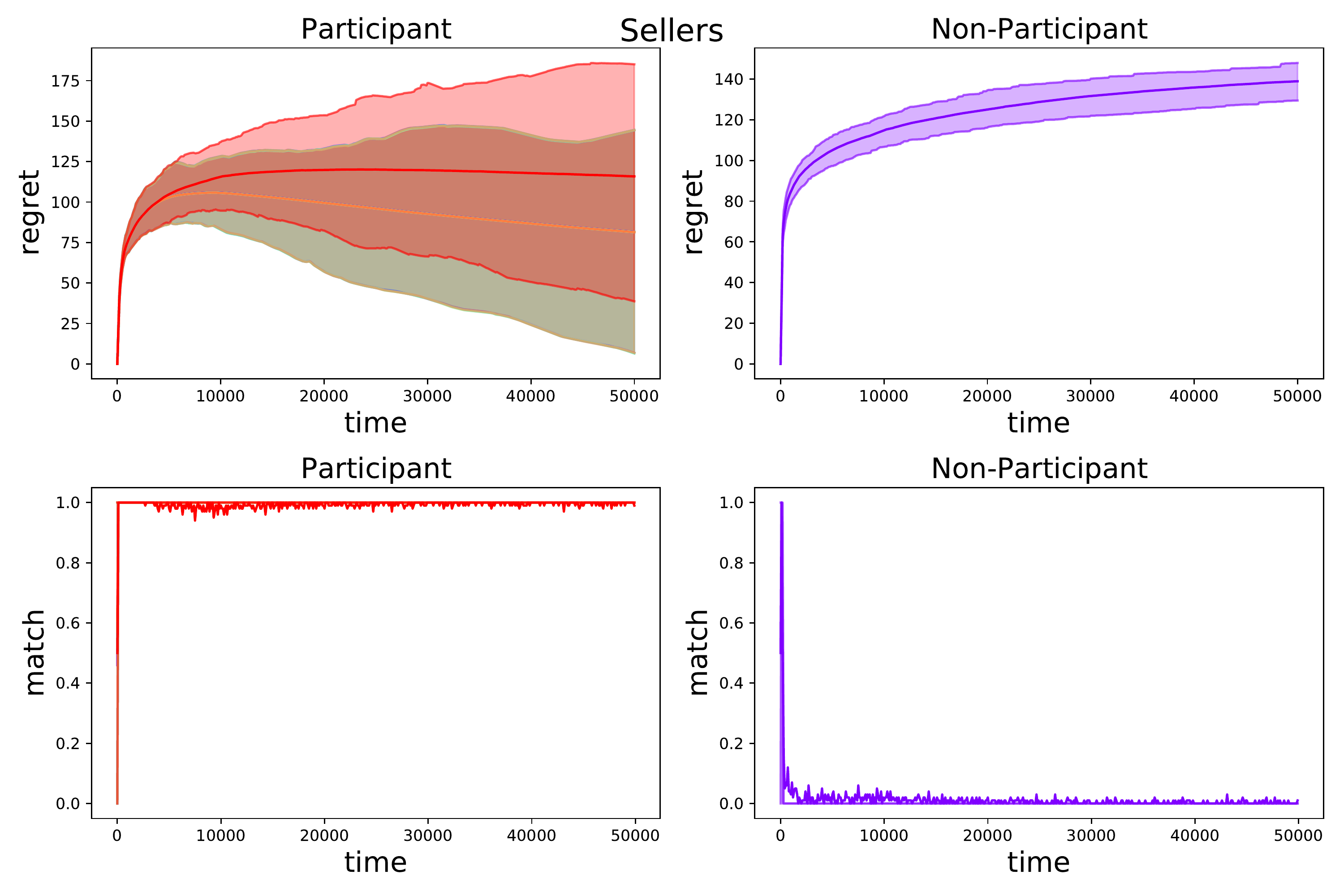}}
        \caption{\tiny Regret and Matching of Sellers, $K^*=7$}
    \end{subfigure}
    \hfill
    \begin{subfigure}[t]{0.2\textwidth}
        \raisebox{-\height}{\includegraphics[width=\textwidth]{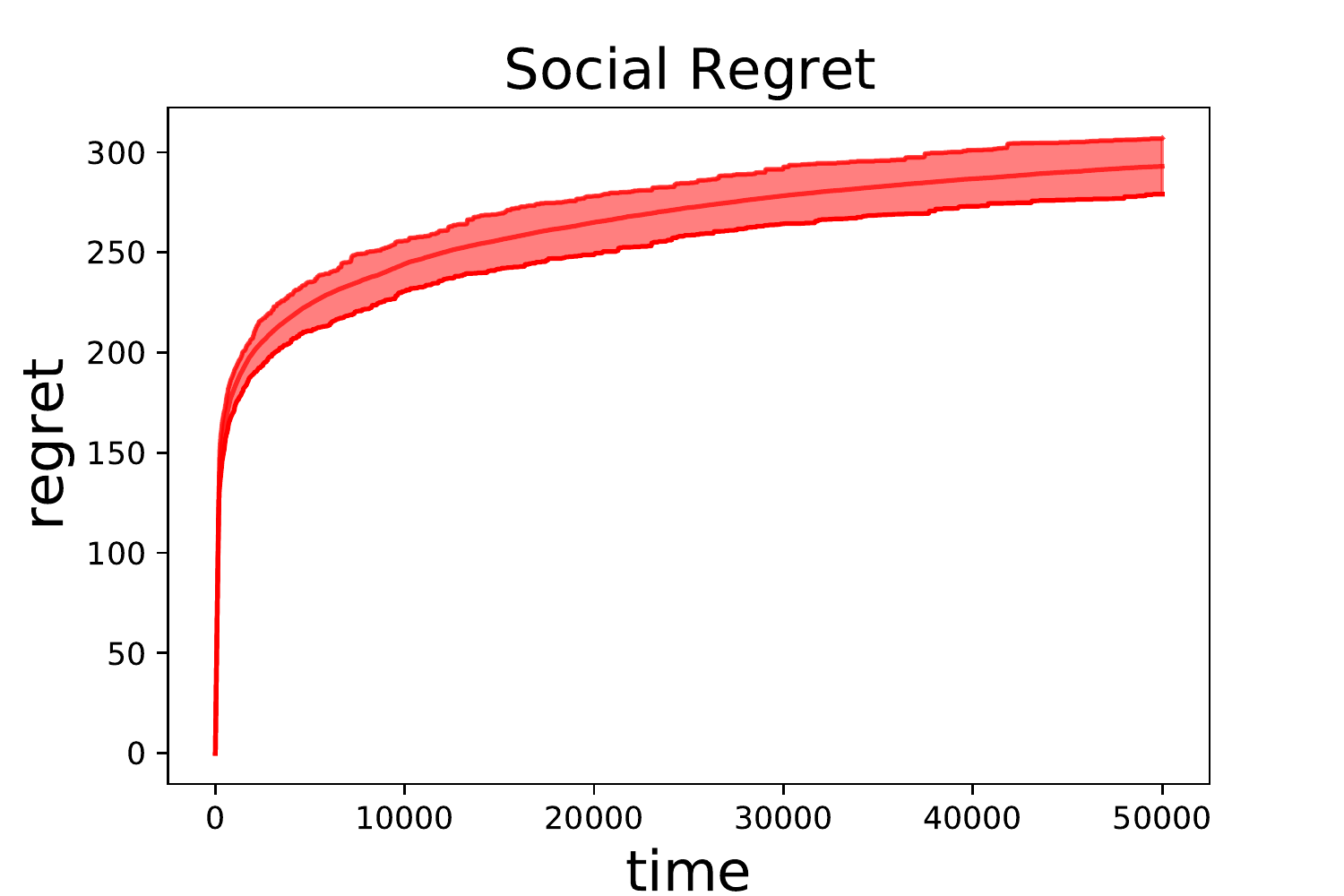}}
        \caption{\tiny Social Regret, $K^*=7$}
    \end{subfigure}
    \caption{Double Auction with varying $K^*$ ($N=M=8$, $\Delta = 0.2$,  $\alpha_1=4$, $\alpha_2=8$)}
    \label{fig:Fig885_K}
\end{figure}

\subsection{Impact of size difference in $M$ and $N$.}
In this part, we assess the size difference between number of sellers $M$, and number of buyers $N$. As the buyer and seller are symmetric, we keep $N=8$, and $K^*=5$ fixed. We next vary $M \in \{5, 8, 15\}$ to study the effect. In Figure~\ref{fig:Fig885_MN}, we observe the regret of non-participant buyers and sellers increases with increase in $M$. This is mainly because, the  presence of more non-participant sellers is allowing the non-participant buyers to match with these non-participant sellers more increasing the regret. Similar observation holds for the social regret. The participant regret which is mostly dominated by price estimation error is mostly unaffected by this.

\begin{figure}
     \centering
    \begin{subfigure}[t]{0.37\textwidth}
        \raisebox{-\height}{\includegraphics[width=\textwidth]{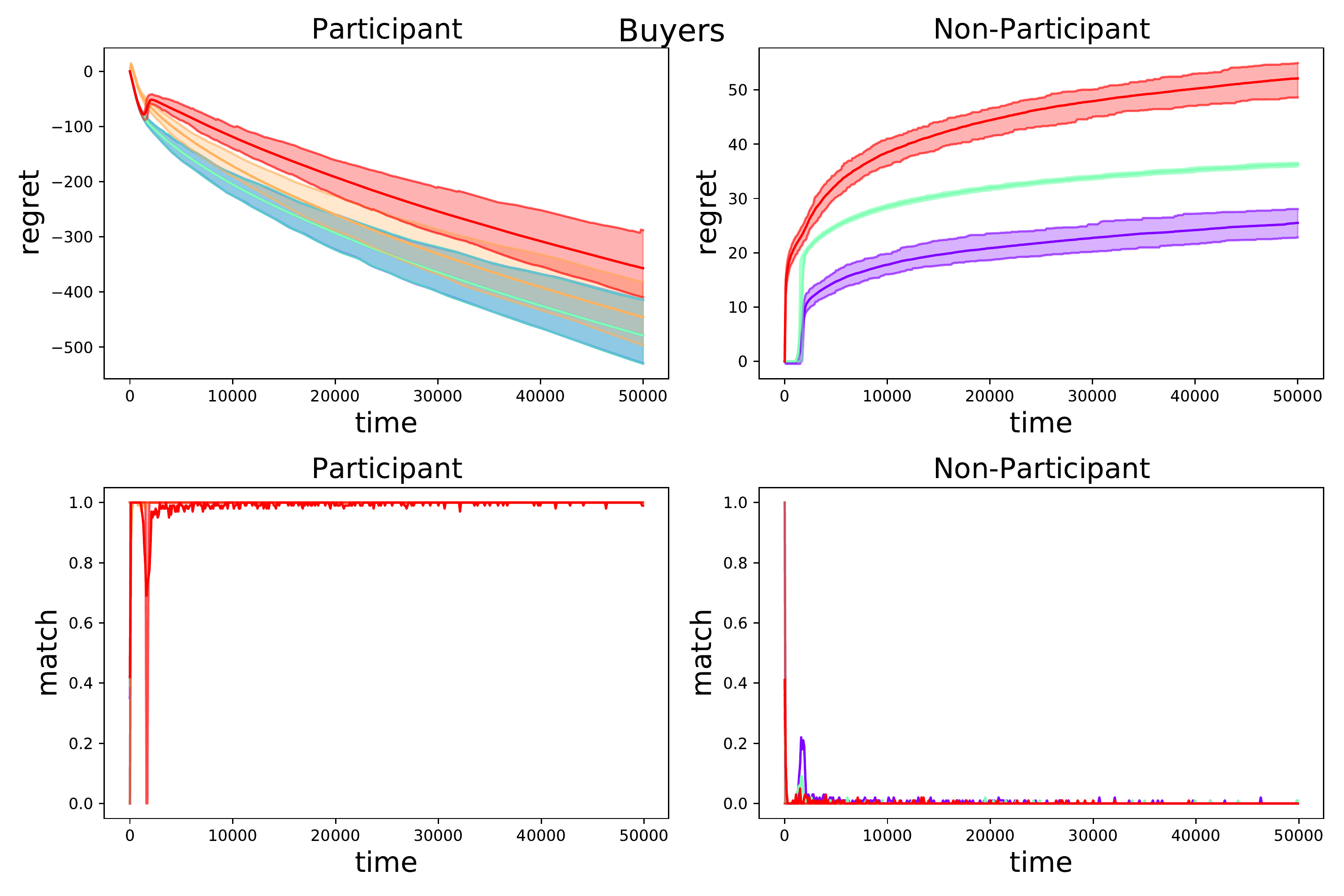}}
        \caption{\tiny Regret and Matching of Buyers, $M=5$}
    \end{subfigure}
    \hfill
    \begin{subfigure}[t]{0.37\textwidth}
        \raisebox{-\height}{\includegraphics[width=\textwidth]{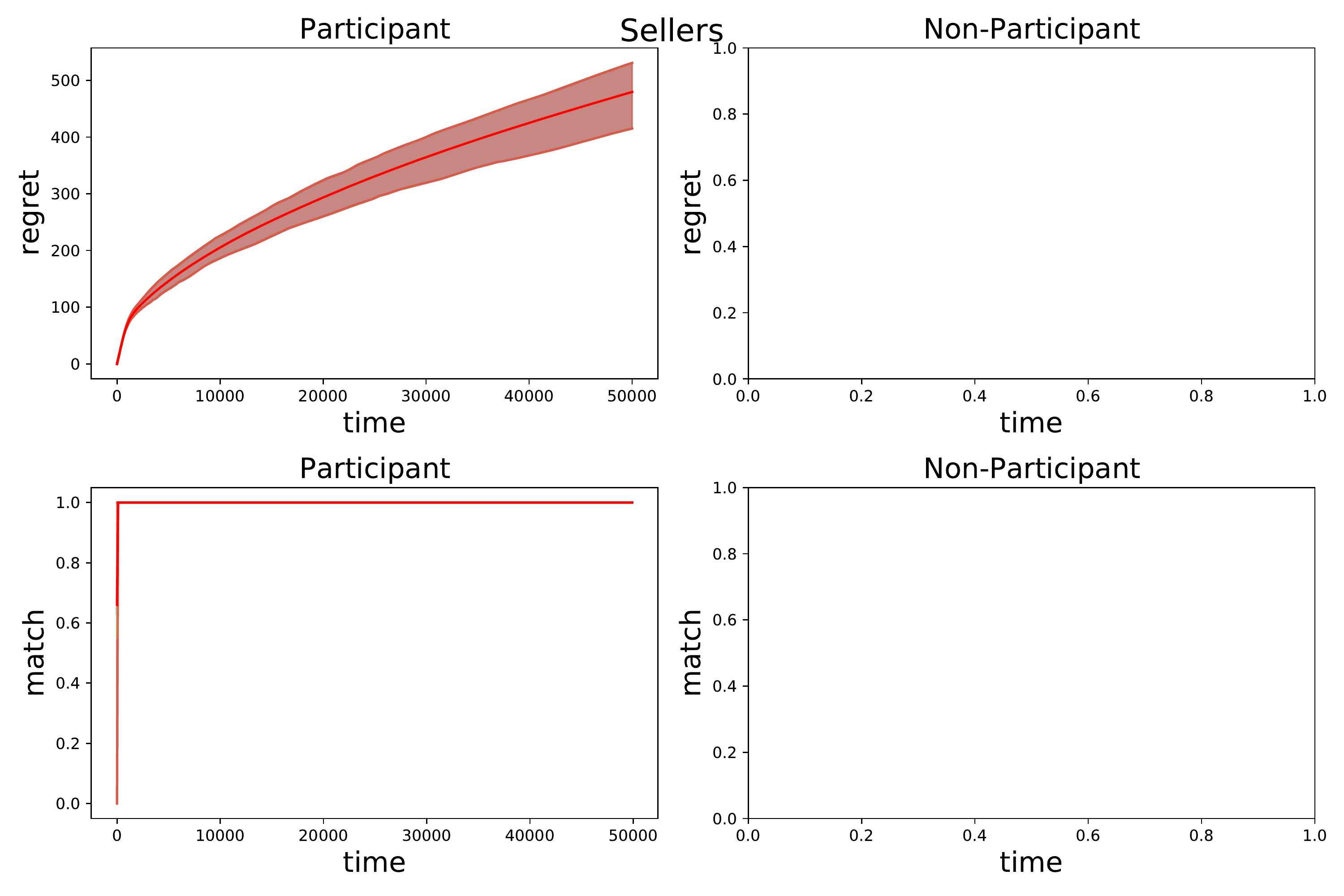}}
        \caption{\tiny Regret and Matching of Sellers,$M=5$}
    \end{subfigure}
    \hfill
    \begin{subfigure}[t]{0.2\textwidth}
        \raisebox{-\height}{\includegraphics[width=\textwidth]{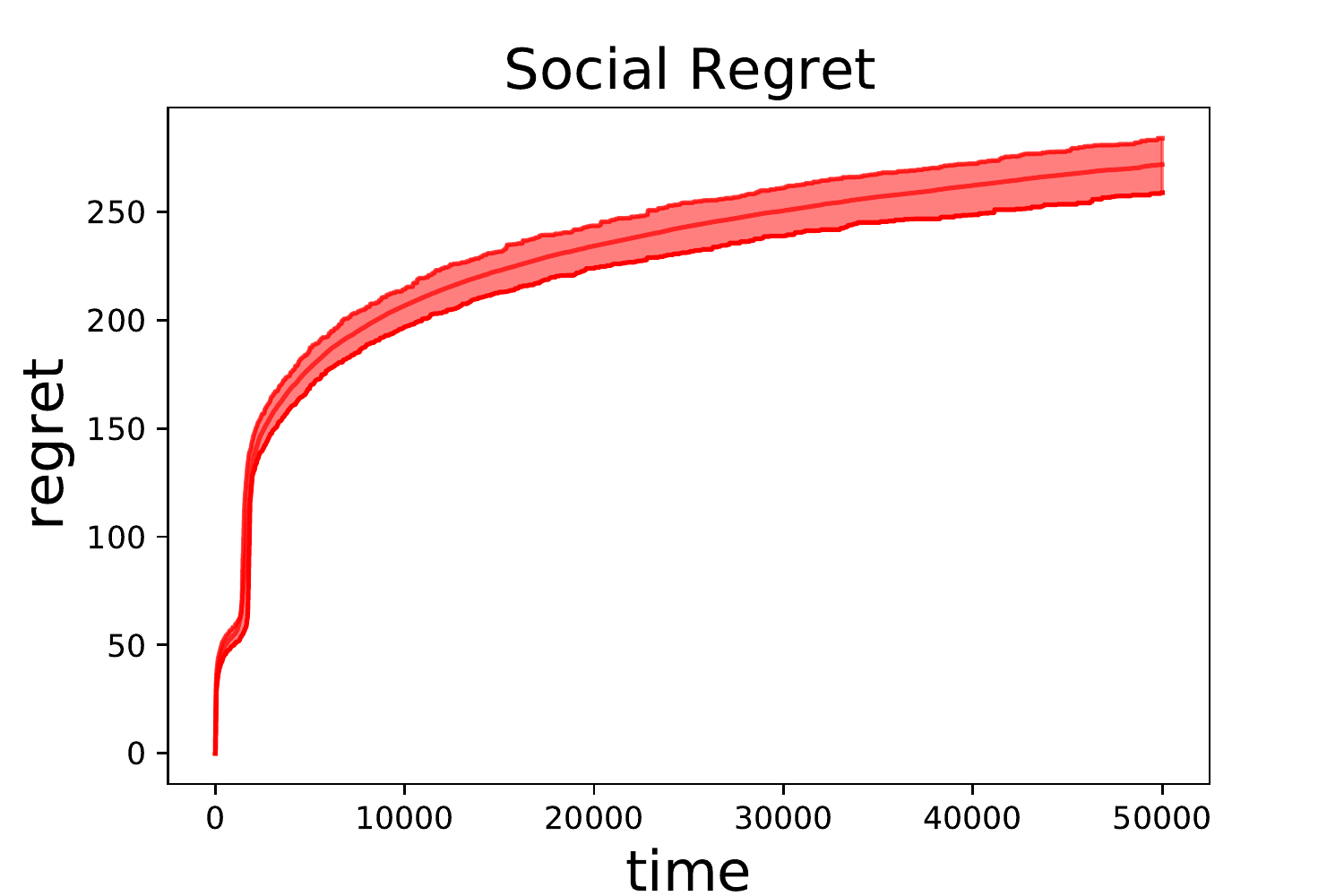}}
        \caption{\tiny Social Regret, $M=5$}
    \end{subfigure}
    \begin{subfigure}[t]{0.37\textwidth}
        \raisebox{-\height}{\includegraphics[width=\textwidth]{fig/1004_regret_buyer.pdf}}
        \caption{\tiny Regret and Matching of Buyers, $M=8$}
    \end{subfigure}
    \hfill
    \begin{subfigure}[t]{0.37\textwidth}
        \raisebox{-\height}{\includegraphics[width=\textwidth]{fig/1004_regret_seller.pdf}}
        \caption{\tiny Regret and Matching of Sellers,  $M=8$}
    \end{subfigure}
    \hfill
    \begin{subfigure}[t]{0.2\textwidth}
        \raisebox{-\height}{\includegraphics[width=\textwidth]{fig/1004_sw_regret.pdf}}
        \caption{\tiny Social Regret, $M=8$}
    \end{subfigure}
    \begin{subfigure}[t]{0.37\textwidth}
        \raisebox{-\height}{\includegraphics[width=\textwidth]{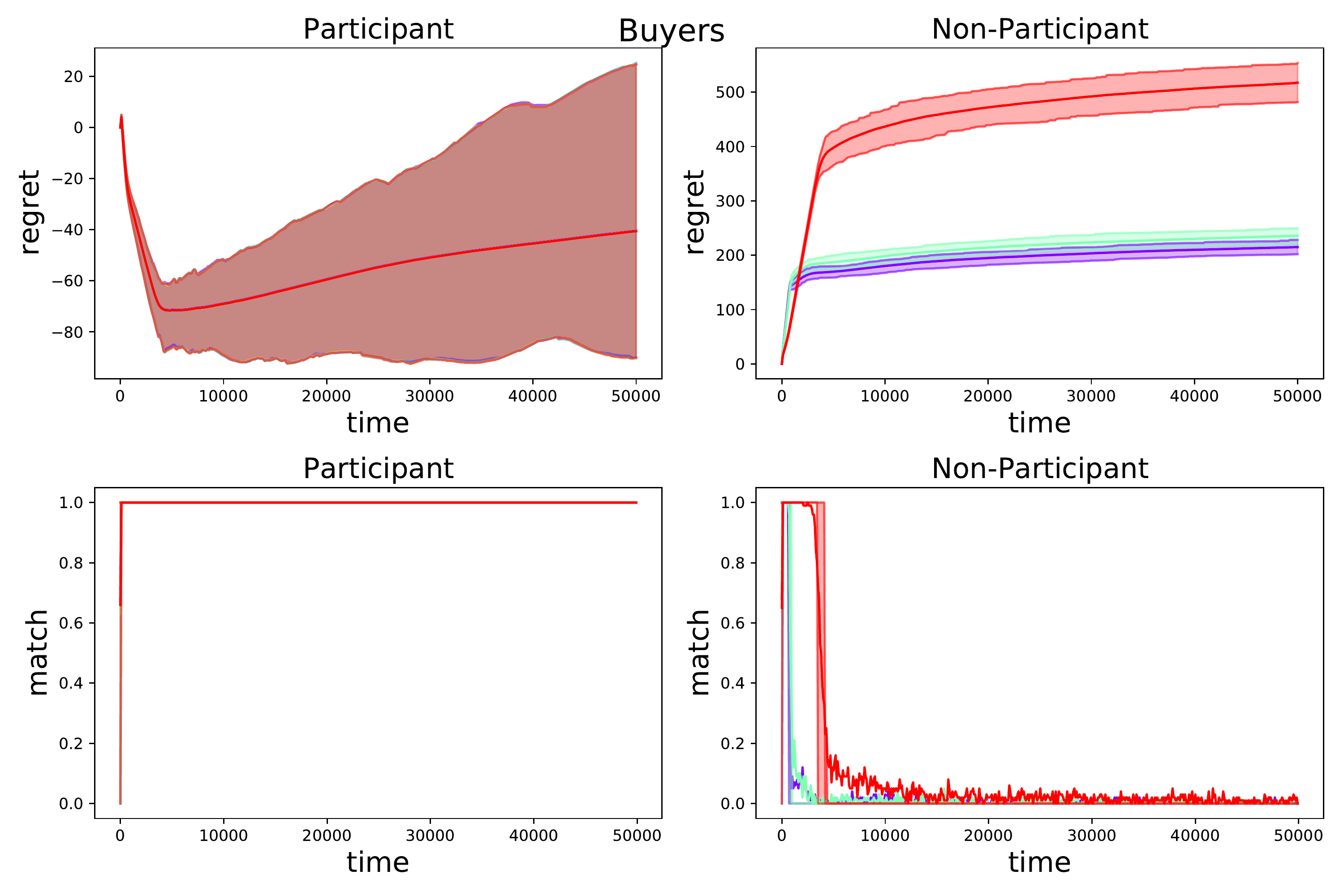}}
        \caption{\tiny Regret and Matching of Buyers, $M=15$}
    \end{subfigure}
    \hfill
    \begin{subfigure}[t]{0.37\textwidth}
        \raisebox{-\height}{\includegraphics[width=\textwidth]{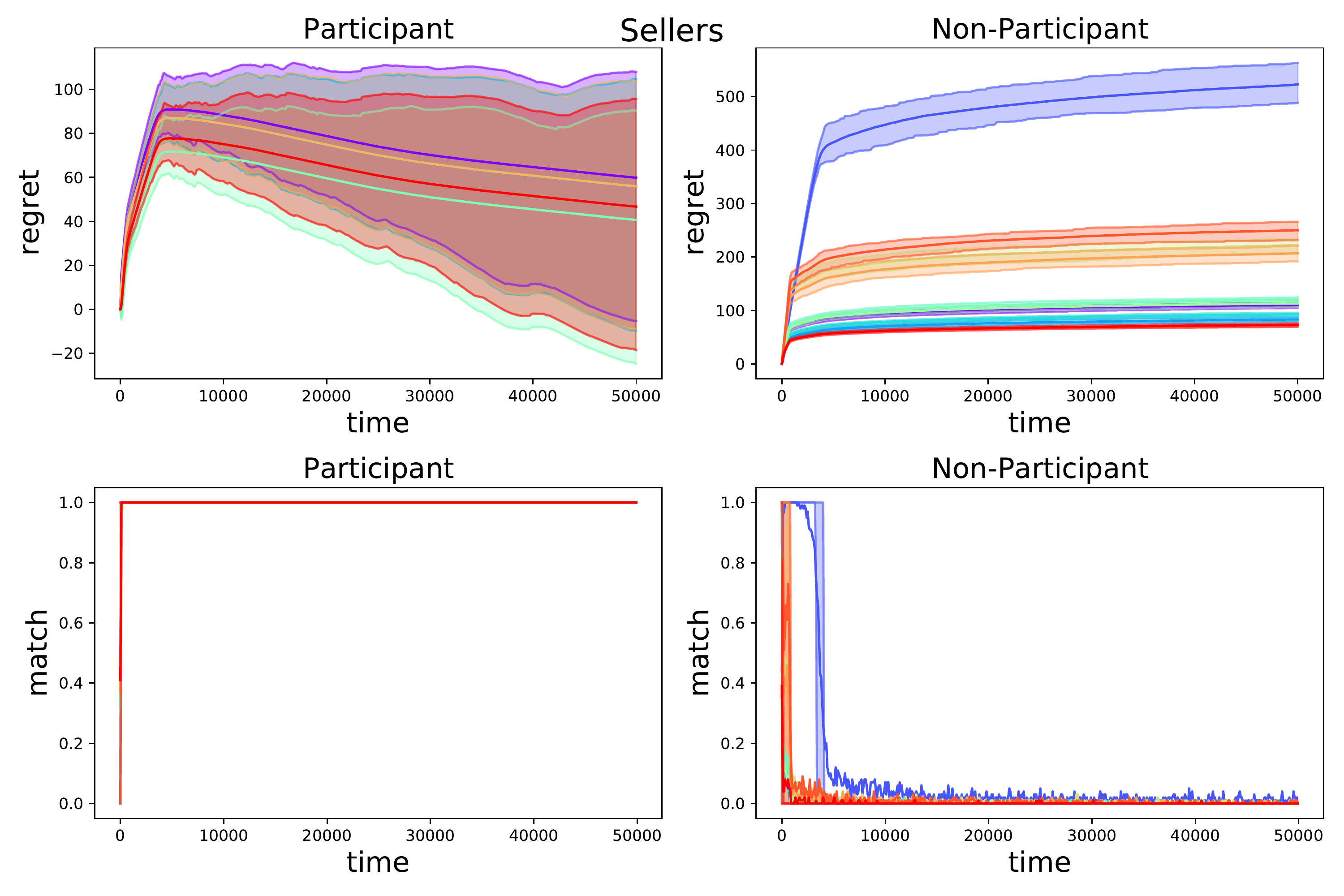}}
        \caption{\tiny Regret and Matching of Sellers, $M=15$}
    \end{subfigure}
    \hfill
    \begin{subfigure}[t]{0.2\textwidth}
        \raisebox{-\height}{\includegraphics[width=\textwidth]{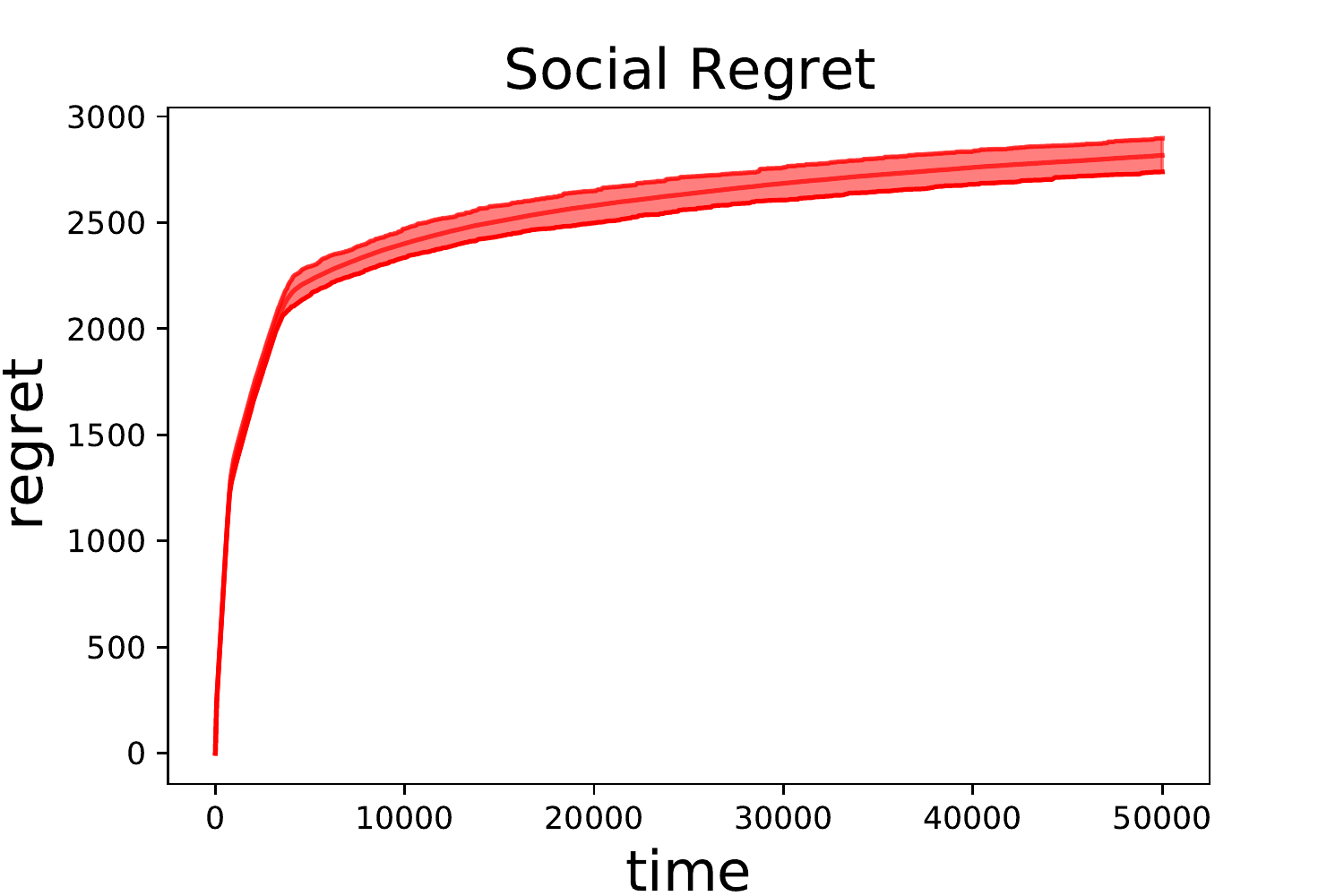}}
        \caption{\tiny Social Regret, $M=15$}
    \end{subfigure}
    \caption{Double Auction with varying $M$ ($N=8$, $K^*=5$, $\Delta = 0.2$,  $\alpha_1=4$, $\alpha_2=8$)}
    \label{fig:Fig885_MN}
\end{figure}

\end{document}